\documentclass[11pt,letterpaper]{article}

\usepackage[margin=1in]{geometry}
\usepackage[T1]{fontenc}
\usepackage[utf8]{inputenc}
\usepackage{lmodern}
\usepackage{microtype}
\usepackage{amsmath,amssymb,amsthm,mathtools}
\usepackage{bm,dsfont}
\allowdisplaybreaks[1]

\usepackage{enumitem}
\usepackage[table]{xcolor}
\usepackage{graphicx}
\usepackage{subcaption}
\usepackage{booktabs,tabularx,makecell,array}
\usepackage{tikz}
\usetikzlibrary{arrows.meta,positioning,fit,calc}
\usepackage{adjustbox}
\usepackage[normalem]{ulem}
\usepackage{algorithm}
\usepackage{algpseudocode}

\theoremstyle{plain}
\newtheorem{theorem}{Theorem}
\newtheorem{lemma}{Lemma}
\newtheorem{proposition}{Proposition}
\newtheorem{corollary}{Corollary}

\theoremstyle{definition}
\newtheorem{assumption}{Assumption}

\newtheorem{question}{Question}
\theoremstyle{remark}

\usepackage[numbers,sort&compress]{natbib}

\usepackage{xurl}

\usepackage[unicode=true]{hyperref}
\hypersetup{
    pdftitle={Steady-State Convergence of Stochastic Approximation},
    pdfauthor={Yixuan Zhang and Qiaomin Xie},
    colorlinks=true,
    linkcolor=blue,
    citecolor=blue,
    urlcolor=blue,
    bookmarksnumbered=true,
    bookmarksopen=false,
    breaklinks=true
}
\usepackage{bookmark}

\renewcommand{\d}{\mathrm{d}}
\newcommand{\E}{\mathbb{E}}

\newcommand{\mS}{\mathcal{S}}
\newcommand{\mA}{\mathcal{A}}
\newcommand{\mT}{\mathcal{T}}

\newcommand{\R}{\mathbb{R}}
\renewcommand{\P}{\mathbb{P}}
\newcommand{\W}{\mathcal{W}}

\newcommand{\tmT}{\widetilde{\mT}}
\newcommand{\law}{\mathcal{L}}

\newcommand{\B}{\mathcal{B}}

\title{Steady-State Convergence of Stochastic Approximation}
\author{
Yixuan Zhang and Qiaomin Xie
\\[0.8em]
Department of Industrial \& Systems Engineering, University of Wisconsin--Madison
\\
\texttt{\{yzhang2554,qiaomin.xie\}@wisc.edu}
}
\date{}

\begin{document}
\maketitle

\begin{abstract}
For constant-stepsize stochastic approximation (SA), the iterates converge in distribution to a stationary law that depends on the stepsize $\alpha.$
Steady-state convergence (SSC) concerns the limit of the scaled stationary distribution as $\alpha \downarrow 0.$ 
Existing SSC theory requires i.i.d.\ or additive noise and global differentiability of the mean operator, and yields suboptimal rates. We
develop a unified SSC theory for constant-stepsize contractive SA driven by Markovian, multiplicative noise, covering  both locally differentiable and locally nondifferentiable mean operators. A key methodological contribution is a multi-step
universality framework that progressively reduces the original stochastic
recursion to tractable auxiliary dynamics while preserving its steady-state
limit. Under local quadratic linearization at the fixed point, we obtain a Gaussian approximation of the scaled steady state at the optimal rate $\mathcal O(\sqrt{\alpha})$ in Wasserstein-2 distance, which further gives finite-time Gaussian approximations for the raw iterates. In the locally nondifferentiable regime, we establish a general SSC
result and show that the leading-order asymptotic bias can be of order
\(\sqrt{\alpha}\), in contrast to the $\alpha$-order bias in the smooth regime. We apply the theory to Markovian linear SA and asynchronous Q-learning,
neither of which is covered by prior results. We further propose a bias-reduction scheme for Q-learning that requires no knowledge of the local smoothness regime, validated by numerical experiments.

\end{abstract}


\section{Introduction}\label{sec:Intro}
Stochastic approximation (SA) is a foundational algorithmic framework for
solving fixed-point problems from noisy observations through iterative
updates. SA-type methods are ubiquitous in reinforcement learning (RL),
stochastic control, and optimization
\citep{Bertsekas19-RL-book,Sutton18-RL-book,kushner2003-yin-sa-book,moulines2011non}.
A prototypical constant-stepsize SA recursion takes the form
\begin{equation}\label{eq:general-SA}
\theta_{t+1}^{(\alpha)}
=
\theta_t^{(\alpha)}
+
\alpha\bigl(
\tmT(x_t,\theta_t^{(\alpha)})
-
\theta_t^{(\alpha)}
\bigr),
\qquad t\geq 0,
\end{equation}
where $\alpha>0$ is a constant stepsize and $\{x_t\}_{t\geq 0}$ is a
noise sequence describing how data are generated. Constant stepsizes are
particularly attractive in practice because of their simple implementation
and rapid initial convergence
\citep{chen2020finite,zhang2024constant}. We assume that the noise sequence $\{x_t\}_{t \geq 0}$ admits a
limiting distribution $\mu$---for example, the common marginal distribution
in the i.i.d.\ setting or the stationary distribution in the Markovian
setting. The recursion~\eqref{eq:general-SA} seeks to approximate a fixed
point $\theta^*$ of the associated mean operator
\[
\mT(\theta)
:=
\mathbb{E}_{x\sim\mu}\bigl[\tmT(x,\theta)\bigr],
\qquad
\text{that is,}
\qquad
\mT(\theta^*)=\theta^*.
\]
The general formulation~\eqref{eq:general-SA} covers a broad class of
stochastic iterative algorithms, including stochastic gradient descent (SGD)
for stochastic optimization \citep{lan2020first} and a variety of
methods in reinforcement learning \citep{Sutton18-RL-book}. In this work, we
focus on \emph{contractive stochastic approximation}, where the mean operator
$\mT$ is a contraction with respect to a norm $\|\cdot\|_c$. This contraction
property guarantees that the fixed point $\theta^*$ is unique.

A canonical example is the celebrated Q-learning algorithm in reinforcement learning \citep{Watkins92-QLearning}, in which \(\mT\) corresponds to a linearly
transformed Bellman optimality operator \citep{chen2024lyapunov}. Figure~\ref{fig:q-learning-long-run-tail-average}
illustrates both the long-run behavior and the transient evolution  of
constant-stepsize Q-learning; the experimental setup is described in
Section~\ref{sec:experiment-intro}. The first three panels show that, after a
sufficiently long run, the iterates continue to fluctuate in a neighborhood
of the optimal action-value function $q^*$, with the magnitude of the
fluctuations increasing with the stepsize $\alpha$. The final panel shows the evolution of the error of the corresponding
tail-averaged iterates, revealing rapid initial convergence followed by a
nonvanishing error floor.
\begin{figure}[!htbp]
    \centering
    \includegraphics[width=0.235\textwidth]{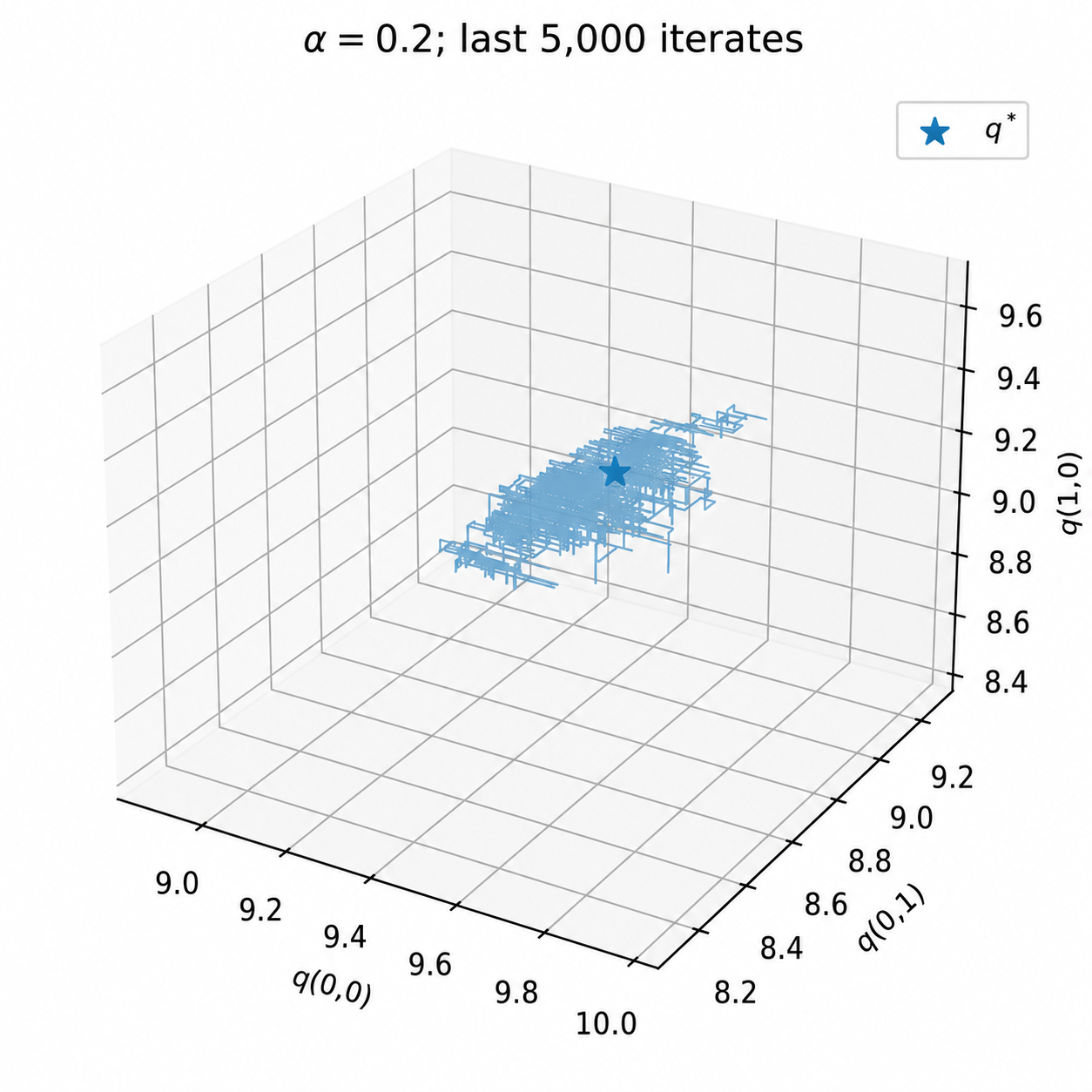}
    \hfill
    \includegraphics[width=0.235\textwidth]{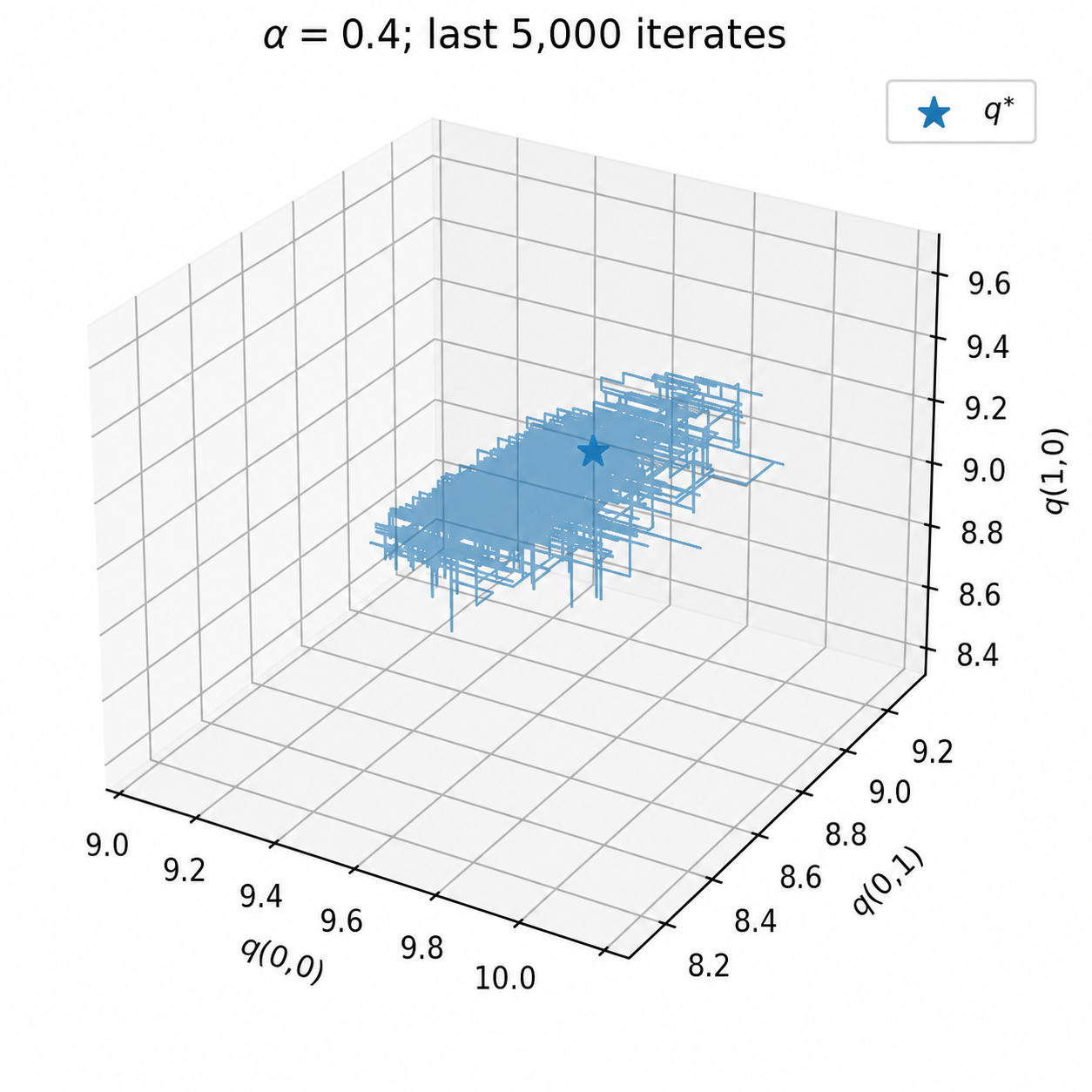}
    \hfill
    \includegraphics[width=0.235\textwidth]{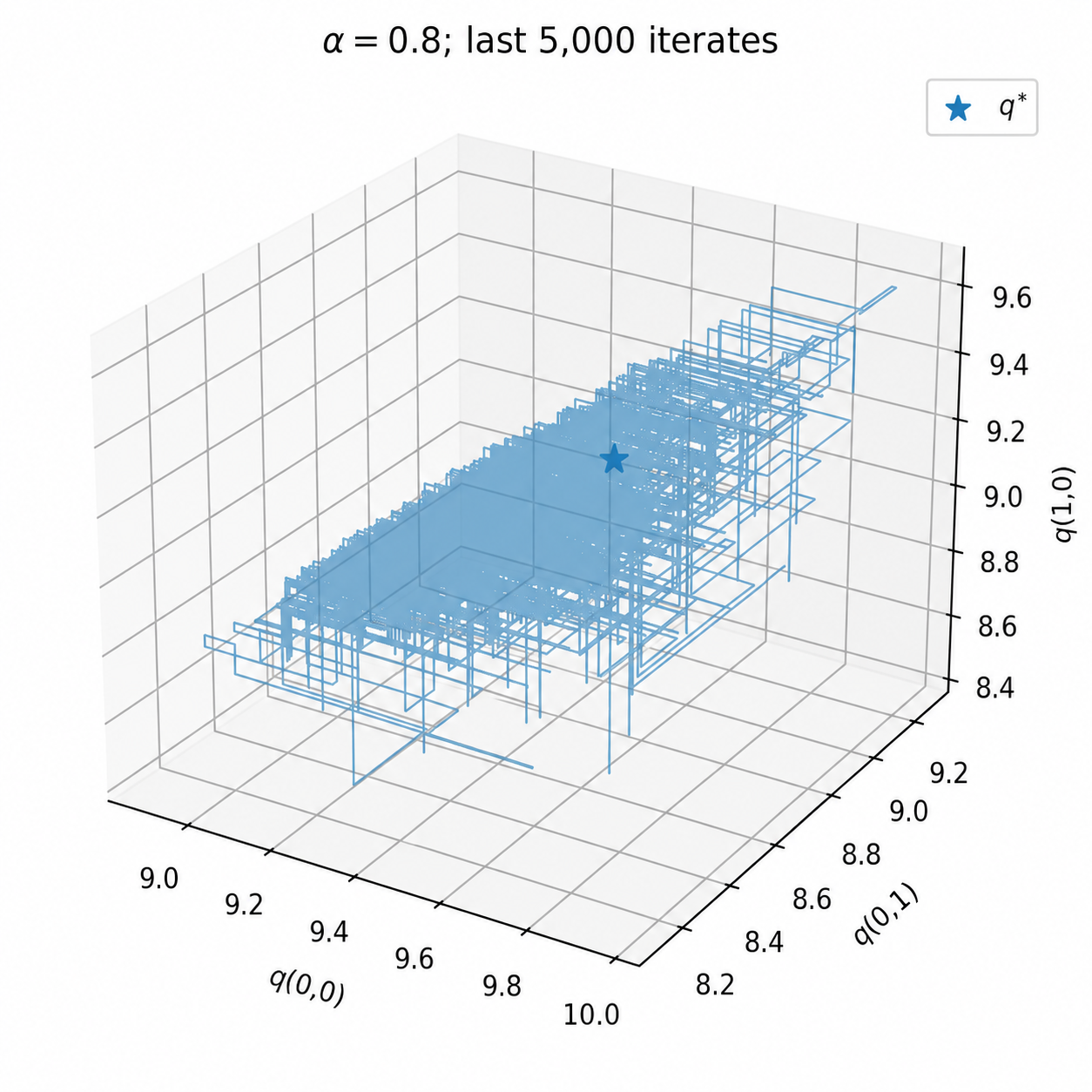}
    \hfill
    \includegraphics[width=0.235\textwidth]{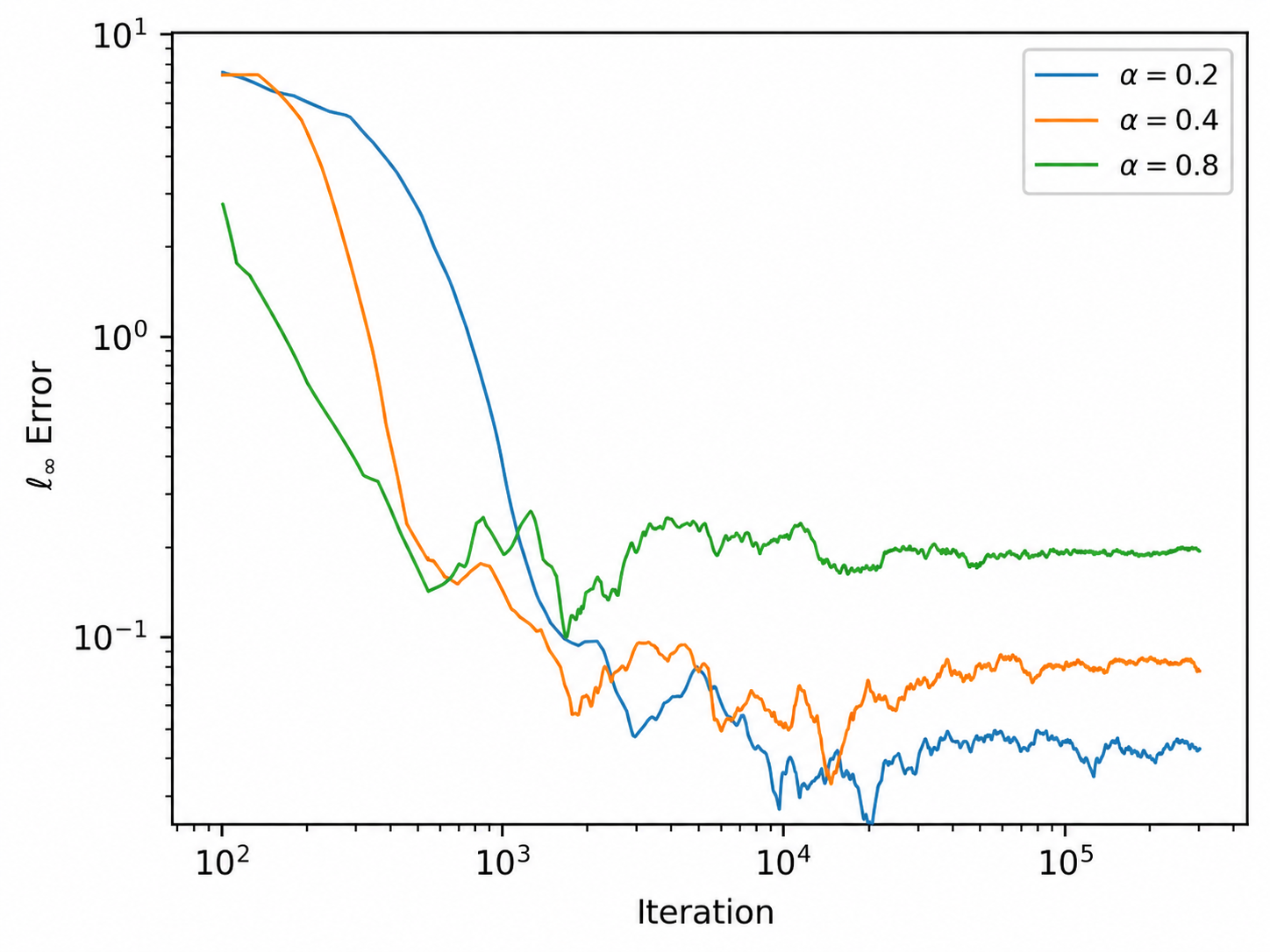}
    
    \caption{
    Long-run behavior and tail-averaging performance of
    constant-stepsize Q-learning. The first three panels display the last
    $5{,}000$ iterates after $300{,}000$ iterations for different stepsizes, with the star indicating
    the optimal $q^*$. Each blue point cloud shows the iterates projected onto the coordinates $q(0,0),q(0,1),q(1,0)$. The final panel plots
    $\bigl\|\bar q_t^{(\alpha)}-q^*\bigr\|_\infty$ on a log--log scale, where
    $\bar q_t^{(\alpha)}$ averages the latter half of the first $t$ iterates.
    }
    \label{fig:q-learning-long-run-tail-average}
\end{figure}

The Q-learning example above illustrates the types of long-run and transient
behavior that can arise in a particular SA algorithm. Such numerical evidence,
however, reveals only the phenomena themselves and does not explain the
mathematical mechanisms underlying them. A rigorous theory should address
questions such as: Why do the iterates continue to fluctuate around the target in the long run? How does the magnitude of these fluctuations depend on the stepsize
$\alpha$? And what quantitative guarantees can be established for a raw 
iterate $\theta_t^{(\alpha)}$ before the long-run regime is reached?
Answering these questions not only deepens our understanding of SA, but can also provide principles for designing more efficient
algorithms and statistical procedures.

A natural way to develop such a theory is to study the stochastic recursion
through the lens of Markov chains. Under suitable conditions on the random
operator $\tmT$ and the noise process $\{x_t\}_{t\ge 0}$, the joint process
$\{(x_t,\theta_t^{(\alpha)})\}_{t\ge 0}$ forms a time-homogeneous Markov chain
that converges geometrically fast to a unique stationary distribution
\citep{dieuleveut2020bridging, huo2023bias}. Thus, the blue clouds in
Figure~\ref{fig:q-learning-long-run-tail-average} can be interpreted as
empirical visualizations of the corresponding long-run distributions of the
iterates. Moreover, this geometric convergence toward stationarity helps explain the
rapid initial decay observed in the final panel of
Figure~\ref{fig:q-learning-long-run-tail-average}. We denote a random pair
drawn from the stationary distribution of the joint chain by
$(x_\infty,\theta_\infty^{(\alpha)})$. A growing literature has developed
fine-grained characterizations of the stationary iterate
$\theta_\infty^{(\alpha)}$
\citep{huo2023bias,zhang2024constant,huo2024collusion}.

A central feature of the stationary iterate $\theta_\infty^{(\alpha)}$ is
that, under a constant stepsize, its mean generally does not coincide with the
target fixed point:
\[
    \E[\theta_\infty^{(\alpha)}] \neq \theta^*.
\]
This phenomenon is also suggested empirically by
Figure~\ref{fig:q-learning-long-run-tail-average}: the blue clouds are not
perfectly centered around the fixed point, and the tail-averaging error in
the final panel incurs a nonvanishing error. The
discrepancy $\E[\theta_\infty^{(\alpha)}]-\theta^*$ is commonly referred to as the \emph{asymptotic bias}. When the mean operator
$\mT$ is locally differentiable around $\theta^*$, under additional regularity conditions, existing analyses exploit
a Taylor expansion of $\mT$ to show
that the bias admits the expansion
\begin{equation}\label{eq:bias-smooth}
    \E\big[\theta_\infty^{(\alpha)}\big]-\theta^*
    =
    \alpha c + o(\alpha),
\end{equation}
where $c$ is a vector independent of $\alpha$. Such fine-grained
characterizations of the asymptotic bias provide a foundation for principled
bias-reduction schemes 
\citep{dieuleveut2020bridging,huo2023bias,zhang2024constant,huo2024collusion}.

Despite the recent advance, our understanding of the stationary law of
$\theta_\infty^{(\alpha)}$ remains incomplete. In particular, two fundamental
questions remain open.

\begin{question}\label{q:smooth}
When the mean operator $\mT$ is locally differentiable, can one go beyond the  bias expansion~\eqref{eq:bias-smooth} and obtain a fine-grained characterization of the full distribution of $\theta_\infty^{(\alpha)}$? In particular, under an appropriate scaling, does the stationary distribution admit a Gaussian approximation, and can the approximation error be quantified nonasymptotically? Moreover, can such a  Gaussian approximation be leveraged to obtain a quantitative Gaussian approximation for the raw iterate $\theta_t^{(\alpha)}$ at finite time?
\end{question}

\begin{question}\label{q:nonsmooth}
When the mean operator $\mT$ is locally nondifferentiable, can one still obtain an asymptotic bias characterization analogous to \eqref{eq:bias-smooth}? If so, what structural conditions suffice, and how does the leading-order bias scale with the stepsize $\alpha$? In particular, does the linear-in-$\alpha$ scaling persist in the locally nondifferentiable setting?
\end{question}

To address these two questions, we develop a general theory of
\emph{steady-state convergence} and show that it provides a unified framework for analyzing both the locally  differentiable and locally nondifferentiable regimes described above.
\subsection{Steady-State Convergence}

Recently, in the setting where the noise sequence $\{x_t\}_{t\ge 0}$ is
i.i.d., a growing line of work has developed a theory for the
steady-state behavior of constant-stepsize SA
\citep{chen2022stationary,zhang2024prelimit,wang2026steady}.
These works consider the diffusion-scaled iterates and their stationary
counterpart,
\[
Y_t^{(\alpha)} := (\theta_t^{(\alpha)}-\theta^*)/\sqrt{\alpha},
\qquad
Y_\infty^{(\alpha)} := (\theta_\infty^{(\alpha)}-\theta^*)/\sqrt{\alpha}.
\]
The central goal is to establish the \emph{steady-state convergence (SSC)} of $\{Y^{(\alpha}_{\infty}\}_{\alpha}$:
\[
    Y_\infty^{(\alpha)}
    \Rightarrow
    Y_\infty,
    \qquad
    \text{as } \alpha\downarrow 0,
\]
for a limiting random variable $Y_\infty$. This limit provides a
distributional characterization of the stationary fluctuations of
$\theta_\infty^{(\alpha)}$ around $\theta^*$ at their natural
$\sqrt{\alpha}$ scale. As illustrated by the red route in
Figure~\ref{fig:limit-plot}, steady-state convergence corresponds to first
letting $t\to\infty$ for a fixed stepsize $\alpha$ and then sending
$\alpha\downarrow0$.

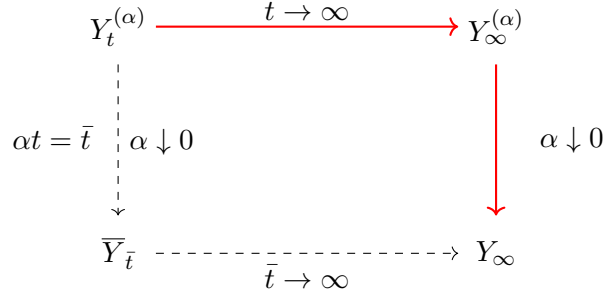
\begin{figure}[htbp!]
    \centering
     \begin{adjustbox}{max width=\linewidth}
\begin{tikzpicture}%
     \draw[dashed,->] (0,2.5)--(0,0.5);
      \draw[dashed, ->] (0.5,0)--(4.5,0);
       \draw[red, thick, ->] (0.5,3)--(4.5,3);
        \draw[red, thick, ->] (5,2.5)--(5,0.5);
    \draw (0,3) node {$Y_t^{(\alpha)}$};
    \draw (5,3) node {$Y^{(\alpha)}_\infty$};
    \draw (2.5,3.2) node {$t\to\infty $};
     \draw (6,1.5) node {$\alpha\downarrow0 $};
     \draw (-0.2,1.5) node {$\alpha t = \bar{t}\quad$ $\alpha \downarrow 0$};
    \draw (0,0) node {$\overline{Y}_{\bar{t}}$};
    \draw (2.5 ,-0.3) node {$\bar{t} \to \infty$};
    \draw (5,0) node {$Y_\infty$};
    \end{tikzpicture}
\end{adjustbox}
    \caption{Two methods for studying steady-state convergence. This paper considers the red route; the dashed route is the classical interchange-of-limits method via the SDE limit $\overline Y_{\bar t}$.
    }
    \label{fig:limit-plot}
\end{figure}

The implications of such an SSC result are twofold.
First, under the global differentiability conditions imposed in prior work, the limiting random variable $Y_\infty$ has been shown to be
Gaussian \citep{chen2022stationary} in some cases. More recently,
\cite{wang2026steady} establish explicit rates for Gaussian approximation.
Such quantitative Gaussian approximations provide a principled reference
distribution for statistical inference on the raw iterate error
$\theta_t^{(\alpha)}-\theta^*$ by controlling the discrepancy between the
transient diffusion-scaled iterate $Y_t^{(\alpha)}$ and its steady-state
Gaussian limit $Y_\infty$. Although \cite{wang2026steady} also treat
Markovian data, their analysis in both the i.i.d.\ and Markovian settings
requires an additive-noise structure of the form
$\tmT(x,\theta)=\mT(\theta)+h(x)$, and the resulting Gaussian approximation
rates are generally suboptimal. These limitations leave open whether sharp
quantitative Gaussian approximations can be established  under substantially weaker structural assumptions.

Second, when \(\mT\) is locally nondifferentiable at \(\theta^*\), classical
bias analyses based on linearization or Taylor expansion of the mean operator
are no longer directly applicable. Nevertheless, \cite{zhang2024prelimit}
establishes a general SSC result for a broad class of mean operators that are
one-sided directionally differentiable at \(\theta^*\), encompassing several
widely used nondifferentiable algorithms. 
This general SSC need not yield either an explicit characterization
of the steady-state limit or a quantitative convergence rate. Importantly, once
SSC is established, the asymptotic bias can still be characterized through
the steady-state limit \(Y_\infty\):
\begin{equation}\label{eq:bias-nonsmooth}
    \E\bigl[\theta_\infty^{(\alpha)}\bigr]-\theta^*
    =
    \sqrt{\alpha}\,\E[Y_\infty]
    +
    o(\sqrt{\alpha}).
\end{equation}
Moreover, \cite{zhang2024prelimit} shows that, in genuinely
nondifferentiable settings, the limiting mean \(\E[Y_\infty]\) can be
nonzero, giving rise to a leading-order bias of order \(\sqrt{\alpha}\).
This stands in sharp contrast to the order-\(\alpha\) bias arising in the
smooth settings admitting \eqref{eq:bias-smooth}. Such a characterization also enables
principled bias-reduction procedures for estimating \(\theta^*\), analogous
to those developed from \eqref{eq:bias-smooth} in the smooth setting.
However, establishing an analogous general SSC result for SA under Markovian
noise remains an open problem.

The existing SSC results discussed above are summarized
in Table~\ref{tab:ssc-literature}. Here, $\W_p$ denotes the Wasserstein
distance of order $p$, and $\law(X)$ denotes the probability law of a random
variable $X$. We defer the formal definition of $\W_p$ to
Section~\ref{sec:notation}.

\begin{table}[htbp!]
\centering

\definecolor{tablerule}{RGB}{120,120,120}
\arrayrulecolor{tablerule}
\setlength{\arrayrulewidth}{0.8pt}

\renewcommand{\arraystretch}{1.25}
\setlength{\tabcolsep}{5pt}

\footnotesize

\begin{tabularx}{0.94\textwidth}{
    |>{\centering\arraybackslash}p{0.18\textwidth}
    |>{\centering\arraybackslash}X
    |>{\centering\arraybackslash}p{0.26\textwidth}|
}
\hline

\bfseries Regime
&
\bfseries I.i.d.\ Noise
&
\bfseries Markovian Noise
\\[0.15em]

\hline

\begin{minipage}[c]{\linewidth}
\centering
$\mT$ is globally continuously differentiable. Noise is
additive.

\end{minipage}
&
\begin{minipage}[c]{\linewidth}
\centering
\[
\W_1\!\left(
    \mathcal{L}\!\left(Y^{(\alpha)}_\infty\right),
    \mathcal{N}(0,V)
\right)
\in
\mathcal O\!\left(
    \sqrt{\alpha}\log(1/\alpha)
\right).
\]
\citep{wang2026steady}
\end{minipage}
&
\begin{minipage}[c]{\linewidth}
\centering
Same rate as in the i.i.d.\ setting.\\
\citep{wang2026steady}
\end{minipage}
\\[0.4em]

\hline

\begin{minipage}[c]{\linewidth}
\centering
$\mT$ is one-sided directionally differentiable at \(\theta^*\).
\end{minipage}
&
\begin{minipage}[c]{\linewidth}
\centering
\[
\lim_{\alpha\downarrow 0}
\W_2\!\left(
    \mathcal{L}\!\left(Y^{(\alpha)}_\infty\right),
    \mathcal{L}\!\left(Y_\infty\right)
\right)
=
0.
\]
\citep{zhang2024prelimit}
\end{minipage}
&
\begin{minipage}[c]{\linewidth}
\centering
{\Large $\diagup$}
\end{minipage}
\\[0.4em]

\hline
\end{tabularx}

\caption{Existing steady-state convergence results. $V$ denotes a problem-dependent asymptotic covariance matrix, defined in Section~\ref{sec:differentiable}.}
\label{tab:ssc-literature}
\end{table}
Therefore, SSC provides a key tool for addressing
Questions~\ref{q:smooth} and~\ref{q:nonsmooth}. However, to the best of our
knowledge, there is no general and sharp SSC theory for
contractive SA under Markovian and multiplicative noise. This gap is particularly consequential
in practice, as several widely used algorithms fall outside the scope of
existing results. For instance, Markovian linear SA \citep{srikant2019finite}, which arises naturally in
reinforcement learning and stochastic control, typically involves
multiplicative noise, while Q-learning \citep{Watkins92-QLearning} combines Markovian data with a Bellman
optimality operator that may be locally nondifferentiable at the fixed point.
Neither setting is covered by the existing SSC results
discussed above. In this work, we develop a general and sharp SSC theory for contractive SA that accommodates both settings, thereby
providing a unified framework for analyzing these practically important
algorithms.

\subsection{Our Contributions}\label{sec:contribution}
\noindent\textbf{Multi-Step Universality Reduction.}
Our first contribution is a \emph{multi-step universality reduction}
framework for establishing SSC. The main idea is the
following. To prove the SSC $\lim_{\alpha\downarrow 0}
\W_2\!\left(
    \law\!\left(Y_\infty^{(\alpha)}\right),
    \law\!\left(Y_\infty\right)
\right)
=
0,$ we introduce an auxiliary diffusion-scaled steady-state
$\mathcal Y_\infty^{(\alpha)}$, derived from a simpler SA recursion, whose SSC
is easier to analyze. We then quantify the discrepancy between the original and
auxiliary steady states in $\W_2$, namely, $\W_2\!\left(
    \law\!\left(Y_\infty^{(\alpha)}\right),
    \law\!\left(\mathcal Y_\infty^{(\alpha)}\right)
\right).$ For general contractive SA, we show that this simplification can be carried
out through a sequence of universality reductions, as summarized by the first
two arrows in each branch of Figure~\ref{fig:proof-roadmap}.

We first establish an \emph{additive-noise universality reduction} to an
auxiliary steady state $A_\infty^{(\alpha)}$, obtained by replacing the
original random operator $\tmT(x,\theta)$ with its additive-noise
counterpart
\[
\mT(\theta)+h(x),
\qquad\mbox{where }
h(x):=\tmT(x,\theta^*)-\theta^*.
\]
This reduction incurs an $\mathcal O(\sqrt{\alpha})$ error in $\W_2$
(Proposition~\ref{prop:additive}). Starting from $A_\infty^{(\alpha)}$, we then develop two further universality
reductions.

\begin{itemize}
    \item 
    When $\mT$ admits a local quadratic linearization at $\theta^*$, we establish a
    \emph{Jacobian-drift universality reduction} to
    $B_\infty^{(\alpha)}$ by replacing the mean operator $\mT(\theta)$ with
    its linearization 
    \[
        \theta^* + J(\theta-\theta^*),
        \qquad\mbox{where }
        J=\nabla \mT(\theta^*).
    \]
    This reduction also incurs an $\mathcal O(\sqrt{\alpha})$ error in $\W_2$
    (Proposition~\ref{prop:Jacobian}). The resulting recursion defining
    $B_\infty^{(\alpha)}$ has a linear drift and additive noise, and is
    therefore much more amenable to SSC analysis.

    \item For a general contractive  \(\mT\), we establish an
\emph{independent-Gaussian-noise universality reduction} to
\(Z_\infty^{(\alpha)}\) by replacing the additive noise sequence
\(\{h(x_t)\}\) with an i.i.d.\ centered Gaussian sequence whose covariance
matches the long-run covariance of \(h(x_t)\). This reduction incurs an error
of order \(\mathcal O(\alpha^{1/4})\) in \(\W_2\);
see Proposition~\ref{prop:ind-Gaussian}. The resulting
\(Z_\infty^{(\alpha)}\) is precisely the diffusion-scaled steady state of an
SA recursion with addtive i.i.d.\ Gaussian noise, for which SSC was established in
\cite{zhang2024prelimit} under one-sided directional differentiability of
\(\mT\) at \(\theta^*\). A key technical ingredient in this universality reduction is a uniform
blockwise Gaussian coupling (Proposition~\ref{prop:uniform}), together with a
\(\W_p\)-decoupling argument (Lemma~\ref{lem:wp-decoupling}), both of which
may be of independent interest.
\end{itemize}
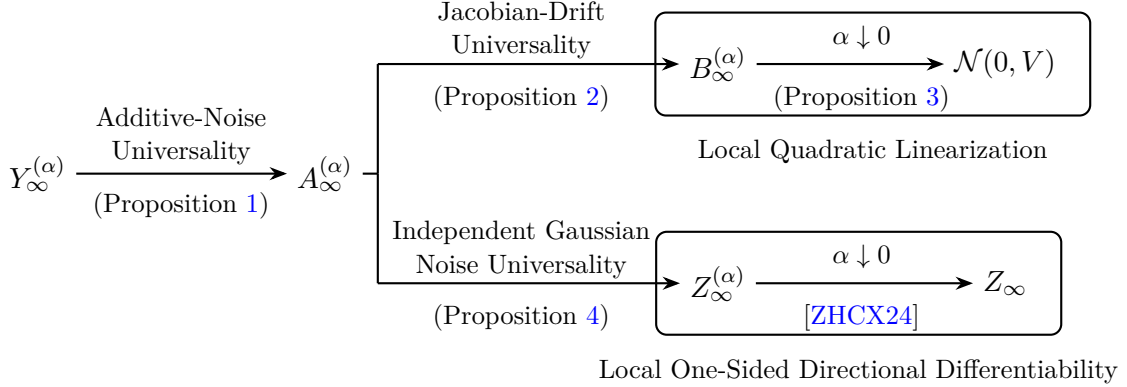
\begin{figure}[htbp!]
\centering
\begin{adjustbox}{max width=\linewidth}
\begin{tikzpicture}[
    >=Stealth,
    state/.style={font=\normalsize},
    labeltext/.style={font=\small, align=center},
    arrow/.style={->, thick},
    connector/.style={thick},
    branchbox/.style={draw, thick, rounded corners=4pt, inner sep=9pt}
]

\node[state] (Y) at (-1.4,0) {$Y_\infty^{(\alpha)}$};
\node[state] (A) at (2.4,0) {$A_\infty^{(\alpha)}$};

\coordinate (F)  at (3.1,0);
\coordinate (Fu) at (3.1,1.45);
\coordinate (Fd) at (3.1,-1.45);

\node[state] (Balpha) at (7.6,1.45) {$B_\infty^{(\alpha)}$};
\node[state] (B)      at (11.4,1.45) {$\mathcal N(0,V)$};

\node[state] (Zalpha) at (7.6,-1.45) {$Z_\infty^{(\alpha)}$};
\node[state] (Z)      at (11.4,-1.45) {$Z_\infty$};

\draw[arrow] (Y) -- (A);
\node[labeltext] at (0.5,0.50) {Additive-Noise\\Universality};
\node[labeltext] at (0.5,-0.42) {(Proposition~\ref{prop:additive})};

\draw[connector] (A) -- (F);
\draw[connector] (F) -- (Fu);
\draw[connector] (F) -- (Fd);

\draw[arrow] (Fu) -- (Balpha);
\node[labeltext] at (5,1.90) {Jacobian-Drift\\Universality};
\node[labeltext] at (5,1.00) {(Proposition~\ref{prop:Jacobian})};

\draw[arrow] (Fd) -- (Zalpha);
\node[labeltext] at (5,-1.00) {Independent Gaussian\\Noise Universality};
\node[labeltext] at (5,-1.90) {(Proposition~\ref{prop:ind-Gaussian})};

\draw[arrow] (Balpha) -- (B);
\node[labeltext] at (9.5,1.80) {$\alpha\downarrow 0$};
\node[labeltext] at (9.5,1.00) {(Proposition~\ref{prop:geometric-gaussian})};

\draw[arrow] (Zalpha) -- (Z);
\node[labeltext] at (9.5,-1.10) {$\alpha\downarrow 0$};
\node[labeltext] at (9.5,-1.90) {\citep{zhang2024prelimit}};

\node[branchbox, fit=(Balpha)(B)] (smoothbox) {};
\node[branchbox, fit=(Zalpha)(Z)] (nonsmoothbox) {};

\node[labeltext, below=5pt of smoothbox] {Local Quadratic Linearization};
\node[labeltext, below=5pt of nonsmoothbox] {Local One-Sided Directional Differentiability};

\end{tikzpicture}
\end{adjustbox}
\caption{Proof roadmap for the multi-step universality reductions.}
\label{fig:proof-roadmap}
\end{figure}
\noindent\textbf{Steady-State Convergence.}
Building on the multi-step universality reductions described above, we
establish SSC for contractive SA in two regimes. Our results are
summarized in Table~\ref{tab:ssc}, with the main new results highlighted in
red.

\begin{itemize}
    \item 
    When $\mT$ admits a local quadratic linearization at $\theta^*$, we first establish a
    quantitative Gaussian approximation for the auxiliary steady-state
    $B_\infty^{(\alpha)}$: $\W_2\!\left(
            \law\!\left(B_\infty^{(\alpha)}\right),
            \mathcal N(0,V)
        \right)
        =
        \mathcal O(\sqrt{\alpha}),$ where $V$ is the unique solution to the Lyapunov equation determined by
    the Jacobian
    $J$ and the long-run covariance of the 
    noise $h(x_t)$
    (Proposition~\ref{prop:geometric-gaussian}). Combining this result with the additive-noise and Jacobian-drift universality
reductions yields the steady-state Gaussian approximation for the
original SA recursion (Theorem~\ref{thm:steady-state-gaussian-approx}):
\[
    \W_2\!\left(
        \law\!\left(Y_\infty^{(\alpha)}\right),
        \mathcal N(0,V)
    \right)
    =
    \mathcal O(\sqrt{\alpha}).
\]
    This substantially strengthens existing steady-state Gaussian approximation
results in several respects: we require a quadratic linearization remainder
only at $\theta^*$, with no differentiability assumption elsewhere; we allow general
multiplicative noise instead of imposing an additive-noise structure; we
obtain the optimal $\mathcal O(\sqrt{\alpha})$ rate without the additional
logarithmic factor appearing in prior work; and our result holds in the
stronger $\W_2$ metric rather than $\W_1$. As a byproduct, the steady-state Gaussian approximation also yields a
finite-time Gaussian approximation for the raw iterates
\(\theta_t^{(\alpha)}\) of general contractive SA
(Corollary~\ref{cor:raw-Gaussian}), whereas existing results on raw-iterate
Gaussian approximation are restricted to independent data
\cite{wei2026gaussian,haque2026accurately}.

    \item 
   When $\mT$ is one-sided
directionally
differentiable at $\theta^*$, the steady-state
convergence theory for i.i.d.\ data developed in
\cite{zhang2024prelimit} can be lifted to the Markovian setting through our
universality reductions. In particular, as illustrated in
Figure~\ref{fig:proof-roadmap}, the auxiliary i.i.d.\ Gaussian-noise recursion
satisfies $\lim_{\alpha\downarrow 0}
    \W_2\!\left(
        \law\!\left(Z_\infty^{(\alpha)}\right),
        \law(Z_\infty)
    \right)
    =
    0.$ Applying the triangle inequality and combining with the independent-Gaussian-noise universality reduction shows that $Z_\infty$ is also the steady-state
limit of the original Markovian SA. We therefore identify
$Y_\infty:=Z_\infty$ and obtain
\[
    \lim_{\alpha\downarrow 0}
    \W_2\!\left(
        \law\!\left(Y_\infty^{(\alpha)}\right),
        \law(Y_\infty)
    \right)
    =
    0.
\]
This extends the general SSC result established for i.i.d.\ data in
\cite{zhang2024prelimit} to the Markovian-noise setting. Moreover, the same
criterion developed in \cite{zhang2024prelimit} can be used to characterize
when \(\E[Y_\infty]\neq 0\), yielding the
\(\sqrt{\alpha}\)-order asymptotic bias expansion in
\eqref{eq:bias-nonsmooth} and enabling principled bias-reduction procedures
under Markovian noise.
\end{itemize}

\begin{table}[htbp!]
\centering

\definecolor{tablerule}{RGB}{120,120,120}
\definecolor{resultred}{RGB}{190,0,0}
\arrayrulecolor{tablerule}
\setlength{\arrayrulewidth}{0.8pt}

\renewcommand{\arraystretch}{1.25}
\setlength{\tabcolsep}{5pt}

\footnotesize

\begin{tabularx}{0.94\textwidth}{
    |>{\centering\arraybackslash}p{0.18\textwidth}
    |>{\centering\arraybackslash}X
    |>{\centering\arraybackslash}p{0.26\textwidth}|
}
\hline

\bfseries Regime
&
\bfseries I.i.d.\ Noise
&
\bfseries Markovian Noise
\\[0.15em]

\hline

\begin{minipage}[c]{\linewidth}
\centering
\textcolor{resultred}{$\mT$ admits a local quadratic linearization at $\theta^*$.}
\end{minipage}
&
\begin{minipage}[c]{\linewidth}
\centering
\textcolor{resultred}{
\[
\W_2\!\left(
    \law\!\left(Y_\infty^{(\alpha)}\right),
    \mathcal N(0,V)
\right)
=
\mathcal O(\sqrt{\alpha}).
\]
}
\end{minipage}
&
\begin{minipage}[c]{\linewidth}
\centering
\textcolor{resultred}{
Same $\mathcal O(\sqrt{\alpha})$ rate as in the i.i.d.\ setting.
}
\end{minipage}
\\[0.4em]

\hline

\begin{minipage}[c]{\linewidth}
\centering
$\mT$ is one-sided directionally differentiable at \(\theta^*\).
\end{minipage}
&
\begin{minipage}[c]{\linewidth}
\centering
\[
\lim_{\alpha\downarrow 0}
\W_2\!\left(
    \law\!\left(Y_\infty^{(\alpha)}\right),
    \law(Y_\infty)
\right)
=
0.
\]
\citep{zhang2024prelimit}
\end{minipage}
&
\begin{minipage}[c]{\linewidth}
\centering
\textcolor{resultred}{
Same result as in the i.i.d.\ setting.
}
\end{minipage}
\\[0.4em]

\hline
\end{tabularx}

\caption{Results on steady-state convergence in different operator regimes under different noise models. Red entries highlight the results established in this work.}
\label{tab:ssc}
\end{table}
\noindent\textbf{Applications to Markovian Linear SA and Q-Learning.}
We further apply our SSC theory to two prototypical examples: Markovian
linear SA and Q-learning. For Markovian linear SA, the mean operator is
affine and hence satisfies local quadratic linearization with zero remainder. Our theory therefore provides, to the best of our knowledge, the first steady-state
Gaussian approximation for Markovian LSA, along with the raw-iterate guarantee.

For asynchronous Q-learning, we uncover a sharp distinction governed by
the uniqueness of the optimal action. When the optimal action is unique at
every state, the Bellman optimality operator is locally affine around
the fixed point, so the mean operator satisfies local quadratic linearization;
in the presence of optimal-action ties, it is generally
locally nondifferentiable and the asymptotic bias can be of order
\(\sqrt{\alpha}\). Our theory covers both regimes, providing a unified
steady-state analysis of asynchronous Q-learning and, to the best of our
knowledge, the first such characterization in the literature. We further
propose a unified extrapolation scheme that cancels the leading-order bias
without prior knowledge of whether optimal-action ties are present. Numerical
experiments in Section~\ref{sec:experiments} further illustrate and support
these theoretical findings.
\subsection{Related Approaches to Steady-State Convergence}
SSC is a classical problem in stochastic dynamical
systems, exemplified by queueing networks \citep{GamaZeev2006}. In the queueing network problems, one typically considers a
sequence of systems approaching a heavy-traffic regime. After appropriate
space--time scaling, the queue-length process is first shown to converge over
finite time horizons to a limiting diffusion, often a reflected diffusion.
SSC asks whether the stationary distributions of the
prelimit queueing systems converge, under the same scaling, to the stationary
distribution of this limiting diffusion. Equivalently, one seeks to justify
an \emph{interchange of limits}
\citep{GamaZeev2006,Gurv2014a,YeYao2016,YeYao2018}:
the heavy-traffic limit and the long-time limit can be taken in either order.
In the terminology of Figure~\ref{fig:limit-plot}, this amounts to showing
that the solid and dashed paths lead to the same limiting distribution. For the general contractive SA setting considered here, however, such a route is not directly
available. In particular, it is not clear  whether the
diffusion-scaled transient process admits an SDE limit
$\overline{Y}_{\bar t}$, especially in the presence of Markovian noise, let alone whether the corresponding interchange of
limits can be justified.

Another classical route to SSC is based on the
\emph{Basic Adjoint Relationship} (BAR) associated with the generator of the
underlying Markov process. At stationarity, BAR characterizes the invariant
distribution through identities satisfied by suitable test functions.
When combined with exponential test functions, this approach can yield
convergence of moment generating functions and, consequently, weak
convergence of the stationary distributions
\citep{BravDaiMiya2017,BravDaiMiya2023,chen2022stationary}.
However, even in the comparatively simple setting of i.i.d.\ noise and a
globally differentiable mean operator $\mT$, existing BAR-based analyses for
SA require additional regularity and structural assumptions that can be
difficult to verify in general \citep{chen2022stationary}. These limitations
become even more pronounced when $\mT$ is only locally nondifferentiable and
the noise process is Markovian.

The $\sqrt{\alpha}$ scaling also suggests a natural connection with Langevin
diffusions and the extensive nonasymptotic theory of the Unadjusted Langevin
Algorithm (ULA)
\citep{durmus2017nonasymptotic,durmus2019high}. However, the rescaled SA recursion
$Y_t^{(\alpha)}$ reduces to a ULA-type recursion only under restrictive
structural conditions, as discussed in
\cite{zhang2024prelimit}. In particular, such a reduction requires additive
Gaussian noise, together with a gradient-field structure and an appropriate
positive-homogeneity property for the local dynamics induced by $\mT$ around
$\theta^*$. Consequently, standard ULA techniques are not directly applicable
to the general contractive SA setting considered here.

\subsection{ Notation}\label{sec:notation}
We write $\mathbb{S}^{d-1} := \{\theta \in \mathbb{R}^d : \|\theta\| = 1\}$ for the unit sphere in $\mathbb{R}^d$, where $\|\cdot\|$ denotes the Euclidean norm. For a matrix $A\in\mathbb{R}^{d\times d}$, we denote its spectral radius by
$\rho(A)$.
We use \(\gamma\) for the contraction modulus of the mean operator \(\mT\)
and \(J\) for its Jacobian at the fixed point, whenever this Jacobian exists.
In the Q-learning application, \(\rho\) denotes the MDP discount factor;
\(\rho(A)\) continues to denote spectral radius and is distinguished by its
matrix argument.
 We use $\overline{\B}_d(\theta, \epsilon)$ to denote an closed ball in $\R^d$ centering at $\theta \in \R^d$ with radius $\epsilon>0$ with respect to $\|\cdot\|$. 
 
Let $(\Omega,\mathcal F,\mathbb P)$ be a probability space and $(\mathcal M,d)$ a
Polish metric space. For a Borel measurable random variable
$Y:\Omega\to\mathcal M$, we write $\mathcal L(Y)$ for the law (distribution) of
$Y$. Let $\mathcal P(\mathcal M)$ denote the set of Borel probability measures on
$\mathcal M$. For $p\in[1,\infty)$, define the $p$-moment class
\[
\mathcal P_p(\mathcal M)
:=\Bigl\{\mu\in\mathcal P(\mathcal M): \int_{\mathcal M} d(x,x_0)^p\,\mu(\mathrm dx)
<\infty \text{ for some (equivalently any) } x_0\in\mathcal M\Bigr\}.
\]
For $\mu,\nu\in\mathcal P(\mathcal M)$, let $\Pi(\nu,\mu)$ be the collection of
all couplings of $(\nu,\mu)$, i.e., probability measures $\pi$ on
$\mathcal M\times\mathcal M$ whose marginals are $\nu$ and $\mu$. The
Wasserstein--$p$ distance on $(\mathcal M,\mathsf d)$ is defined by
\[
\mathcal W_{p,\mathsf d}(\nu,\mu)
:=
\inf_{\pi\in\Pi(\nu,\mu)}
\left(
\int_{\mathcal M\times\mathcal M} \mathsf d(x,y)^p \,\pi(\mathrm dx,\mathrm dy)
\right)^{1/p}, \qquad \forall\nu,\mu\in\mathcal P_p(\mathcal M).
\]
When $\mathcal M=\mathbb R^d$ and $\mathsf d(x,y)=\|x-y\|$ is the Euclidean metric, we abbreviate
$\mathcal W_{p,d}$ as $\mathcal W_p$:
\[
\mathcal W_p(\nu,\mu)
:=
\inf_{\pi\in\Pi(\nu,\mu)}
\left(
\int_{\mathbb R^d\times\mathbb R^d} \|x-y\|^p \,\pi(\mathrm dx,\mathrm dy)
\right)^{1/p},
\qquad \forall \nu,\mu \in \mathcal P_p(\mathbb R^d).
\]
Throughout the paper, unless otherwise specified, the notation
\(\mathcal O(\cdot)\) and \(\lesssim\) suppresses finite constants that are
independent of \(\alpha\) and \(t\), but may depend on other fixed problem
parameters. The notation \(\widetilde{\mathcal O}(\cdot)\) additionally
suppresses logarithmic factors in \(1/\alpha\).

\subsection{Organization of the Paper}
The remainder of the paper is organized as follows.
Section~\ref{sec:problem} introduces the problem setup and standing
assumptions.
Section~\ref{sec:additive} establishes additive-noise universality, which
serves as the common starting point for our SSC analysis.
Section~\ref{sec:differentiable} develops steady-state and raw-iterate
Gaussian approximations under local quadratic linearization.
Section~\ref{sec:nondifferentiable} establishes general SSC under one-sided
directional differentiability  and discusses its implications for asymptotic bias
characterization and reduction.
Sections~\ref{sec:linearSA} and~\ref{sec:Q-learning} apply the theory to
Markovian linear SA and asynchronous Q-learning, respectively.
Section~\ref{sec:experiments} presents numerical experiments comparing
tail averaging with regime-specific, misspecified, and unified
Richardson--Romberg extrapolations.
We conclude with a discussion of future research directions.
Technical preliminaries, detailed proofs, and additional experimental
details are provided in the appendices.
\section{Problem Setup}\label{sec:problem}
We consider the following constant-stepsize stochastic approximation (SA) recursion:
\begin{equation}\label{eq:update}
\theta_{t+1}^{(\alpha)}
= \theta_t^{(\alpha)}
+ \alpha\Big( \tmT\big(x_t,\theta_t^{(\alpha)}\big)
- \theta_t^{(\alpha)} \Big),
\end{equation}
where \(\alpha>0\) is a fixed stepsize, \(\{x_t\}_{t\ge 0}\) is a stochastic process on \(\mathcal X\) that admits a unique limiting distribution \(\mu\), and
\(\tmT:\mathcal X\times \mathbb R^d\to \mathbb R^d\) is a random operator.
The corresponding mean operator is defined by
\[
\mT(\theta) := \mathbb E_{x\sim \mu}\big[\tmT(x,\theta)\big],
\qquad \theta\in\mathbb R^d .
\]
In this work, we focus on contractive SA recursions satisfying the following assumption.

\begin{assumption}[Contractive SA]\label{assumption:contraction} 
There exist a norm $\|\cdot\|_c$ and a constant
$\gamma\in(0,1)$ such that
\[
    \|\mT(\theta)-\mT(\theta')\|_c
    \leq
    \gamma \|\theta-\theta'\|_c,
    \qquad
    \forall\,\theta,\theta'\in\mathbb{R}^d.
\]
\end{assumption}
In this work, we consider two classes of noise sequences: i.i.d. noise and
Markovian noise. In both settings, we impose the following regularity
conditions on the noise sequence and the random operator. These assumptions
are standard in the analysis of SA; see, for example,
\cite{dieuleveut2020bridging,zhang2024prelimit} for the i.i.d. setting and
\cite{chen2024lyapunov,zhang2024constant,huo2023bias} for the Markovian
setting.

\begin{assumption}[Noise and Operator Regularity]
\label{assumption:noise}
The random operator \(\tmT\) and the noise sequence \(\{x_t\}_{t\ge 0}\)
satisfy one of the following two conditions.

\begin{enumerate}[label=(\roman*), ref=\theassumption(\roman*)]

\item \textbf{I.i.d. noise.}\label{assumption:noise:iid}
The sequence \(\{x_t\}_{t\ge 0}\) is i.i.d. with common distribution \(\mu\).
Moreover, there exists a constant \(L>0\) such that, for all
\(\theta,\theta'\in\mathbb R^d\),
\begin{align*}
\left(
\mathbb E_{x\sim\mu}
\left[
\big\|\tmT(x,\theta)-\tmT(x,\theta')\big\|^4
\right]
\right)^{1/4}
\le L\|\theta-\theta'\|, \qquad \left(
\mathbb E_{x\sim\mu}
\left[
\big\|\tmT(x,\theta^*)\big\|^4
\right]
\right)^{1/4}
\le L .
\end{align*}

\item \textbf{Markovian noise.}\label{assumption:noise:markovian}
The sequence \(\{x_t\}_{t\ge 0}\) is an irreducible and aperiodic Markov chain
on \(\mathcal X\), with transition kernel \(P\) and stationary distribution
\(\mu\). The chain is uniformly ergodic: there exist constants
\(c_{\mathrm{mix}}\ge 0\) and \(\rho_{\mathrm{mix}}\in(0,1)\) such that
\begin{equation}\label{eq:uniform-GE}
\big\|P^t(x,\cdot)-\mu(\cdot)\big\|_{\mathrm{TV}}
\le c_{\mathrm{mix}}\rho_{\mathrm{mix}}^t,
\qquad
\forall x\in\mathcal X,\ \forall t\ge 1 .
\end{equation}
Moreover, there exists a constant \(L>0\) such that, for all
\(\theta,\theta'\in\mathbb R^d\) and \(x\in\mathcal X\),
\begin{equation*}
\big\|\tmT(x,\theta)-\tmT(x,\theta')\big\|
\le L\|\theta-\theta'\|,
\qquad
\big\|\tmT(x,\theta^*)\big\|
\le L .
\end{equation*}

\end{enumerate}
\end{assumption}

We remark that under Assumption~\ref{assumption:noise:iid}, the iterates $\{\theta_t\}_{t \geq 0}$ from \eqref{eq:update} form a time-homogeneous Markov chain; meanwhile, under Assumption~\ref{assumption:noise:markovian}, we should augment the state $\{(\theta_t,x_t)\}_{t \geq 0}$ to define a time-homogeneous Markov chain. Without loss of generosity, we consider the joint process $\{(\theta_t,x_t)\}_{t \geq 0}$ as the induced Markov chain.

We define $\tau_\alpha
    :=
    \min\left\{
        t\geq 1:
        \sup_{x\in\mathcal X}
        \bigl\|P^t(x,\cdot)-\mu(\cdot)\bigr\|_{\mathrm{TV}}
        \leq \alpha
    \right\},$ which denotes the mixing time of the data process to accuracy \(\alpha\).
Under Assumption~\ref{assumption:noise:iid}, we have \(\tau_\alpha=1\), 
whereas under Assumption~\ref{assumption:noise:markovian}, $\tau_\alpha=\mathcal{O}\bigl(\log(1/\alpha)\bigr).$ We  note that the update~\eqref{eq:update} under
Assumptions~\ref{assumption:contraction} and \ref{assumption:noise} already covers many widely used algorithms,
including Markovian linear SA \citep{huo2023bias} and asynchronous Q-learning
\citep{zhang2024constant}. We discuss these two examples in detail in
Sections~\ref{sec:linearSA} and~\ref{sec:Q-learning}, respectively.

Beyond the update~\eqref{eq:update} and Assumption~\ref{assumption:noise},
\cite{chen2024lyapunov} also incorporates an additional martingale-difference
noise sequence $\{w_t\}_{t\ge 0}$ and considers the more general recursion
\begin{equation}\label{eq:updategeneral}
\theta_{t+1}^{(\alpha)}
= \theta_t^{(\alpha)} + \alpha\big(\tmT(x_t,\theta_t^{(\alpha)}) - \theta_t^{(\alpha)} + w_t(\theta_t^{(\alpha)})\big).
\end{equation}
They impose a weaker condition on $\{w_t\}_{t\ge 0}$: for each
$\theta\in\R^d$, the noise satisfies a linear-growth bound $\|w_t(\theta)\|\lesssim\|\theta\|+1$ almost surely, uniformly over $t$. Notably, this assumption does \emph{not} require $w_t(\theta)$ to be uniformly Lipschitz
in $\theta$. Consequently, their recursion generally cannot be rewritten in the
Markovian SA form~\eqref{eq:update} under Assumption~\ref{assumption:noise}.
Moreover, since their goals differ from ours, it is unclear whether the more
general dynamics~\eqref{eq:updategeneral} (or an augmented state incorporating $\{x_t\}_{t \geq 0}$ and 
$\{w_t\}_{t \geq 0}$) defines a Markov chain that admits a well-behaved limiting
distribution---a question that is beyond the scope of this paper.

In this work, we study \emph{steady-state convergence}, where the steady state
refers to the limiting distribution of the induced Markov chain as
$t\to\infty$, whenever such a limit exists. The existence and uniqueness of a
steady-state distribution have been established in a variety of important
settings. For example, \cite{dieuleveut2020bridging,zhang2024prelimit} consider contractive SA with
i.i.d.\ noise, \cite{huo2023bias} study Markovian linear SA,
\cite{zhang2024constant} analyze asynchronous Q-learning, and
\cite{huo2024collusion} consider Markovian nonlinear SA with smooth,
strongly monotone dynamics. Rather than imposing the specific conditions considered
in these individual settings, we encapsulate their common consequence in the
following assumption, tailored to general Markovian contractive SA. As
discussed above, this assumption can be verified in each of the aforementioned
settings.
\begin{assumption}[Geometric Distributional Convergence to Steady State]
\label{assumption:convergence}
For every SA recursion~\eqref{eq:update} under
Assumptions~\ref{assumption:contraction} and~\ref{assumption:noise},
there exists $\alpha_0>0$ such that, for every
$\alpha\in(0,\alpha_0)$, the Markov chain $\{(\theta_t,x_t)\}_{t \geq 0}$ converges to a unique stationary distribution $\bar\nu_\alpha\in\mathcal P_2(\R^d\times \mathcal X)$. 
Let $\law\bigl(\theta_\infty^{(\alpha)}\bigr)$ be the first marginal of $\Bar{\nu}_\alpha$.  Moreover, for any initial distribution
$\law(\theta_0^{(\alpha)})\in\mathcal P_2(\mathbb R^d)$,
\[
    \W_2\!\left(
        \law\bigl(\theta_t^{(\alpha)}\bigr),
        \law\bigl(\theta_\infty^{(\alpha)}\bigr)
    \right)
    \leq
    c_1(1-\alpha c_2)^t,
    \qquad
    t\geq c_3\tau_\alpha,
\]
where \(c_1\) may depend on the initial distribution and $c_1,c_2,c_3>0$ are constants independent of $\alpha$ and $t$.
\end{assumption}
Under Assumption~\ref{assumption:convergence}, we define the diffusion-scaled
steady-state iterate $Y_\infty^{(\alpha)}
    :=
    \frac{\theta_\infty^{(\alpha)}-\theta^*}{\sqrt{\alpha}}.$ Our primary goal is to establish the existence of a limiting random variable
$Y_\infty$ such that
\begin{equation}\label{eq:SSC}
    \lim_{\alpha\downarrow 0}
    \W_2\!\left(
        \law\!\left(Y_\infty^{(\alpha)}\right),
        \law(Y_\infty)
    \right)
    =
    0.
\end{equation}
We further provide regime-specific refinements of
\eqref{eq:SSC}. Under local quadratic linearization, we precisely
characterize $\law(Y_\infty)$ and establish a quantitative convergence rate
for~\eqref{eq:SSC}; see Section~\ref{sec:differentiable}. In the locally
nondifferentiable setting, we show that,
under additional structural conditions, the limiting distribution need not be
centered, so that $\E[Y_\infty]\neq 0$; see
Section~\ref{sec:nondifferentiable}.

\section{Additive-Noise Universality}\label{sec:additive}
Before turning to the SSC analysis in
Sections~\ref{sec:differentiable} and~\ref{sec:nondifferentiable}, we first
establish an \emph{additive-noise universality} result. This universality
reduction serves as the first key step in our analysis and will be used
both under local quadratic linearization and in the locally nondifferentiable
regime.

Consider the following auxiliary SA recursion:
\begin{equation}\label{eq:auxiliary-additive}
a_{t+1}^{(\alpha)}
=
a_t^{(\alpha)}
+
\alpha\left(
\mT\big(a_t^{(\alpha)}\big)
-
a_t^{(\alpha)}
+
h(x_t)
\right),
\qquad t\geq 0,
\end{equation}
where $h(x)
:=
\tmT(x,\theta^*)-\theta^*$ denotes the noise at equilibrium. The recursion in \eqref{eq:auxiliary-additive} has the same mean operator
\(\mT\) as the original recursion \eqref{eq:update}, while its noise is additive. Moreover, \(h\) is centered under the stationary
distribution \(\mu\), since
\[
\mathbb E_{x\sim\mu}[h(x)]
=
\mathbb E_{x\sim\mu}\big[\tmT(x,\theta^*)\big]-\theta^*
=
\mT(\theta^*)-\theta^*
=
0 .
\]
Recursion~\eqref{eq:auxiliary-additive} is a special case of the general SA
recursion~\eqref{eq:update} and satisfies
Assumptions~\ref{assumption:contraction} and~\ref{assumption:noise}.
For all sufficiently small $\alpha>0$, it admits a unique steady-state
law $\law(a_\infty^{(\alpha)})\in\mathcal P_2(\mathbb R^d)$. For this
additive-noise auxiliary, existence and uniqueness follow directly from
synchronous contraction, rather than from an additional application of
Assumption~\ref{assumption:convergence}; the stationary construction and
the marginal $\W_2$ convergence needed below are proved in
Section~\ref{sec:additive-marginal-convergence}.
Analogously to the scaled steady-state variable associated with the original
SA recursion \eqref{eq:update}, we
define the scaled steady-state variable of  \eqref{eq:auxiliary-additive} by $A_\infty^{(\alpha)}
:=
\frac{a_\infty^{(\alpha)}-\theta^*}{\sqrt{\alpha}} .$ The following proposition states the additive-noise universality result.
\begin{proposition}[Additive-Noise Universality]\label{prop:additive}
Under Assumptions \ref{assumption:contraction}, \ref{assumption:noise} and \ref{assumption:convergence}, 
\[
\W_2\Big(
\law\big(Y_\infty^{(\alpha)}\big),
\law\big(A_\infty^{(\alpha)}\big)
\Big)
=
\mathcal O\big(\alpha^{1/2}\big).
\]
\end{proposition}
Proposition~\ref{prop:additive} shows that replacing the original
state-dependent noise by the additive  equilibrium noise \(h(x_t)\) changes the scaled
steady-state law by at most
\(\mathcal{O}(\alpha^{1/2})\) in Wasserstein-2 distance. 

Thus, to establish SSC, it suffices to analyze
\(A_\infty^{(\alpha)}\), whose additive-noise structure is more tractable than
that of \(Y_\infty^{(\alpha)}\). We develop the resulting SSC theory in
Sections~\ref{sec:differentiable} and~\ref{sec:nondifferentiable}, under
local quadratic linearization and, more generally, local one-sided
directional differentiability, respectively. The proof of
Proposition~\ref{prop:additive} is deferred to
Section~\ref{sec:proof-additive}.

\section{Steady-State Gaussian Approximation under Local Quadratic Linearization}
\label{sec:differentiable}

In this section, we focus on the regime in which the mean operator
\(\mT\) admits a local linearization with a quadratic remainder at its fixed
point \(\theta^*\).

\begin{assumption}[Local Quadratic Linearization]
\label{assumption:local-quadratic}
There exist a matrix \(J\in\R^{d\times d}\) and constants
\(\epsilon>0\) and \(L_R\in[0,\infty)\) such that
\[
    \bigl\|
        \mT(\theta^*+u)-\mT(\theta^*)-Ju
    \bigr\|
    \leq L_R\|u\|^2,
    \qquad
    \forall u\in\R^d\text{ with }\|u\|\leq\epsilon.
\]
\end{assumption}

Assumption~\ref{assumption:local-quadratic} implies that \(\mT\) is differentiable
at \(\theta^*\), with \(J=\nabla\mT(\theta^*)\), but does not require
differentiability at other points in a neighborhood of \(\theta^*\).
A sufficient condition is that \(\mT\) has a locally Lipschitz Jacobian
near \(\theta^*\). Assumption~\ref{assumption:local-quadratic} is satisfied by a broad class of
commonly used stochastic algorithms. Examples include:
(i) stochastic gradient descent for minimizing a differentiable, strongly convex
objective whose gradient admits a local quadratic linearization at its minimizer;
(ii) Markovian linear stochastic approximation; and
(iii) asynchronous Q-learning when the optimal policy is unique.
We discuss the latter two examples in detail in
Sections~\ref{sec:linearSA} and~\ref{sec:Q-learning}, respectively.

Our main result in this section establishes a quantitative Gaussian
approximation for the appropriately scaled steady-state distribution under
Assumption~\ref{assumption:local-quadratic}, with an explicit optimal convergence rate
in \(\W_2\). As a direct application, Section~\ref{sec:raw-gaussian} derives a
Gaussian approximation for the raw iterates. We then present the two
main ingredients underlying the proof in
Sections~\ref{sec:Jacobian} and~\ref{sec:geometric-weight}, respectively.

Before stating the main result, we introduce several quantities that will be
used throughout the analysis.  Define the long-run covariance of the stationary equilibrium-noise process
\(\{h(x_t)\}_{t\geq0}\) by
\[
    \Sigma_h
    :=
    \lim_{n\to\infty}
    \frac{1}{n}
    \operatorname{Cov}\!\left(
        \sum_{t=0}^{n-1}h(x_t)
    \right).
\] 
If \(\{x_t\}_{t\geq0}\) is an i.i.d.\ sequence with common distribution
\(\mu\), then
\[
    \Sigma_h
    =
    \E_{\mu}\!\left[h(x)h(x)^\top\right]
    =
    \operatorname{Cov}_{\mu}\!\left(
        \tmT(x,\theta^*)
    \right).
\]
If \(\{x_t\}_{t\geq0}\) is Markovian, we extend it to a two-sided stationary
chain \(\{x_t\}_{t\in\mathbb Z}\), in which case
\[
    \Sigma_h
    =
    \sum_{\ell\in\mathbb Z}
    \E_{\mu}\!\left[
        h(x_0)h(x_\ell)^\top
    \right].
\]
Under Assumption~\ref{assumption:noise:markovian}, uniform 
ergodicity and boundedness of $h$ imply that $\left\|
        \E_\mu\!\left[h(x_0)h(x_\ell)^\top\right]
    \right\|
    \lesssim \rho_{\mathrm{mix}}^{|\ell|},$ for any $\ell\in\mathbb Z$,
where \(\rho_{\mathrm{mix}}\) is as in \eqref{eq:uniform-GE}; see, e.g.,
\cite[Chapter~16]{Meyn12_book}.
Hence the covariance series defining $\Sigma_h$ is absolutely convergent. Finally, let \(V\) denote the unique positive-semidefinite solution to the
Lyapunov equation
\begin{equation}\label{eq:Lyapunov}
(J-I_d)V + V(J-I_d)^\top + \Sigma_h = 0.
\end{equation}
Indeed, Assumption~\ref{assumption:contraction} implies
\(\rho(J)<1\), and hence \(J-I_d\) is Hurwitz, so the solution is uniquely
given by $V
    =
    \int_0^\infty
    e^{(J-I_d)t}\Sigma_h e^{(J-I_d)^\top t}\,\mathrm dt.$ We are now ready to state the steady-state Gaussian approximation theorem.

\begin{theorem}[Steady-State Gaussian Approximation]
\label{thm:steady-state-gaussian-approx}
Under Assumptions~\ref{assumption:contraction},
\ref{assumption:noise}, \ref{assumption:convergence}, and
\ref{assumption:local-quadratic},
\[
        \W_2\Big(
            \law\big(Y_\infty^{(\alpha)}\big),
            \mathcal N(0,V)
        \Big)
         \in
    \mathcal{O}(\sqrt{\alpha}).
    \]
\end{theorem}

\noindent\textbf{Comparison with Prior Work.}
To the best of our knowledge, \cite{wang2026steady} provide the only prior
steady-state Gaussian approximation result in a comparable setting. Their
analysis assumes that $\mT$ is globally continuously differentiable, requires
additive noise, and measures the approximation error in $\W_1$. In contrast,
Theorem~\ref{thm:steady-state-gaussian-approx} requires a local quadratic
linearization at $\theta^*$ without imposing global differentiability,
allows general multiplicative noise, and establishes convergence in the
stronger $\W_2$ metric.

\noindent\textbf{Sharpness of the Rate.}
The \(\sqrt{\alpha}\) rate is sharp in general.
Indeed, existing bias characterizations under additional local smoothness
conditions show that $\E\bigl[\theta_\infty^{(\alpha)}\bigr]-\theta^*
    =
    \alpha c+o(\alpha)$ for some vector $c$ independent of $\alpha$ \citep{dieuleveut2020bridging,zhang2024constant,huo2024collusion,zhang2025piecewise}. Consequently,
\[
    \E\bigl[Y_\infty^{(\alpha)}\bigr]
    =
    \frac{
        \E\bigl[\theta_\infty^{(\alpha)}\bigr]-\theta^*
    }{\sqrt{\alpha}}
    =
    \sqrt{\alpha}c+o\bigl(\sqrt{\alpha}\bigr).
\]
Since $\mathcal N(0,V)$ is centered, the definition of $\W_2$ together with
Jensen's inequality implies that, whenever $c\neq 0$,
\[
    \W_2\!\left(
        \law\bigl(Y_\infty^{(\alpha)}\bigr),
        \mathcal N(0,V)
    \right)
    \geq
    \left\|
        \E\bigl[Y_\infty^{(\alpha)}\bigr]
        -
        \E_{X\sim\mathcal N(0,V)}[X]
    \right\|
    =
    \sqrt{\alpha}\,\|c\|
    +
    o\bigl(\sqrt{\alpha}\bigr).
\]
Thus, the rate in
Theorem~\ref{thm:steady-state-gaussian-approx} is optimal in general. In
particular, it removes the additional \(\log(1/\alpha)\) factor appearing in
\cite{wang2026steady}.

\subsection{Application: Gaussian Approximation for Raw Iterates}
\label{sec:raw-gaussian}

As an immediate consequence of
Theorem~\ref{thm:steady-state-gaussian-approx} and the geometric convergence
of the SA iterates to steady state, we obtain a finite-time Gaussian
approximation for the raw iterates.
\begin{corollary}[Raw Iterate Gaussian Approximation]
\label{cor:raw-Gaussian}
Suppose Assumptions~\ref{assumption:contraction},
\ref{assumption:noise}, \ref{assumption:convergence}, and
\ref{assumption:local-quadratic} hold. Then there exist constants
$c_1,c_2,c_3>0$, independent of $\alpha$ and $t$, such that, for all
sufficiently small $\alpha>0$ and $t\geq c_3\tau_\alpha$,
\[
    \W_2\left(
        \law\left(
            \frac{\theta_t^{(\alpha)}-\theta^*}{\sqrt{\alpha}}
        \right),
        \mathcal N(0,V)
    \right)
    \leq
    \frac{c_1(1-c_2\alpha)^t}{\sqrt{\alpha}}
    +
    \mathcal O\bigl(\sqrt{\alpha}\bigr).
\]
Consequently, for any fixed \(c_4\geq 1/c_2\), setting $\alpha = c_4\log t/t$ yields
\[
    \W_2\left(
        \law\left(
            \frac{\theta_t^{(\alpha)}-\theta^*}{\sqrt{\alpha}}
        \right),
        \mathcal N(0,V)
    \right)
    \in 
    \mathcal O\left(
        \sqrt{\frac{\log t}{t}}
    \right).
\]
\end{corollary}
\noindent\textbf{Comparison with Prior Work.} A closely related result is \cite{wei2026gaussian}, which establishes a
Gaussian approximation bound for the raw iterate of constant-stepsize SGD
under strong convexity and i.i.d.\ data. With
\(\alpha\asymp \log t/t\), their result yields a rate of order
\(\sqrt{\log t/t}\) in the convex-set distance. Corollary~\ref{cor:raw-Gaussian}
achieves the same dependence on \(t\) in $\W_2$ distance for the
broader class of contractive SA considered here, while also allowing
Markovian noise. Since the convex-set and $\W_2$ distances are not
comparable in general, neither result directly implies the other. Extending
our Gaussian approximation guarantees to the convex-set distance is an
interesting direction for future work.

Another closely related result is \cite{haque2026accurately}, which provides
 Gaussian approximation bounds for raw iterates of a more general
class of SA algorithms, without requiring global contraction, but under
independent data and in $\W_1$ distance. Rather than proceeding
through a steady-state Gaussian approximation as we do, they establish the raw-iterate
Gaussian approximation directly. In the regime considered here, however,
their bound contains an error term of order
\(\sqrt{\alpha}\log(1/\alpha)\), which, under
\(\alpha\asymp \log t/t\), yields $\mathcal O\left((\log t)^{3/2}/\sqrt t\right)$. In contrast, our optimal \(\mathcal O(\sqrt{\alpha})\) steady-state Gaussian approximation
leads to the sharper rate $\mathcal O(\sqrt{\log t/t})$ in the stronger $\W_2$ distance.

\subsection{Step 1: Jacobian-Drift Universality}
\label{sec:Jacobian}

The proof of Theorem~\ref{thm:steady-state-gaussian-approx} proceeds by
establishing a quantitative Gaussian approximation for
\(A_\infty^{(\alpha)}\). The argument
consists of two main steps. In this subsection, we develop the first step,
which replaces the mean operator \(\mT\) by its linearization at the fixed
point \(\theta^*\).

More precisely, we introduce an auxiliary SA recursion driven by the same
additive equilibrium-noise process \(\{h(x_t)\}_{t\geq0}\), but with \(\mT\) replaced by
the linear operator $\mT_{\mathrm{lin}}(\theta)
    :=
    \theta^*+J(\theta-\theta^*),$ where $J:=\nabla T(\theta^*)$. The resulting recursion is
\begin{equation}
\label{eq:auxiliary-Jacobian}
    b_{t+1}^{(\alpha)}
    =
    b_t^{(\alpha)}
    +
    \alpha\Big(
        (J-I_d)\bigl(b_t^{(\alpha)}-\theta^*\bigr)
        +
        h(x_t)
    \Big),
    \qquad t\geq 0.
\end{equation}
Recursion~\eqref{eq:auxiliary-Jacobian} is a special case of the general SA
recursion~\eqref{eq:update} and satisfies
Assumptions~\ref{assumption:contraction} and~\ref{assumption:noise}
with the same contraction norm \(\|\cdot\|_c\) and contraction factor
\(\gamma\); see, e.g., \cite[Proposition~6.4]{arora2021alternative}.
Consequently, Assumption~\ref{assumption:convergence} implies that, for all
sufficiently small \(\alpha>0\),
recursion~\eqref{eq:auxiliary-Jacobian} admits a unique steady-state
distribution $\law\bigl(b_\infty^{(\alpha)}\bigr)
    \in \mathcal{P}_2(\mathbb{R}^d).$ We define its scaled steady-state variable by $B_\infty^{(\alpha)}
    :=
    \frac{b_\infty^{(\alpha)}-\theta^*}{\sqrt{\alpha}}.$ The following proposition establishes the Jacobian-drift universality result.

\begin{proposition}[Jacobian-Drift Universality]
\label{prop:Jacobian}
Under Assumptions~\ref{assumption:contraction},
\ref{assumption:noise}, \ref{assumption:convergence}, and
\ref{assumption:local-quadratic},
\[
    \W_2\left(
        \law\bigl(A_\infty^{(\alpha)}\bigr),
        \law\bigl(B_\infty^{(\alpha)}\bigr)
    \right)
   \in
    \mathcal{O}(\sqrt{\alpha}).
\]
\end{proposition}
Starting from the recursion~\eqref{eq:auxiliary-additive},
we linearize \(\mT\) at \(\theta^*\). Proposition~\ref{prop:Jacobian} shows
that this changes the scaled steady-state law by at most
\(\mathcal O(\sqrt{\alpha})\) in Wasserstein-\(2\) distance. It therefore remains to analyze \(B_\infty^{(\alpha)}\). Since its recursion
has linear drift and additive noise, \(B_\infty^{(\alpha)}\) admits a
geometrically weighted-sum representation, enabling a direct quantitative
Gaussian approximation. We carry out this analysis in the next subsection
and defer the proof of Proposition~\ref{prop:Jacobian} to
Section~\ref{sec:proof-Jacobian}.

\subsection{Step 2: Gaussian Approximation for Geometrically Weighted Sums}
\label{sec:geometric-weight}

Having reduced the original SA recursion to a linear recursion with additive
noise in Section~\ref{sec:Jacobian}, we now establish a quantitative Gaussian
approximation for its scaled steady-state distribution. Define
\[
    B_t^{(\alpha)}
    :=
    \frac{b_t^{(\alpha)}-\theta^*}{\sqrt{\alpha}},
    \qquad
    Q_\alpha
    :=
    (1-\alpha)I_d+\alpha J.
\]
Then
\begin{equation}
\label{eq:scaled-linearized-recursion}
    B_{t+1}^{(\alpha)}
    =
    Q_\alpha B_t^{(\alpha)}
    +
    \sqrt{\alpha}\,h(x_t).
\end{equation}
Since \(J-I_d\) is Hurwitz, standard linear stability theory implies that,
for all sufficiently small \(\alpha>0\), there exist constants
\(c_1,c_2>0\), independent of \(\alpha\) and $k$, such that
\begin{equation}\label{eq:Qalpha-geometric-decay}
 \|Q_\alpha^k\|
    \leq
    c_1 e^{-c_2\alpha k},
    \qquad k\geq0.
\end{equation}
Hence, on a two-sided stationary extension
\(\{x_t\}_{t\in\mathbb Z}\), the stationary solution admits the representation
\begin{equation}
\label{eq:geometric-sum-representation}
    B_\infty^{(\alpha)}
    \overset{\mathrm d}{=}
    \sqrt{\alpha}
    \sum_{k=0}^{\infty}
    Q_\alpha^k h(x_{-k-1}),
\end{equation}
where the series converges in \(L^2\). Thus,
\(B_\infty^{(\alpha)}\) is a geometrically weighted sum of the centered
process \(\{h(x_t)\}\), to which we apply a quantitative Gaussian
approximation in the following proposition.

\begin{proposition}[Gaussian Approximation for Geometrically Weighted Sums]
\label{prop:geometric-gaussian}
Under Assumptions~\ref{assumption:contraction},
\ref{assumption:noise}, \ref{assumption:convergence}, and
\ref{assumption:local-quadratic},
\[
\W_2\left(
    \law\bigl(B_\infty^{(\alpha)}\bigr),
    \mathcal{N}(0,V)
\right)
\in 
\mathcal{O}(\sqrt{\alpha}).
\]
\end{proposition}

Proposition~\ref{prop:geometric-gaussian} shows that
\(B_\infty^{(\alpha)}\) converges to \(\mathcal N(0,V)\) at rate
\(\mathcal O(\sqrt{\alpha})\) in \(\W_2\). Its proof exploits the weighted-sum
representation~\eqref{eq:geometric-sum-representation}. In the Markovian
setting, we use the Poisson equation to decompose the additive functional into
a martingale sum and a telescoping remainder, apply a quantitative Gaussian
approximation to the weighted martingale sum, and show that the remainder
contributes only \(\mathcal O(\sqrt{\alpha})\). For i.i.d.\ data, the argument
reduces to a Gaussian approximation for weighted sums of independent random
vectors. The required Gaussian approximation results are collected in
Section~\ref{sec:prior-GA}, and the proof of
Proposition~\ref{prop:geometric-gaussian} is deferred to
Section~\ref{sec:proof-geometric-gaussian}.

Finally, Theorem~\ref{thm:steady-state-gaussian-approx} follows by combining
Propositions~\ref{prop:additive}, \ref{prop:Jacobian}, and
\ref{prop:geometric-gaussian} with the triangle inequality:
\begin{align*}
\W_2\left(
    \law\bigl(Y_\infty^{(\alpha)}\bigr),
    \mathcal N(0,V)
\right)
&\leq
\W_2\left(
    \law\bigl(Y_\infty^{(\alpha)}\bigr),
    \law\bigl(A_\infty^{(\alpha)}\bigr)
\right) +
\W_2\left(
    \law\bigl(A_\infty^{(\alpha)}\bigr),
    \law\bigl(B_\infty^{(\alpha)}\bigr)
\right) \\
&\quad+
\W_2\left(
    \law\bigl(B_\infty^{(\alpha)}\bigr),
    \mathcal N(0,V)
\right)
\in 
\mathcal O\bigl(\sqrt{\alpha}\bigr).
\end{align*}
\section{General Steady-State Convergence}\label{sec:nondifferentiable}
In this section, we remove the local quadratic linearization requirement on mean operator $\mT$. Instead, we focus a broad class of mean operator $\mT$, which was studied in the prior work \cite{zhang2024prelimit} when the data is i.i.d.
\begin{assumption}\label{assumption:one-side}
The mean operator $\mT$ is one-sided direactionally differentiable at its fixed point $\theta^*$: there exists a function $G$: $\mathbb{S}^{d-1}\to \R^d$ such that
\begin{align*}
    \lim_{w \to 0^+} \Big\|\frac{\mT(\theta^* + w\theta)-\mT( \theta^*)}{w} - G(\theta)\Big\| = 0, \quad \forall \theta \in \mathbb{S}^{d-1}.
\end{align*}
\end{assumption}
The class of one-sided directionally differentiable mappings is broad,
encompassing several widely studied nonsmooth subclasses, including
\(g\circ F\) composite functions
\citep{shapiro2003class,sagastizabal2013composite},
prox-regular functions~\citep{poliquin1996prox}, and Clarke-regular
functions~\citep{clarke1990optimization,hiriart2004fundamentals}.
It also includes classical max-type functions, spectral mappings such as the
largest-eigenvalue map, the sparsity-promoting \(\ell_1\)-norm, and their
compositions with smooth transformations. A prominent example of a max-type structure at the fixed point arises in
asynchronous Q-learning when the optimal policy is not unique, which we
discuss in detail in Section~\ref{sec:Q-learning}.

Under the broader Assumption~\ref{assumption:one-side}, together with the
standing assumptions of Section~\ref{sec:problem}, we obtain the following
general SSC result.

\begin{theorem}[General Steady-State Convergence]
\label{thm:general-ssc}
Under Assumptions~\ref{assumption:contraction},
\ref{assumption:noise}, \ref{assumption:convergence}, and
\ref{assumption:one-side}, there exists a unique probability measure
\(\law(Y_\infty)\in\mathcal P_2(\R^d)\), determined solely by the mean operator
\(\mT\) and the long-run covariance matrix \(\Sigma_h\), such that
\[
    \lim_{\alpha \downarrow 0}
    \W_2\left(
        \law\bigl(Y_\infty^{(\alpha)}\bigr),
        \law\bigl(Y_\infty\bigr)
    \right)
    =
    0.
\]
\end{theorem}
To the best of our knowledge, \cite{zhang2024prelimit} is the only existing
work establishing a general SSC result for contractive SA without requiring
 \(\mT\) to be locally differentiable at \(\theta^*\).
That result, however, is restricted to i.i.d.\  noise.
Theorem~\ref{thm:general-ssc} extends this theory from the i.i.d.\ setting to
Markovian data.

From an applications perspective, a direct consequence of
Theorem~\ref{thm:general-ssc} is a  characterization of the
asymptotic bias. Indeed, since convergence in
\(\W_2\) implies convergence of first moments, we have
\begin{equation}\label{eq:general-bias}
\E\bigl[\theta_\infty^{(\alpha)}\bigr]-\theta^*
    =
    \sqrt{\alpha}\,
    \E\bigl[Y_\infty^{(\alpha)}\bigr]
    =
    \sqrt{\alpha}\,\E[Y_\infty]
    +
    o\bigl(\sqrt{\alpha}\bigr).
\end{equation}
When the data are i.i.d., \cite{zhang2024prelimit} shows that, in genuinely
nondifferentiable settings, the steady-state limit \(Y_\infty\) need not be
centered, so the asymptotic bias may have a nonvanishing leading term of
order \(\sqrt{\alpha}\). This characterization can in turn be used to design
and analyze bias-reduction procedures. Since
Theorem~\ref{thm:general-ssc} shows that \(Y_\infty\) depends only on the mean
operator \(\mT\) and the long-run covariance matrix \(\Sigma_h\), the same
bias analysis and bias-reduction principles extend to the more general
Markovian setting, as discussed in Section~\ref{sec:bias}.

From a technical perspective, the key technical ingredient is a \emph{Gaussian-noise universality principle}. Let \(Z_\infty^{(\alpha)}\) denote the scaled steady state of an auxiliary SA recursion with the same mean operator \(\mT\), but driven by additive i.i.d.\ Gaussian noise with covariance \(\Sigma_h\). We show that $\W_2\left(
        \law\bigl(A_\infty^{(\alpha)}\bigr),
        \law\bigl(Z_\infty^{(\alpha)}\bigr)
    \right)
    \in
    \mathcal{O}\left(
       \alpha^{1/4}
    \right).$ Thus, at the diffusion scale, the steady-state distribution depends on the Markovian noise only through its long-run covariance \(\Sigma_h\). Combined with Proposition~\ref{prop:additive}, this universality principle also gives a direct proof of Theorem~\ref{thm:general-ssc}. Indeed, the auxiliary Gaussian-noise recursion falls within the i.i.d.\ setting studied in \cite{zhang2024prelimit}, which yields \[ \W_2\left( \law\bigl(Z_\infty^{(\alpha)}\bigr), \law(Y_\infty) \right) \longrightarrow 0 \qquad\text{as }\alpha\downarrow0 \] for a limiting random variable \(Y_\infty\) determined by \(\mT\) and \(\Sigma_h\). Therefore, by the triangle inequality, \[ \begin{aligned} \W_2\left( \law\bigl(Y_\infty^{(\alpha)}\bigr), \law(Y_\infty) \right) &\leq \W_2\left( \law\bigl(Y_\infty^{(\alpha)}\bigr), \law\bigl(A_\infty^{(\alpha)}\bigr) \right) + \W_2\left( \law\bigl(A_\infty^{(\alpha)}\bigr), \law\bigl(Z_\infty^{(\alpha)}\bigr) \right) \\ &\quad+ \W_2\left( \law\bigl(Z_\infty^{(\alpha)}\bigr), \law(Y_\infty) \right) \longrightarrow 0. \end{aligned} \] The Gaussian-noise universality principle is developed in Section~\ref{sec:ind-Gaussian-universality}.

\subsection{Bias Characterization and Reduction for Locally Nondifferentiable SA}
\label{sec:bias}

Equation~\eqref{eq:general-bias} provides a general characterization of the
asymptotic bias for contractive SA whenever the mean operator \(\mT\) satisfies
Assumption~\ref{assumption:one-side}. Under the local quadratic linearization
condition in Assumption~\ref{assumption:local-quadratic}, however,
\eqref{eq:general-bias} is not sharp, because its
\(\sqrt{\alpha}\)-order term vanishes. Indeed, since
\(\mathcal N(0,V)\) is centered, Theorem~\ref{thm:steady-state-gaussian-approx}
implies
\begin{align*}
    \left\|
        \E\bigl[\theta_\infty^{(\alpha)}\bigr]-\theta^*
    \right\|_2=
    \sqrt{\alpha}\,
    \left\|
        \E\bigl[Y_\infty^{(\alpha)}\bigr]
    \right\|_2 \leq
    \sqrt{\alpha}\,
    \W_2\left(
        \law\bigl(Y_\infty^{(\alpha)}\bigr),
        \mathcal N(0,V)
    \right)
    \in
    \mathcal O(\alpha).
\end{align*}
Under stronger smoothness conditions, prior work further identifies the
leading \(\alpha\)-order term explicitly, as in \eqref{eq:bias-smooth}; see,
for example, \cite{dieuleveut2020bridging,zhang2024constant,huo2024collusion}.

The most informative regime of \eqref{eq:general-bias} is when
\(\E[Y_\infty]\neq0\). To characterize when this occurs, we adopt the same
criterion as in \cite{zhang2024prelimit}. Recall that
\(G:\mathbb S^{d-1}\to\R^d\) denotes the one-sided directional derivative map of \(\mT\) at \(\theta^*\) in
Assumption~\ref{assumption:one-side}, and define its positively homogeneous extension \(H:\R^d\to\R^d\) by $H(y):=\|y\|\,G\left(\frac{y}{\|y\|}\right)$ with \(H(0):=0\). Notice that if \(\mT\) is differentiable at \(\theta^*\) with
\(\nabla \mT(\theta^*)=J\), then \(H(y)=Jy\). In fact,
\cite{zhang2024prelimit} shows that these two conditions are equivalent.
The following corollary is therefore an immediate consequence of
\cite[Theorem~3]{zhang2024prelimit}.
\begin{corollary}[Bias Characterization]
\label{co:bias}
Under Assumptions~\ref{assumption:contraction},
\ref{assumption:noise}, \ref{assumption:convergence}, and
\ref{assumption:one-side}, suppose that \(\Sigma_h\) is positive definite.
If there exists an index \(i\in\{1,\ldots,d\}\) such that the \(i\)-th
coordinate \(H_i\) of \(H\) has a nonsingleton Fenchel subdifferential or Fenchel
superdifferential at \(0\), then $\E[Y_\infty]\neq 0.$
\end{corollary}

Roughly speaking, the condition in Corollary~\ref{co:bias} captures a
genuinely nondifferentiable regime in which \(\mT\) is not differentiable at
\(\theta^*\) and the associated \(H\) is nonlinear. Provided
that the noise at equilibrium, $h(x):=\tmT(x,\theta^*)-\theta^*,$ is nondegenerate in the long run, this condition implies
\(\E[Y_\infty]\neq0\). Consequently, \eqref{eq:general-bias} yields an asymptotic bias of order \(\sqrt{\alpha}\), in sharp contrast to the
\(\mathcal O(\alpha)\) bias under local quadratic linearization. In
Section~\ref{sec:Q-learning}, we show that this nondifferentiability condition
is satisfied for asynchronous Q-learning when some state admits multiple
optimal actions.

\noindent\textbf{Bias Reduction.} A classical approach to bias reduction is Richardson--Romberg (RR)
extrapolation applied to tail-averaged iterates. For simplicity, suppose that
\(t\) is even and define the tail average by
\[
    \bar{\theta}_t^{(\alpha)}
    :=
    \frac{2}{t}
    \sum_{k=t/2}^{t-1}\theta_k^{(\alpha)}.
\]
Under ergodicity, $\bar{\theta}_t^{(\alpha)}
    \to\E\bigl[\theta_\infty^{(\alpha)}\bigr]$ as $t\to\infty.$
RR extrapolation~\cite{hildebrand1987introduction} then combines tail averages
corresponding to two different stepsizes. A natural choice is \(\alpha\) and
\(2\alpha\), leading to
\[
    \widetilde{\theta}_t^{(\alpha)}
    :=
    \lambda_1\bar{\theta}_t^{(\alpha)}
    +
    \lambda_2\bar{\theta}_t^{(2\alpha)}.
\]
The coefficients \((\lambda_1,\lambda_2)\) are chosen according to the
leading-order of the asymptotic bias. When the leading bias is of
order \(\sqrt{\alpha}\), the appropriate choice is $(\lambda_1,\lambda_2)
    =
    (2+\sqrt{2},-1-\sqrt{2}),$ which cancels the leading \(\sqrt{\alpha}\)-order bias term
\cite{zhang2024prelimit}. Corollary~\ref{co:bias} then provides a theoretical
guarantee that these extrapolated tail-averaged iterates achieve bias
reduction for Markovian SA with a locally
nondifferentiable mean operator \(\mT\).

In some settings, it may be unclear whether the leading 
bias is of order \(\alpha\) or \(\sqrt{\alpha}\), although it is known to take
one of these two forms. In this case, one can use three stepsizes and choose
the corresponding extrapolation coefficients so as to cancel both the
\(\sqrt{\alpha}\)- and \(\alpha\)-order terms simultaneously. We discuss this
three-level extrapolation scheme in the Q-learning setting in
Section~\ref{sec:Q-learning}.

\subsection{Independent Gaussian Noise Universality}\label{sec:ind-Gaussian-universality}
For general contractive SA, the mean operator \(\mT\) need not be
differentiable at \(\theta^*\), so the Jacobian-drift universality 
is unavailable. Building on the additive-noise universality principle, we
instead establish an \emph{independent Gaussian-noise universality}:
the steady state can be approximated by that of an auxiliary recursion with
the same mean operator \(\mT\), but with the Markovian, non-Gaussian 
\(\{h(x_t)\}_{t\geq0}\) replaced by i.i.d.\ Gaussian 
\(\{w_t\}_{t\geq0}\). Specifically, the auxiliary recursion is given by
\begin{equation}
\label{eq:auxiliary-Gaussian}
    \beta_{t+1}^{(\alpha)}
    =
    \beta_t^{(\alpha)}
    +
    \alpha\Big(
        \mT(\beta_t^{(\alpha)})
        -
        \beta_t^{(\alpha)}
        +
        w_t
    \Big),
    \qquad t\geq0,
\end{equation}
where \(\{w_t\}_{t\geq0} \overset{\mathrm{i.i.d.}}{\sim} \mathcal{N}(0,\Sigma_h).\) Recursion~\eqref{eq:auxiliary-Gaussian} satisfies the assumptions of the general SA recursion~\eqref{eq:update}. Hence, for all sufficiently small \(\alpha>0\), it admits a unique steady-state law \(\law(\beta_\infty^{(\alpha)})\in\mathcal{P}_2(\mathbb{R}^d)\). Define the scaled steady state by $Z_\infty^{(\alpha)}
    :=
    \frac{\beta_\infty^{(\alpha)}-\theta^*}{\sqrt{\alpha}}.$ The following proposition states the independent Gaussian
noise universality result.

\begin{proposition}[Independent Gaussian Noise Universality]
\label{prop:ind-Gaussian}
Under Assumptions~\ref{assumption:contraction}, \ref{assumption:noise},
and \ref{assumption:convergence},
\[
    \W_2\bigl(\law(A_\infty^{(\alpha)}),\law(Z_\infty^{(\alpha)})\bigr)
    \in\mathcal O(\alpha^{1/4}).
\]
\end{proposition}
We note that \cite{zhang2024prelimit} also proves SSC via an independent
Gaussian-noise universality principle; see their Proposition~2. Their setting,
however, assumes i.i.d.\ data, whereas we allow Markovian noise. This
Markovian-to-independent reduction is an additional difficulty. Nevertheless,
our universality bound matches the \(\mathcal O(\alpha^{1/4})\) rate in
\cite{zhang2024prelimit}.

To illustrate the difficulty, consider the key blockwise coupling estimate used in \cite{zhang2024prelimit}. For the steady-state comparison, we may initialize the data process in stationarity, \(x_0\sim\mu\). For any fixed \(n\in\mathbb N\), define 
\begin{equation}\label{eq:def_block_sum}
    S_m^{(n)}
    :=
    \frac{1}{\sqrt n}
    \sum_{j=nm}^{n(m+1)-1} h(x_j),
    \qquad
    Z_m^{(n)}
    :=
    \frac{1}{\sqrt n}
    \sum_{j=nm}^{n(m+1)-1} w_j,
    \qquad m\geq0.
\end{equation}
For each \(n\), we seek a coupling, which may depend on \(n\), between
\(\{h(x_t)\}_{t\geq0}\) and \(\{w_t\}_{t\geq0}\) such that
\[
    \sup_{m\geq0}
    \left(
        \E\!\left[
            \bigl\|S_m^{(n)}-Z_m^{(n)}\bigr\|^2
        \right]
    \right)^{1/2}
   \in 
    \mathcal O\!\left(n^{-1/2}\right).
\]
Throughout this subsection, the constant hidden in the
\(\mathcal O(\cdot)\) notation is independent of \(n\). When \(\{x_t\}_{t\geq0}\) are i.i.d., this follows directly by coupling each block independently and applying the optimal \(\mathcal W_2\) Gaussian approximation rate \(\mathcal O(n^{-1/2})\) for sums of independent random variables \cite{bonis2020stein}.

This argument does not directly extend to Markovian noise. Although
\(\mathcal O(n^{-1/2})\) Gaussian approximation rates hold for Markov-chain
additive functionals \cite{zhang2026martingale}, the block sums
\(\{S_m^{(n)}\}_{m\geq0}\) remain dependent. Thus, coupling each block
independently does not guarantee independent Gaussian counterparts.

We overcome this issue by first decoupling the block sums while preserving
their marginals, using the following Wasserstein--\(p\) decoupling lemma.
\begin{lemma}[Wasserstein--$p$ Decoupling]\label{lem:wp-decoupling}
Let $(\Omega,\mathcal F,\mathbb P)$ be a probability space, let $(\mathcal M,d)$ be a
Polish metric space, and let $\mathcal G\subseteq\mathcal F$ be a sub-$\sigma$-field. Fix $p\ge 1$.
Let $Y:\Omega\to\mathcal M$ be Borel measurable with
$\mathcal L(Y)\in\mathcal P_p(\mathcal M)$.  Then there exists an extension $(\widetilde\Omega,\widetilde{\mathcal F},
\widetilde{\mathbb P})$ of $(\Omega,\mathcal F,\mathbb P)$ and an $\mathcal M$-valued
random variable $Y^*$ on $(\widetilde\Omega,\widetilde{\mathcal F},\widetilde{\mathbb P})$
such that
\begin{enumerate}
\item $\mathcal L(Y^*)=\mathcal L(Y)$,
\item $Y^*$ is independent of $\mathcal G$, and
\item $\mathbb E\big[d(Y,Y^*)^p\big]
=
\mathbb E\left[\W_{p,d}\big(\mathcal L(Y\mid\mathcal G),\mathcal L(Y)\big)^p\right].$
\end{enumerate}
\end{lemma}
Lemma~\ref{lem:wp-decoupling} can be viewed as a Wasserstein--\(p\)
extension of the Dedecker--Prieur \(\tau\)-coupling
\cite{dedecker2004coupling} from \(p=1\) to arbitrary \(p\geq1\).
Indeed, when \(p=1\), the identity becomes
\[
\mathbb E\big[d(Y,Y^*)\big]
=
\mathbb E \left[ \mathcal W_{1,d} \big(\mathcal L(Y\mid\mathcal G),\,\mathcal L(Y)\big)\right]
\;=:\;
\tau(\mathcal G,Y),
\]
which is precisely the \(\tau\)-dependence coefficient.
 Moreover,  when  $d(x,y)=\mathds 1_{\{x\neq y\}}$, the Wasserstein distance \(\W_{1,d}\) coincides with the total variation
distance $\mathrm{TV}(\nu,\mu):=\sup_{A\in\mathcal B(\mathcal X)}|\nu(A)-\mu(A)|$. In this case, Lemma~\ref{lem:wp-decoupling} reduces to Berbee's coupling
\cite{doukhan1995invariance}:
\[
\mathbb P(Y\neq Y^*)
=
\mathbb E\left[\mathrm{TV}\big(\mathcal L(Y\mid\mathcal G),\mathcal L(Y)\big)\right] = \frac{1}{2} \mathbb{E}\Big[\sup _{\|f\|_{\infty} \leq 1}|\mathbb{E}[f(Y) \mid \mathcal{G}]-\mathbb{E}[f(Y)]|\Big]
=:
\beta(\mathcal G,\sigma(Y)),
\]
Using Lemma~\ref{lem:wp-decoupling}, we first couple
\(\{S_m^{(n)}\}_{m\geq0}\) with an i.i.d.\ sequence
\(\{S_m'^{(n)}\}_{m\geq0}\) having the same marginals. The decoupling error is
characterized by Lemma~\ref{lem:wp-decoupling}(iii) and controlled using
uniform ergodicity. We then couple each decoupled block to a Gaussian
block sum using fresh independent randomization. The construction is carried
out sequentially: each Gaussian block is independent of the information
available at the preceding block boundary, and the original Markov property
is preserved at these boundaries. This does not assert adaptedness at
interior times within a block. The two-stage construction is illustrated in
Figure~\ref{fig:decoupling} and forms the basis of the proof of
Proposition~\ref{prop:uniform}.

\begin{figure}[htbp!]
\centering
\begin{adjustbox}{max width=\linewidth}
\begin{tikzpicture}[thick]
  \node (t0) {$S_0^{(n)}$};
  \node            (tdotsL) [right=12mm of t0] {$S_1^{(n)}$};
  \node (tk)  [right=12mm of tdotsL] {$S_2^{(n)}$};
  \node (tk1) [right=12mm of tk] {...};

  \node (x0)  [below=4mm of t0] {$S_0'^{(n)}$};
  \node            (xdotsL) [right=12mm of x0] {$S_1'^{(n)}$};
  \node (xk)  [right=12mm of xdotsL] {$S_2'^{(n)}$};
  \node (xk1) [right=12mm of xk] {...};

    \node (y0)  [below=4mm of x0] {$Z_0^{(n)}$};
  \node            (ydotsL) [right=12mm of y0] {$Z_1^{(n)}$};
  \node (yk)  [right=12mm of ydotsL] {$Z_2^{(n)}$};
  \node (yk1) [right=12mm of yk] {...};

  \draw[->] (t0) -- (tdotsL);
  \draw[->] (tdotsL) -- (tk);
  \draw[->] (tk) -- (tk1);

  \draw[->,blue] (t0) -- (x0);
  \draw[->,blue] (tk) -- (xk);
  \draw[->,blue] (tk1) -- (xk1);
  \draw[->,blue] (tdotsL) -- (xdotsL);

    \draw[->,blue] (t0) -- (xdotsL);
  \draw[->,blue] (tdotsL) -- (xk);
  \draw[->,blue] (tk) -- (xk1);

    \draw[->,red] (x0) -- (y0);
  \draw[->,red] (xk) -- (yk);
  \draw[->,red] (xk1) -- (yk1);
  \draw[->,red] (xdotsL) -- (ydotsL);

\end{tikzpicture}
\end{adjustbox}
\caption{The black arrows represent the Markov-chain dependence across blocks, the blue
arrows represent the Wasserstein--\(p\) decoupling constructed via
Lemma~\ref{lem:wp-decoupling}, and the red arrows represent the subsequent
optimal-transport Gaussian coupling.
}
\label{fig:decoupling}
\end{figure}
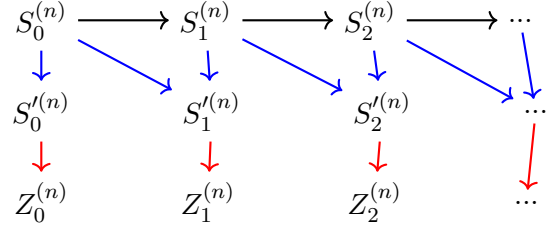
\begin{proposition}[Uniform Blockwise Gaussian Coupling]
\label{prop:uniform}
Let \(\{x_t\}_{t\geq0}\) be a stationary uniformly ergodic Markov chain on
\(\mathcal X\), with transition kernel \(P\), stationary distribution
\(\mu\), and \(x_0\sim\mu\), satisfying \eqref{eq:uniform-GE}.
Let \(h:\mathcal X\to\mathbb R^d\) satisfy \(\E_\mu[h]=0\) and
\(\sup_{x\in\mathcal X}\|h(x)\|<\infty\), and let \(\Sigma_h\) be its
long-run covariance matrix. For every fixed \(n\in\mathbb N\), there exists a coupling between
\(\{h(x_t)\}_{t\geq0}\) and an i.i.d.\ Gaussian sequence
\(\{w_t\}_{t\geq0}\) such that the following holds for the block sum sequences $\{S^{(n)}_m\}_m$ and $\{Z^{(n)}_m\}_m$ defined in \eqref{eq:def_block_sum}:
\begin{equation}
\label{eq:uniform-block-W2}
    \sup_{m\geq0}
    \left(
        \E\!\left[
            \bigl\|S_m^{(n)}-Z_m^{(n)}\bigr\|^2
        \right]
    \right)^{1/2}
    \in \mathcal{O}(n^{-1/2}).
\end{equation}
\end{proposition}
Proposition~\ref{prop:uniform} is a key technical ingredient in the proof of
Proposition~\ref{prop:ind-Gaussian}.  The proofs of
Lemma~\ref{lem:wp-decoupling} and Proposition~\ref{prop:uniform} are deferred
to Section~\ref{sec:proof-uniform}, and that of
Proposition~\ref{prop:ind-Gaussian} to
Section~\ref{sec:proof-universal-noise}.

\section{Applications to Markovian Linear Stochastic Approximation}
\label{sec:linearSA}
In this section, we specialize the steady-state Gaussian approximation results
of Section~\ref{sec:differentiable} to Markovian linear stochastic
approximation (LSA), a canonical SA model with multiplicative Markovian noise
and a linear mean operator.

Consider the constant-stepsize Markovian LSA recursion \cite{srikant2019finite}
\begin{equation}
\label{eq:update-LSA}
    \theta_{t+1}^{(\alpha)}
    =
    \theta_t^{(\alpha)}
    +
    \alpha\left(
        \mathsf A(x_t)\theta_t^{(\alpha)}
        +
        s(x_t)
    \right),
    \qquad t\geq0,
\end{equation}
where  \(\{x_t\}_{t\geq0}\) is a uniformly ergodic Markov chain on
\(\mathcal X\) with unique stationary distribution \(\mu\), satisfying
\eqref{eq:uniform-GE}. Let \(\mathsf A:\mathcal X\to\R^{d\times d}\) and \(s:\mathcal X\to\R^d\) be measurable and uniformly bounded. Define $\overline{\mathsf A}
    :=
    \E_{x\sim\mu}[\mathsf A(x)]
    \in\R^{d\times d}$ and $
    \overline s
    :=
    \E_{x\sim\mu}[s(x)]
    \in\R^d.$ Assume that \(\overline{\mathsf A}\) is Hurwitz, as is standard in the analysis of Markovian LSA; see, e.g., \cite{srikant2019finite,huo2023bias,durmus2025finite}. The target point is the unique solution to \(\overline{\mathsf A}\theta+\overline s=0\), namely, $\theta^* := -\overline{\mathsf A}^{-1}\overline s. $ 
    
Markovian LSA encompasses several widely used reinforcement-learning
algorithms, including TD(0) and TD(\(\lambda\)) with linear function
approximation for policy evaluation
\cite{tsitsiklis1996analysis,bhandari2021finite,srikant2019finite,mou2021optimal}.

We first rewrite \eqref{eq:update-LSA} as a contractive SA recursion of the
form \eqref{eq:update}. Specifically,
\begin{equation}
\label{eq:recursion-LSA}
    \theta_{t+1}^{(\alpha)}
    =
    \theta_t^{(\alpha)}
    +
    \alpha\kappa
    \left[
        \left(I_d+\frac{\mathsf A(x_t)}{\kappa}\right)
        \theta_t^{(\alpha)}
        +
        \frac{s(x_t)}{\kappa}
        -
        \theta_t^{(\alpha)}
    \right],
    \qquad t\geq0,
\end{equation}
where \(\kappa>0\) is a fixed constant satisfying
\[
    \kappa
    >
    \max_{\lambda\in\sigma(\overline{\mathsf A})}
    \frac{|\lambda|^2}{-2\operatorname{Re}(\lambda)},\qquad\sigma(\overline{\mathsf A})
    :=
    \left\{
        \lambda\in\mathbb C:
        \det(\overline{\mathsf A}-\lambda I_d)=0
    \right\}
\]
By Lemma~\ref{lem:LSA}, \eqref{eq:recursion-LSA} satisfies the assumptions of
the general SA recursion~\eqref{eq:update} with stepsize \(\alpha\kappa\),
random operator $\tmT(x,\theta)
    :=
    \left(I_d+\frac{\mathsf A(x)}{\kappa}\right)\theta
    +\frac{s(x)}{\kappa},$ and mean operator $\mT(\theta)
    =
    \left(I_d+\frac{\overline{\mathsf A}}{\kappa}\right)\theta
    +\frac{\overline s}{\kappa}.$ Its Jacobian is $ J_{\mathrm{LSA}}
    :=\nabla\mT(\theta^*)
    =I_d+\frac{\overline{\mathsf A}}{\kappa},$ whereas the equilibrium noise is $ h(x)
    :=\tmT(x,\theta^*)-\theta^*
    =\frac{\mathsf A(x)\theta^*+s(x)}{\kappa}.$ Define
\[
    \Sigma_{\mathrm{LSA}}
    :=
    \lim_{n\to\infty}
    \frac{1}{n}
    \operatorname{Cov}\left(
        \sum_{t=0}^{n-1}
        \bigl(\mathsf A(x_t)\theta^*+s(x_t)\bigr)
    \right).
\]
Since the mean operator is affine, Assumption~\ref{assumption:local-quadratic}
holds with \(J=J_{\mathrm{LSA}}\) and \(L_R=0\).
Theorem~\ref{thm:steady-state-gaussian-approx} and
\eqref{eq:Lyapunov} then yield
\[
    \W_2\left(
        \law\left(
            \frac{\theta_\infty^{(\alpha)}-\theta^*}
                 {\sqrt{\alpha\kappa}}
        \right),
        \mathcal N(0,V)
    \right)
    =
    \mathcal O(\sqrt{\alpha}),\qquad\frac{\overline{\mathsf A}}{\kappa}V
    +
    V\frac{\overline{\mathsf A}^\top}{\kappa}
    +
    \frac{\Sigma_{\mathrm{LSA}}}{\kappa^2}
    =
    0.
\]
Setting \(V_{\mathrm{LSA}}:=\kappa V\) gives the following steady-state Gaussian approximation for
Markovian LSA.
\begin{corollary}[Steady-state Gaussian Approximation for Markovian LSA]
\label{cor:lsa-steady-state-gaussian}
Under the setting of Section~\ref{sec:linearSA}, 
\begin{equation}
\label{eq:lsa-scaled-gaussian-bound}
    \W_2\left(
        \law\left(
            \frac{
                \theta_\infty^{(\alpha)}-\theta^*
            }{
                \sqrt{\alpha}
            }
        \right),
        \mathcal N(0,V_{\mathrm{LSA}})
    \right)
    \in \mathcal O(\sqrt{\alpha}), \qquad \overline{\mathsf A} V_{\mathrm{LSA}} + V_{\mathrm{LSA}}\overline{\mathsf A}^\top   + \Sigma_{\operatorname{LSA}} = 0.
\end{equation}
\end{corollary}
As a consequence, Corollary~\ref{cor:raw-Gaussian} applies directly to
Markovian LSA and yields Gaussian approximation guarantees for the raw
iterates. To the best of our knowledge, our SSC framework therefore provides
the first steady-state Gaussian approximation, together with corresponding
raw-iterate Gaussian approximation guarantees, for Markovian LSA.

\section{Applications to Asynchronous Q-Learning}
\label{sec:Q-learning}

In this section, we specialize the SSC results of
Sections~\ref{sec:differentiable} and~\ref{sec:nondifferentiable} to
asynchronous Q-learning, a canonical SA model with multiplicative Markovian
noise and a potentially nondifferentiable mean operator.

\subsection{Model Setup}
\label{sec:q-learning-model}

Consider a discounted Markov decision process $\bigl(\mS,\mA,\mathsf P,r,\rho\bigr),$ where \(\mS\) and \(\mA\) are finite state and action spaces,
respectively, $\mathsf P(\cdot\mid s,a)\in\Delta(\mS)$ is the transition kernel, $r:\mS\times\mA\to[0,1]$ is the reward function, and \(\rho\in(0,1)\) is the discount factor. For a stationary policy
\(\pi:\mS\to\Delta(\mA)\), its Q-function is
\[
    q^\pi(s,a)
    :=
    \E^\pi\left[
        \sum_{t=0}^{\infty}
        \rho^t r(S_t,A_t)
        \,\middle|\,
        S_0=s,\ A_0=a
    \right],
\]
where, for \(t\geq1\), $A_t\sim\pi(\cdot\mid S_t)$ and $S_{t+1}\sim\mathsf P(\cdot\mid S_t,A_t).$ The optimal Q-function is $q^*(s,a)
    :=
    \sup_{\pi}q^\pi(s,a).$ It is the unique fixed point of the Bellman optimality operator
\(\mathcal B:\R^{|\mS||\mA|}\to\R^{|\mS||\mA|}\), defined by
\begin{equation}
\label{eq:q-bellman-operator}
    [\mathcal B q](s,a)
    :=
    r(s,a)
    +
    \rho
    \sum_{s'\in\mS}
    \mathsf P(s'\mid s,a)
    \max_{a'\in\mA}q(s',a').
\end{equation}
We consider the off-policy setting in which the data are generated by a
fixed stationary behavior policy
\(\pi_b:\mS\to\Delta(\mA)\). More precisely, $A_t\sim\pi_b(\cdot\mid S_t)$ and $S_{t+1}\sim\mathsf P(\cdot\mid S_t,A_t).$ 
We impose the following standard coverage and ergodicity condition for
asynchronous Q-learning \cite{chen2024lyapunov}.

\begin{assumption}[Coverage and Ergodicity under the Behavior Policy]
\label{assumption:MC}
The behavior policy \(\pi_b\) satisfies $\pi_b(a\mid s)>0$ for any $(s,a)\in\mathcal S\times\mathcal A,$ and the Markov chain \(\{S_t\}_{t\geq0}\) induced by \(\pi_b\) is irreducible
and aperiodic.
\end{assumption}

Let \(\rho_{\pi_b}\) denote the stationary distribution of
\(\{S_t\}_{t\ge0}\), and define the stationary state--action distribution $\nu(s,a)
    :=
    \rho_{\pi_b}(s)\pi_b(a\mid s).$ Since \(\mathcal S\) and \(\mathcal A\) are finite,
Assumption~\ref{assumption:MC} implies $\nu_{\min}
    :=
    \min_{(s,a)\in\mathcal S\times\mathcal A}\nu(s,a)
    >0.$ Moreover, $x_t:=(S_t,A_t,S_{t+1})$
is an irreducible and aperiodic finite-state Markov chain on
\[
    \mathcal X
    :=
    \left\{
        (s,a,s')\in\mathcal S\times\mathcal A\times\mathcal S:
        \mathsf P(s'\mid s,a)>0
    \right\}.
\]
Hence, \(\{x_t\}_{t\ge0}\) is uniformly ergodic, with stationary distribution
\[
    \mu(s,a,s')
    =
    \nu(s,a)\mathsf P(s'\mid s,a).
\]

\noindent\textbf{Asynchronous Q-Learning.}
Let \(d_Q:=|\mS||\mA|\), and identify each \(Q\)-function with a vector in
\(\R^{d_Q}\). For each \((s,a)\in\mS\times\mA\), let
\(e_{s,a}\in\R^{d_Q}\) denote the corresponding standard basis vector.
For \(x=(s,a,s')\in\mathcal X\), define
\begin{equation*}
    F(x,q)
    :=
    e_{s,a}
    \left(
        r(s,a)
        +
        \rho\max_{a'\in\mA}q(s',a')
        -
        q(s,a)
    \right).
\end{equation*}
The constant-stepsize asynchronous Q-learning recursion \cite{chen2024lyapunov} is
\begin{equation*}
    q_{t+1}^{(\alpha)}
    =
    q_t^{(\alpha)}
    +
    \alpha F\bigl(X_t,q_t^{(\alpha)}\bigr),
    \qquad t\geq0.
\end{equation*}
Equivalently, defining the random  operator $\widetilde{\mathcal T}(x,q)
    :=
    q+F(x,q),$ we may write
\begin{equation}\label{eq:Q-learning}
 q_{t+1}^{(\alpha)}
    =
    q_t^{(\alpha)}
    +
    \alpha \big(\tmT(x_t, q_t^{(\alpha)})-q_t^{(\alpha)}\big),\qquad t \geq 0.
\end{equation}
Let $D_\nu
    :=
    \operatorname{diag}
    \bigl(\{\nu(s,a)\}_{(s,a)\in\mS\times\mA}\bigr).$ Taking expectation with respect to the stationary distribution \(\mu\) gives
the mean operator
\begin{equation}
\label{eq:q-mean-operator}
    \mT(q):=
    \E_{x\sim\mu}
    \left[
        \tmT(x,q)
    \right] =
    q+D_\nu\bigl(\mathcal Bq-q\bigr) =
    (I_{d_Q}-D_\nu)q+D_\nu\mathcal Bq.
\end{equation}
Since the Bellman optimality operator is a
\(\rho\)-contraction in the supremum norm, the mean operator satisfies
\[
    \|\mT(q)-\mT(q')\|_\infty
    \leq
    \gamma\|q-q'\|_\infty,
    \qquad
    \forall q,q'\in\R^{d_Q},
\]
where
\[
    \gamma:=1-(1-\rho)\nu_{\min}\in(0,1).
\]
Thus, Assumption~\ref{assumption:contraction} holds with
\(\|\cdot\|_c=\|\cdot\|_\infty\) and contraction modulus \(\gamma\).
Moreover, \(\mT(q^*)=q^*\).
Therefore, the boundedness of the rewards, the properties of the max operator,
and the preceding discussion verify that the Q-learning recursion
\eqref{eq:Q-learning} satisfies the assumptions of the general SA recursion
\eqref{eq:update}. Let \(q_\infty^{(\alpha)}\) denote a random variable whose
distribution is the \(Q\)-marginal of the stationary distribution of the
Markov chain \(\{(x_t,q_t^{(\alpha)})\}_{t\ge0}\).

\subsection{Equilibrium Noise and Its Long-Run Covariance}
\label{sec:q-equilibrium-noise}
We write $ V^*(s)
    :=
    \max_{a\in\mA}q^*(s,a)$ and $\mA^*(s)
    :=
    \arg\max_{a\in\mA}q^*(s,a)$. The equilibrium noise  is
\begin{equation*}
    h(x):=
    \widetilde{\mathcal T}
    \bigl((s,a,s'),q^*\bigr)
    -
    q^* =
    e_{s,a}(r(s,a)
    +
    \rho V^*(s')
    -
    q^*(s,a)).
\end{equation*}
The Bellman optimality equation implies that $\E\left[
        h(S_t,A_t,S_{t+1})
        \,\middle|\,
        S_t,A_t
    \right]
    =
    0.$ Consequently, \(\{h(x_t)\}_{t\geq0}\) is a martingale difference
sequence with respect to the natural trajectory filtration. Therefore, its long-run covariance matrix reduces to the zero-lag covariance
\begin{equation*}
    \Sigma_Q
:=
    \sum_{\ell\in\mathbb Z}
    \E_\mu\left[
        h(x_0)h(x_\ell)^\top
    \right] =
    \E_\mu\left[
        h(x_0)h(x_0)^\top
    \right]=
    \rho^2\cdot\operatorname{diag}
    \left(
        \left\{
            \nu(s,a)\sigma_{s,a}^2
        \right\}_{(s,a)\in\mS\times\mA}
    \right),
\end{equation*}
where $\sigma_{s,a}^2
    :=
    \operatorname{Var}_{s'\sim\mathsf P(\cdot\mid s,a)}
    \left( V^*(s')
    \right).$ Since \(\nu(s,a)>0\) for every $(s,a) \in \mathcal S \times \mathcal A$,
\begin{equation}
\label{eq:q-covariance-pd-condition}
    \Sigma_Q \text{ is positive-definite}
    \quad\Longleftrightarrow\quad
    \sigma_{s,a}^2>0
    \quad
    \text{for every }(s,a)\in\mS\times\mA.
\end{equation}

\subsection{Steady-State Gaussian Approximation Under No Optimal-Action Ties}
\label{sec:q-no-ties}

We first consider the case where each state admits a unique optimal action.
For each \(s\in\mS\), let \(a^*(s)\) denote this action. Since \(\mathcal S\) and \(\mathcal A\) are finite, the minimum optimality gap
\[
    \Delta_*
    :=
    \min_{s\in\mS}
    \left\{
        q^*(s,a^*(s))
        -
        \max_{a\neq a^*(s)}q^*(s,a)
    \right\}>0.
\]
Thus, whenever $\|q-q^*\|_\infty\leq \frac{\Delta_*}{2},$ the greedy action associated with \(q\) remains \(a^*(s)\) at every state.
It follows that the Bellman and mean operators are affine in a neighborhood
of \(q^*\). Define the state--action transition matrix induced by the unique optimal
policy by
\begin{equation*}
    \bigl[P^{\pi^*}\bigr]_{(s,a),(s',a')}
    :=
    \mathsf P(s'\mid s,a)
    \mathds{1}_{\{a'=a^*(s')\}}.
\end{equation*}
In a neighborhood of \(q^*\), $\mathcal Bq
    =
    r+\rho P^{\pi^*}q,$ and hence the Jacobian of the mean operator at $q^*$ is $J_Q
    :=
    I_{d_Q}-D_\nu
    +
    \rho D_\nu P^{\pi^*}.$ Let $L_Q
    :=
    J_Q-I_{d_Q}
    =
    D_\nu
    \bigl(\rho P^{\pi^*}-I_{d_Q}\bigr).$
Thus, Assumption~\ref{assumption:local-quadratic} holds with \(J=J_Q\) and
\(L_R=0\), since \(\mT(q)-\mT(q^*)=J_Q(q-q^*)\) for all \(q\) sufficiently
close to \(q^*\).
Theorem~\ref{thm:steady-state-gaussian-approx} and
\eqref{eq:Lyapunov} then yield the following result.
\begin{corollary}[Steady-state Gaussian approximation for Asynchronous
Q-learning]
\label{cor:q-learning-gaussian}
Suppose that Assumption~\ref{assumption:MC} holds. If \(\mA^*(s)\) is a singleton for every $s\in\mS,$ 
\begin{equation*}
    \W_2\left(
        \law\left(
            \frac{
                q_\infty^{(\alpha)}-q^*
            }{
                \sqrt{\alpha}
            }
        \right),
        \mathcal N(0,V_Q)
    \right)
    \in \mathcal O(\sqrt{\alpha}), \qquad L_QV_Q
    +
    V_QL_Q^\top
    +
    \Sigma_Q
    =
    0.
\end{equation*}
\end{corollary}
As a consequence, Corollary~\ref{cor:raw-Gaussian} applies directly to
asynchronous Q-learning under the no-tie condition, yielding Gaussian
approximation guarantees for the raw iterates. To the best of our knowledge,
our SSC framework therefore provides the first steady-state Gaussian
approximation for asynchronous Q-learning, together with corresponding
Gaussian approximation guarantees for the raw iterates.

\subsection{\texorpdfstring{$\sqrt{\alpha}$}{sqrt(alpha)}-Bias in the Presence of Optimal-Action Ties}
\label{sec:q-ties}

We next consider the case where some state admits multiple optimal actions.
In this regime, the mean operator \(\mT\) is nondifferentiable at \(q^*\) and
satisfies the conditions of Corollary~\ref{co:bias}. Consequently, the
leading-order steady-state bias is of order \(\sqrt{\alpha}\), as stated in
the following corollary. Its proof is deferred to
Section~\ref{sec:proof-Q-learning}.

\begin{corollary}[\(\sqrt{\alpha}\) -Bias under Optimal-Action Ties]
\label{cor:q-learning-nonsmooth-bias}
Suppose that Assumption~\ref{assumption:MC} holds and that $\Sigma_Q$ is positive-definite.  If there exists a state \(\bar s\in\mS\) such that $|\mA^*(\bar s)|>1,$ then there exists an vector \(c_Q\in \R^{d_Q}\) and $c_Q \neq 0$ such that
\begin{equation*}
    \E\bigl[q_\infty^{(\alpha)}\bigr]-q^*
    =
    \sqrt{\alpha} c_Q 
    +
    o\bigl(\sqrt{\alpha}\bigr).
\end{equation*}
\end{corollary}

\noindent\textbf{Complete Bias Characterization for Asynchronous Q-learning.}
Under the no-optimal-action-ties assumption, \cite{zhang2024constant} study
the same asynchronous Q-learning recursion and show that the leading-order
asymptotic bias is linear in \(\alpha\). However, neither their work nor the
existing literature provides a corresponding bias characterization in the
presence of optimal-action ties. Our SSC theory fills this gap by showing
that, under ties, the asymptotic bias can be of order \(\sqrt{\alpha}\). Together, these results provide a complete leading-order characterization of
the asymptotic bias of asynchronous Q-learning.

\noindent\textbf{Covariance Condition.}
The positive-definiteness of \(\Sigma_Q\) is a convenient sufficient
condition for applying Corollary~\ref{co:bias}. In the present
deterministic-reward setting, by \eqref{eq:q-covariance-pd-condition}, it is equivalent to
\[
    \sigma_{s,a}^2
    :=
    \operatorname{Var}_{s'\sim\mathsf P(\cdot\mid s,a)}
    \bigl(V^*(s')\bigr)
    >0,
    \qquad
    \forall (s,a)\in\mS\times\mA.
\]

\noindent\textbf{A Unified Extrapolation Scheme.}
Suppose that \(\Sigma_Q\) is positive definite. Depending on the underlying
MDP, the leading-order bias scales as either \(\alpha\) or
\(\sqrt{\alpha}\). Since this regime may be unknown a priori, a robust
approach is to use three stepsizes, \(\alpha\), \(2\alpha\), and \(4\alpha\),
and simultaneously eliminate both candidate leading-order terms. Specifically, define
\begin{equation}
\label{eq:unified-extrapolation}
    \widetilde q_t^{(\alpha)}
    :=
    (4+2\sqrt{2})\,\bar q_t^{(\alpha)}
    -(4+3\sqrt{2})\,\bar q_t^{(2\alpha)}
    +(1+\sqrt{2})\,\bar q_t^{(4\alpha)},
\end{equation}
where \(\bar q_t^{(\alpha)}\) denotes the tail-averaged iterate with stepsize
\(\alpha\). Indeed, the coefficients satisfy
\[
\begin{aligned}
    &(4+2\sqrt{2})-(4+3\sqrt{2})+(1+\sqrt{2})=1,\\
    &(4+2\sqrt{2})
    -\sqrt{2}(4+3\sqrt{2})
    +2(1+\sqrt{2})=0,\\
    &(4+2\sqrt{2})
    -2(4+3\sqrt{2})
    +4(1+\sqrt{2})=0.
\end{aligned}
\]
Thus, the extrapolation preserves an accurate estimate of \(q^*\) while
simultaneously canceling the \(\sqrt{\alpha}\)- and \(\alpha\)-order terms,
yielding a unified bias-reduction scheme that does not require prior knowledge
of whether optimal-action ties are present. In Section~\ref{sec:experiments},
we further illustrate numerically that this unified extrapolation performs
robustly across both regimes, whereas a misspecified RR extrapolation can
lead to substantially worse performance.

\section{Numerical Experiments}\label{sec:experiments}
We illustrate the asymptotic bias theory for asynchronous Q-learning
using two finite MDPs: one in the no-tie regime, where each state admits a
unique optimal action, and one in the tied regime, where optimal-action ties
are present. The complete experimental setup is deferred to
Appendix~\ref{app:numerical-details}.

\noindent\textbf{Regime-Specific RR Extrapolation.}
Our SSC theory, together with the locally differentiable bias characterization
of \cite{zhang2024constant}, predicts different leading-order asymptotic biases
in the two regimes. In the no-tie regime, the mean operator is locally
affine and the bias is of order \(\alpha\); in the tied regime, the
mean operator is locally nondifferentiable and the bias is of order
\(\sqrt{\alpha}\). Accordingly, we use the regime-specific RR extrapolations $\widetilde q_t^{(\alpha)}
    :=
    2\bar q_t^{(\alpha)}
    -
    \bar q_t^{(2\alpha)}$ for the no-tie regime, and $\widetilde q_t^{(\alpha)}
    :=
    (2+\sqrt{2})\bar q_t^{(\alpha)}
    -
    (1+\sqrt{2})\bar q_t^{(2\alpha)}$ for the tied regime, which cancel the corresponding leading-order bias terms.

Figure~\ref{fig:qlearning-correct-rr} compares tail averaging with the
regime-specific RR extrapolation. In both cases, the correct RR
extrapolation significantly reduces the long-run error. These observations
support the distinct leading-order bias characterizations established in the
two regimes. Throughout this section, the solid and dashed curves report the average errors
over independent runs, while the shaded regions indicate pointwise \(95\%\)
confidence intervals.

\begin{figure}[!htbp]
    \centering
    \begin{minipage}{0.49\textwidth}
        \centering
        \includegraphics[width=\textwidth]{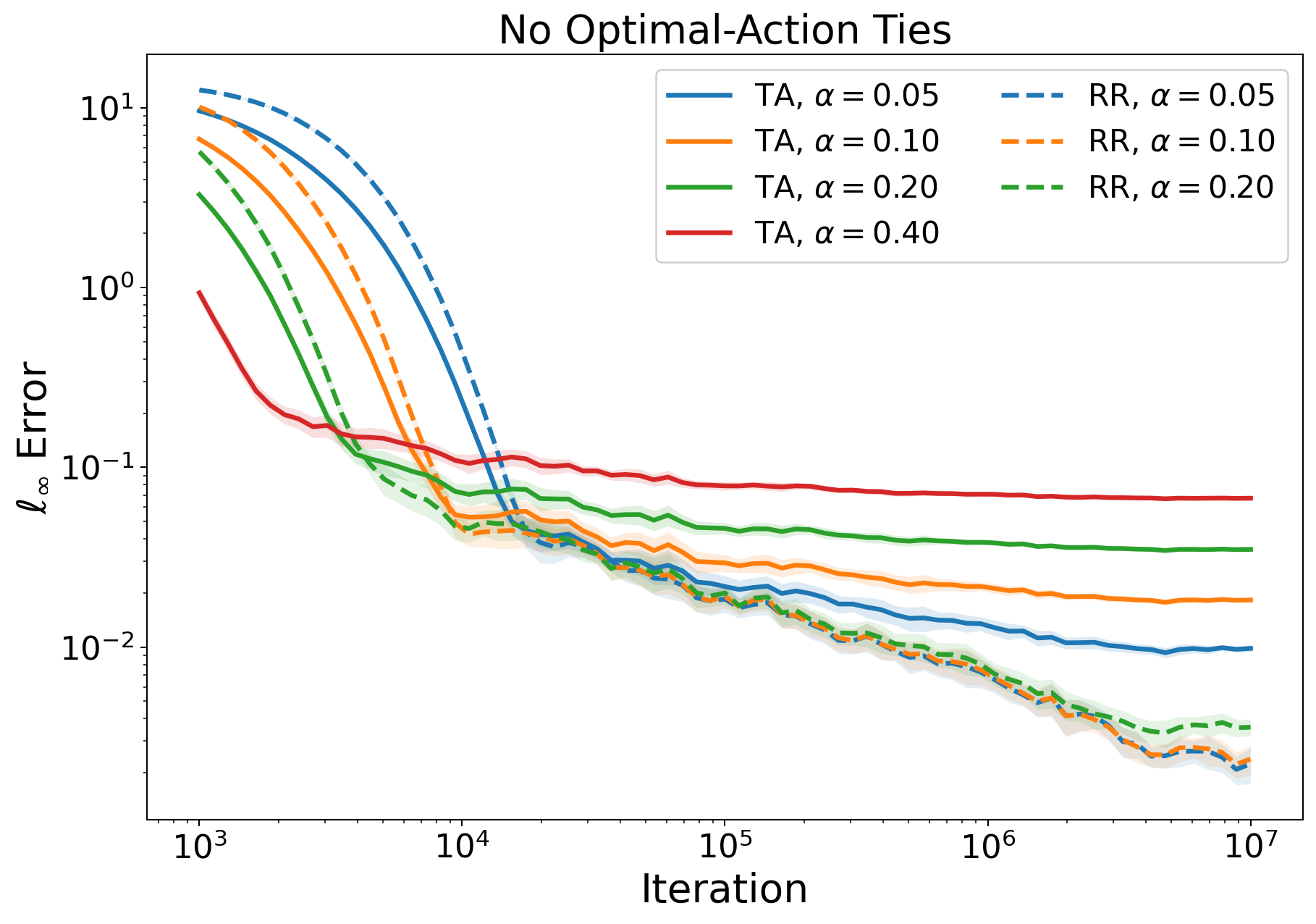}
    \end{minipage}
    \hfill
    \begin{minipage}{0.49\textwidth}
        \centering
        \includegraphics[width=\textwidth]{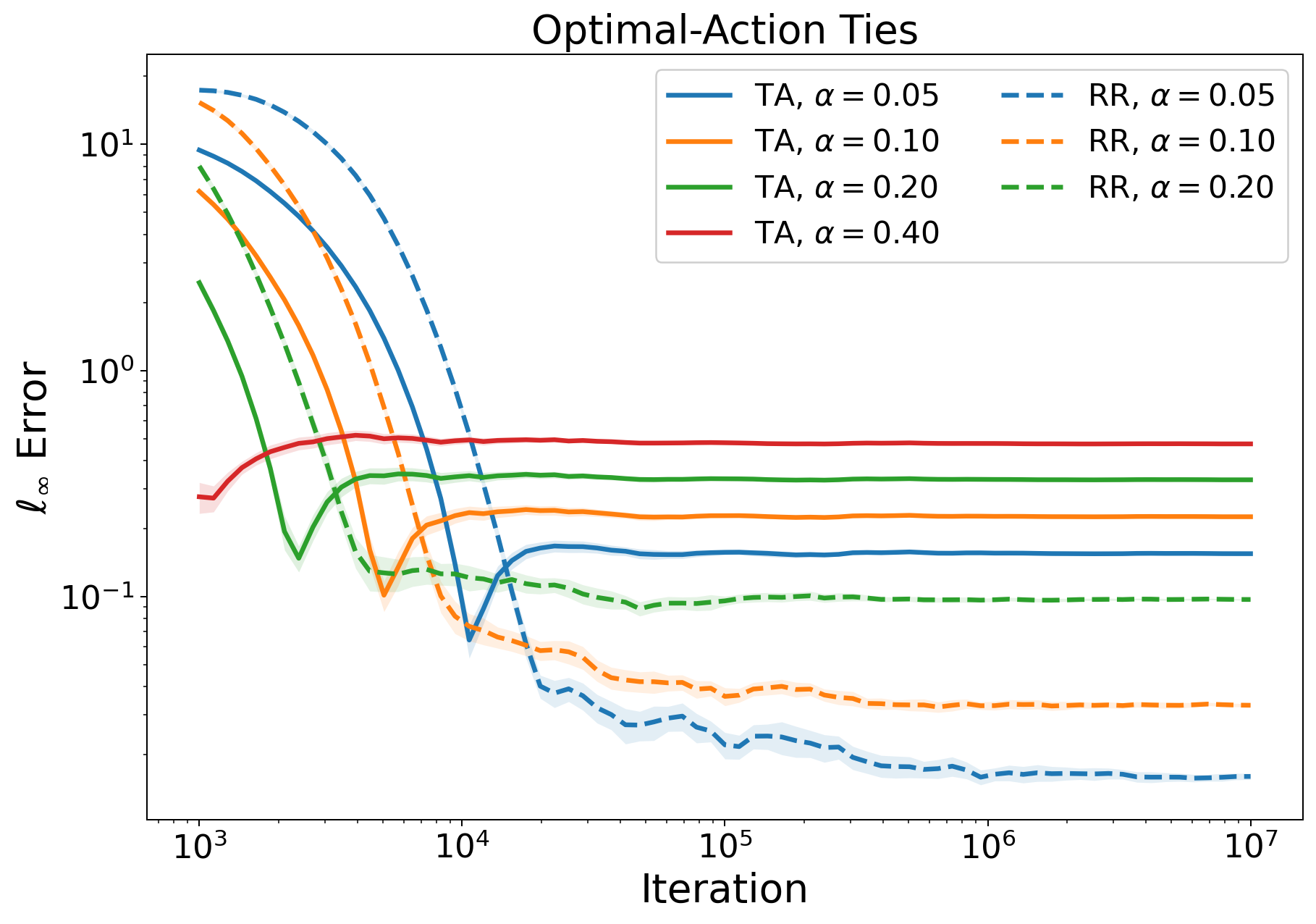}
    \end{minipage}
    \caption{
        Tail averaging and regime-specific RR extrapolation.
    }
    \label{fig:qlearning-correct-rr}
\end{figure}

\noindent\textbf{Misspecified and Unified RR Extrapolation.}
We next examine two alternatives when the bias regime is unknown: a
misspecified extrapolation that is designed for the wrong regime, and the
unified RR extrapolation \eqref{eq:unified-extrapolation}, which simultaneously cancels the
\(\alpha\)- and \(\sqrt{\alpha}\)-order candidate terms.

Figure~\ref{fig:qlearning-wrong-unified} summarizes the comparison. In the
no-tie MDP, the extrapolation designed for a \(\sqrt{\alpha}\)-order bias
fails to remove the dominant linear term and can even degrade performance.
In the tied MDP, the linear-bias RR extrapolation yields only limited
improvement. By contrast, the unified RR extrapolation substantially reduces the
long-run error in both MDPs. This is consistent with the discussion in
Section~\ref{sec:q-ties}: the leading-order bias is determined by the unknown
local structure of the underlying MDP, whereas the unified RR procedure
provides a robust bias-reduction scheme without prior knowledge of the
regime.
\begin{figure}[!htbp]
    \centering
    \begin{minipage}{0.49\textwidth}
        \centering
        \includegraphics[width=\textwidth]{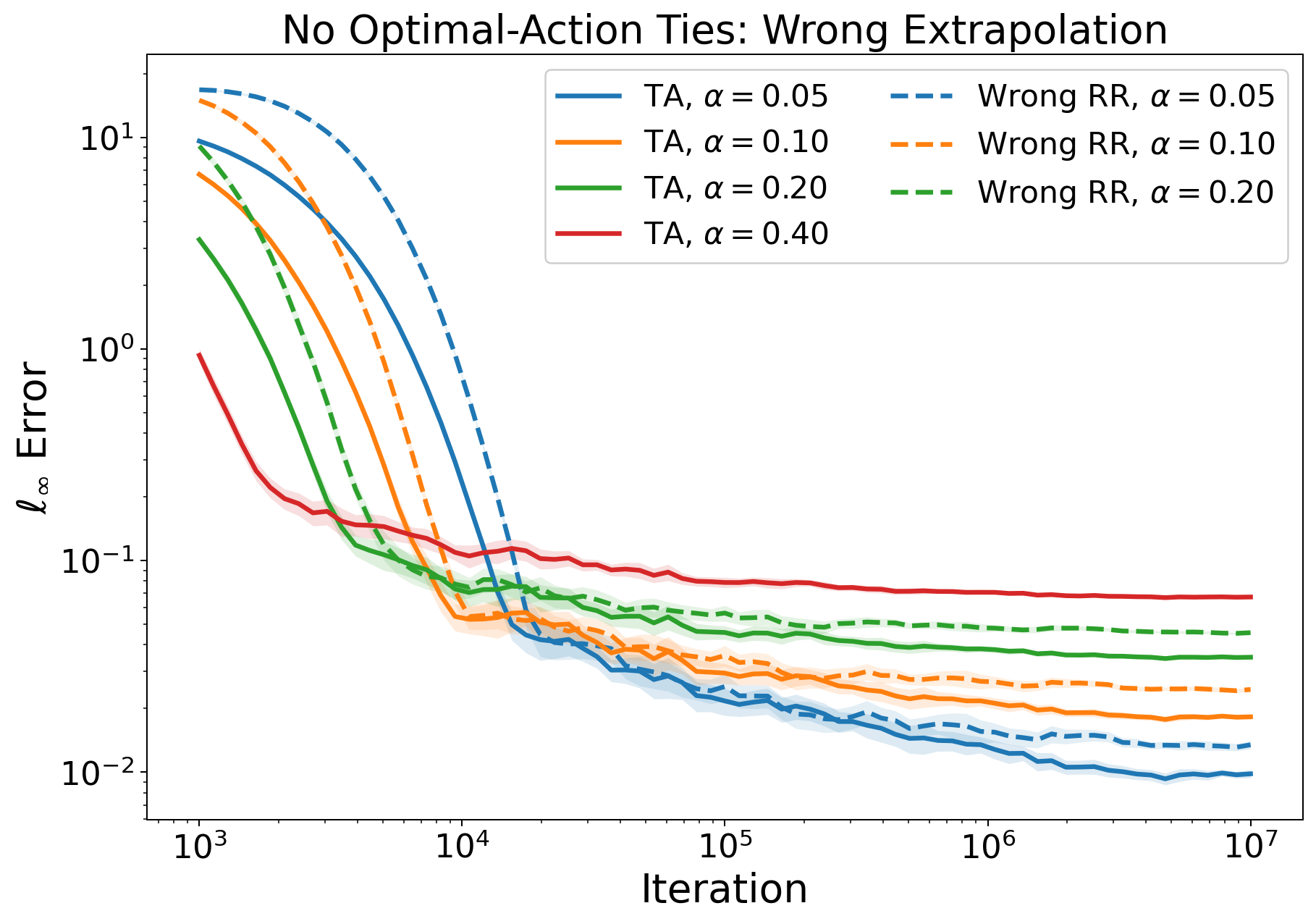}
    \end{minipage}
    \hfill
    \begin{minipage}{0.49\textwidth}
        \centering
        \includegraphics[width=\textwidth]{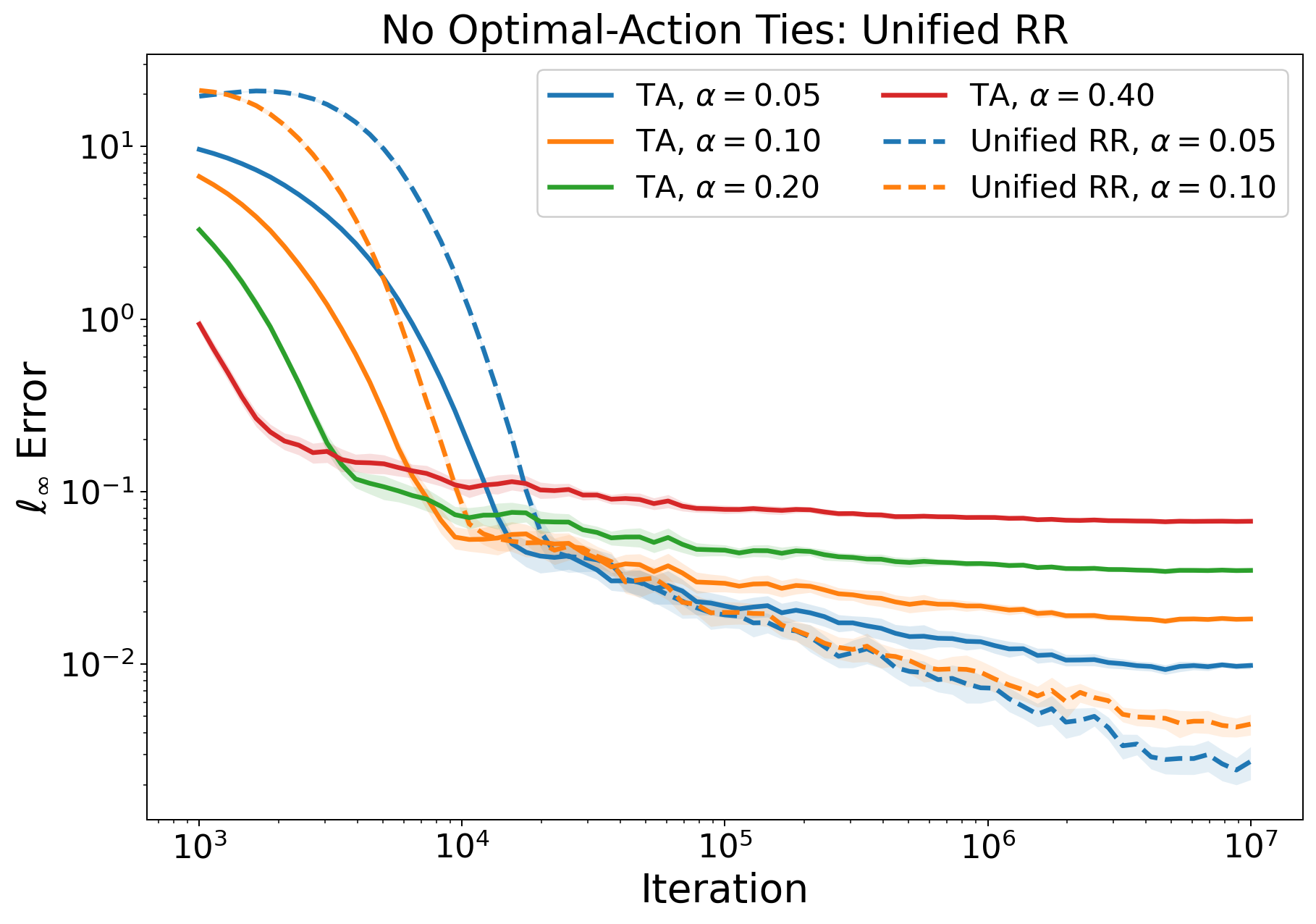}
    \end{minipage}

    \vspace{0.6em}

    \begin{minipage}{0.49\textwidth}
        \centering
        \includegraphics[width=\textwidth]{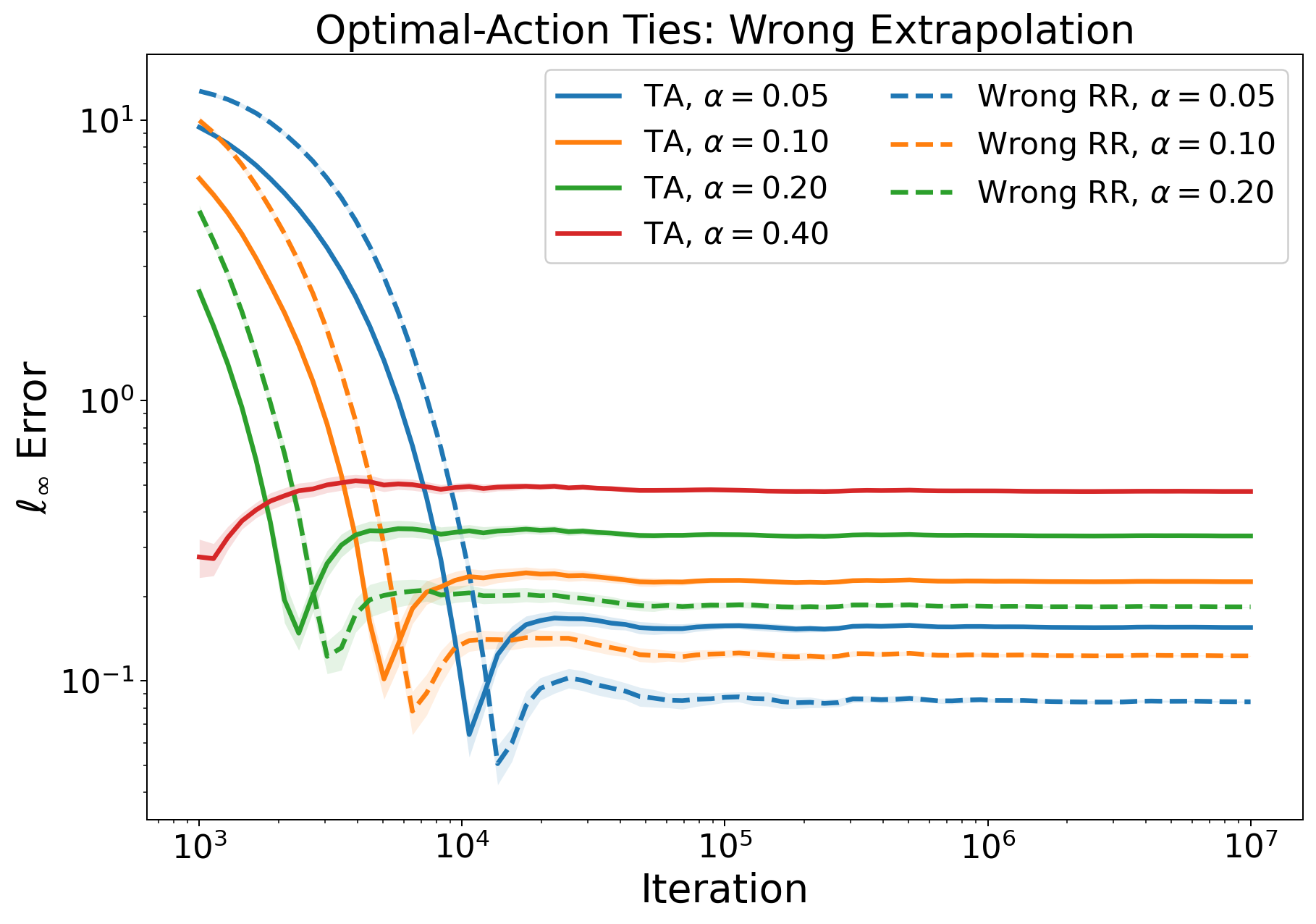}
    \end{minipage}
    \hfill
    \begin{minipage}{0.49\textwidth}
        \centering
        \includegraphics[width=\textwidth]{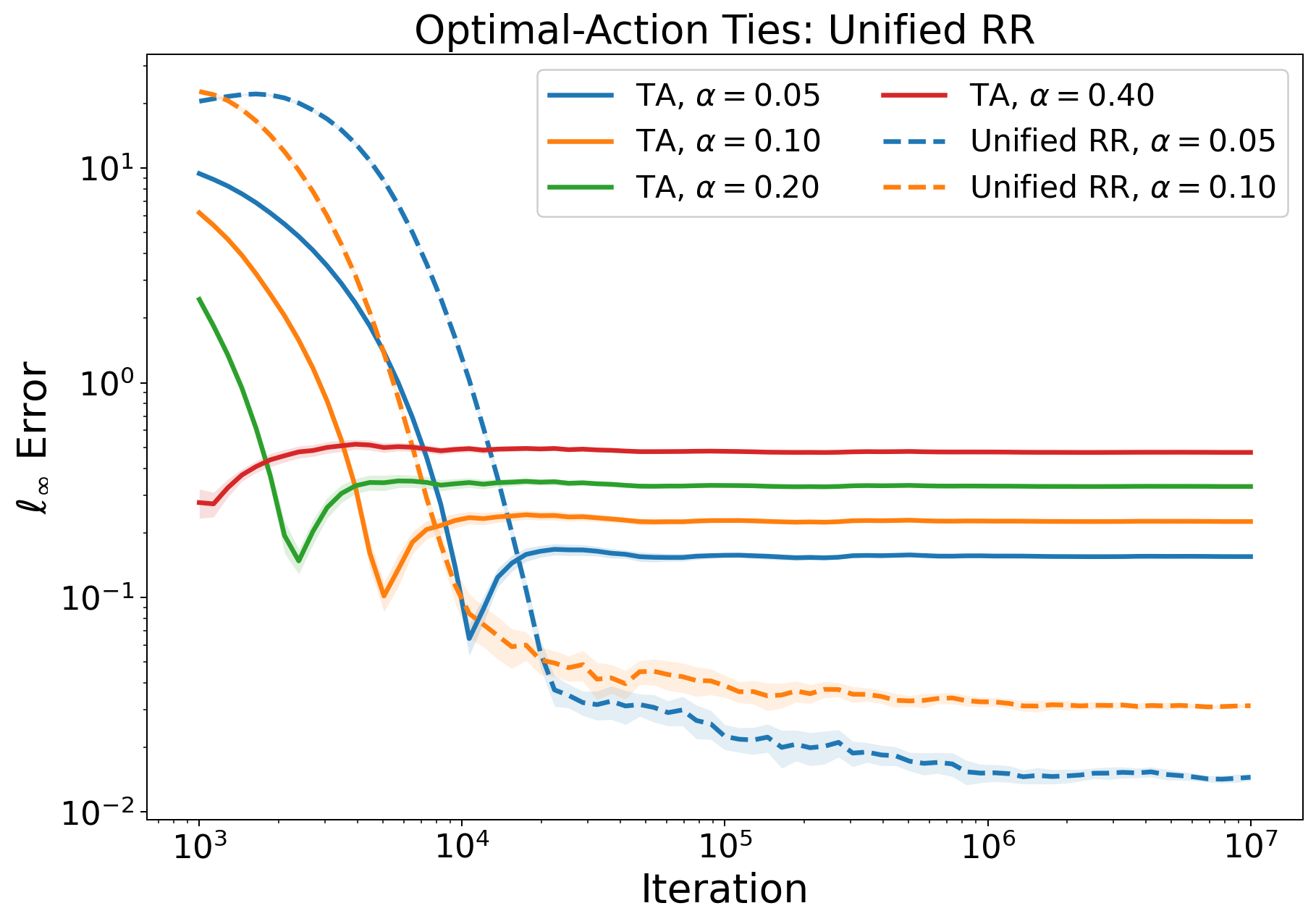}
    \end{minipage}
    \caption{
        Misspecified and unified RR extrapolations.
    }
    \label{fig:qlearning-wrong-unified}
\end{figure}

\section{Conclusion}

In this work, we develop a steady-state convergence (SSC) theory for
constant-stepsize contractive stochastic approximation (SA). Through a multi-step universality framework, we establish steady-state
Gaussian approximations under local quadratic linearization and
develop a general SSC theory for locally one-sided directionally
differentiable mean operators. We further characterize conditions under which
the asymptotic bias is genuinely of order \(\sqrt{\alpha}\). Our results substantially extend existing SSC theory and
yield several applications, including Gaussian approximation guarantees for
raw iterates under local quadratic linearization and principled bias
reduction in the locally nondifferentiable regime. We illustrate the theory
through Markovian linear SA and asynchronous Q-learning. In particular,
for asynchronous Q-learning, we propose a unified bias-reduction scheme,
whose effectiveness is supported by our numerical experiments.

One promising direction for future work is to characterize necessary
conditions for the asymptotic bias to be genuinely of order
\(\sqrt{\alpha}\). Moreover, as conjectured in recent work
\cite{wang2026steady}, the steady-state behavior may differ substantially
when the mean operator is not globally contractive. Developing an SSC theory
with an appropriate scaling in such settings is another interesting
direction.

\section*{Acknowledgments}
Y.\ Zhang and Q.\ Xie are supported in part by NSF grants CNS-1955997, EPCN-2339794 and EPCN-2432546.

\bibliographystyle{alpha} 
\bibliography{main}

\clearpage
\appendix
\section*{Appendices}
\addcontentsline{toc}{section}{Appendices}
\numberwithin{equation}{section}
\numberwithin{theorem}{section}
\numberwithin{lemma}{section}
\numberwithin{proposition}{section}
\numberwithin{corollary}{section}
\numberwithin{assumption}{section}
\numberwithin{definition}{section}
\numberwithin{remark}{section}
\numberwithin{example}{section}
\numberwithin{question}{section}
\numberwithin{property}{section}
\numberwithin{condition}{section}
\numberwithin{observation}{section}
\numberwithin{fact}{section}
\numberwithin{figure}{section}
\numberwithin{table}{section}

\section{Preliminaries}
\subsection{Moreau Envelope}
Let $\phi:\R^d\to\R$ be defined by $\phi(x):=\tfrac12\|x\|_c^2$, where $\|\cdot\|_c$ may be nondifferentiable.  
To handle this, we work with the \emph{Moreau envelope} \citep{parikh2014proximal,chen2024lyapunov,chen2020finite}, which serves as a smooth surrogate of $\phi$.  
Since all norms on $\R^d$ are equivalent \citep{folland1999real}, there exist constants $0<l_{cs}\le u_{cs}<\infty$ such that
\[
l_{cs}\|x\|_2 \le \|x\|_c \le u_{cs}\|x\|_2,\qquad \forall x\in\R^d.
\]
For any $\eta>0$, the Moreau envelope of $\phi$ with respect to the Euclidean prox function $\tfrac12\|\cdot\|_2^2$ is
\begin{equation}\label{eq:Moreau-envelope}
M_\eta(x) = \inf_{u \in \R^d} \Big\{\phi(u) + \frac{1}{2\eta}\|x-u\|_2^2\Big\},\qquad \forall x\in \R^d.
\end{equation}
The basic properties of $M_{\eta}$ are summarized below. 
\begin{lemma}
[Proposition 6 in \cite{zhang2024prelimit}]\label{lem:Moreau-envelope}
$M_\eta$ has the following properties: (1) $M_\eta$ is convex and $\frac{1}{\eta}$-smooth with respect to $\|\cdot\|_2$; (2) there exists a norm $\|\cdot\|_m$ such that $M_\eta(x) = \frac{1}{2}\|x\|_m^2$; (3) it holds that $l_{cm}\|\cdot\|_m\leq \|\cdot\|_c \leq u_{cm}\|\cdot\|_m$, where $l_{cm} = (1+\eta l_{cs}^2)^{\frac{1}{2}} $ and $u_{cm} = (1+\eta u_{cs}^2)^{\frac{1}{2}}$; (4) $\langle \nabla M_\eta(x), y\rangle \leq \|x\|_m \|y\|_m,$ $\forall x, y \in \R^d$; (5) $\nabla M_\eta(cx) = c\nabla M_\eta(x)$, $\forall c \geq 0, x\in \R^d$.
\end{lemma}
The proof of items (1)--(4) above can be found in \cite[Proposition 1]{chen2024lyapunov} and \cite[Lemma A.1]{chen2020finite}. Item (5) follows from (1), (2) and Euler’s homogeneous function theorem \citep{euler2000foundations}.

\subsection{Preliminaries for SA}
In this subsection, we collect several auxiliary results of SA that will be used
throughout the proof. We begin with the following increment bound for
Markovian SA.

\begin{lemma}[Lemma~4 in \cite{chen2024lyapunov}]
\label{lem:increment-markovian}
Under Assumptions~\ref{assumption:contraction} and
\ref{assumption:noise:markovian}, for any \(\alpha>0\) and \(m\geq1\)
satisfying $\alpha m \lesssim 1,$ we have
\[
    \|\theta_{t+m}-\theta_t\|
    \lesssim
    \alpha m\bigl(\|\theta_t\|+1\bigr),
    \qquad \forall t\geq0.
\]
\end{lemma}
We also record a fourth-moment bound for the SA iterates. Such a bound was
established for the i.i.d.\ setting in \cite{zhang2024prelimit} and for
Markovian monotone SA in \cite{hadavi2026revisiting}. The extension to the
general Markovian SA is straightforward: one can follow the argument of
\cite{hadavi2026revisiting} verbatim, replacing its Euclidean Lyapunov
function by the Moreau-envelope Lyapunov function \eqref{eq:Moreau-envelope} to accommodate a general contraction norm.
We therefore omit the proof.

\begin{lemma}
\label{lem:fourthmoment}
Under Assumptions~\ref{assumption:contraction} and
\ref{assumption:noise}, there exist constants
\(c_0,c_1,\alpha_0>0\), independent of \(\alpha\) and \(t\), such that,
for every stepsize \(\alpha\in(0,\alpha_0)\), there exists
\(t_\alpha\geq0\) for which
\[
    \E\!\left[
        \|\theta_t^{(\alpha)}-\theta^*\|^4
    \right]
    \leq
    c_0 e^{-c_1\alpha t}
    \E\!\left[
        \|\theta_0^{(\alpha)}-\theta^*\|^4
    \right]
    +
    c_0\alpha^2,
    \qquad
    \forall t\geq t_\alpha.
\]
\end{lemma}

\subsection{Gaussian Approximation Rates}\label{sec:prior-GA}
In this subsection, we collect several Gaussian approximation results that
will be used throughout the proof. We begin with a direct consequence of the
Gaussian approximation result for Markov-chain-induced martingale differences
in \cite[Theorem~1]{zhang2026martingale}.

\begin{lemma}\label{lemma:GA-MDS}
Let $(\mathcal X,\mathcal B)$ be a standard Borel space, and let
$\{x_t\}_{t\geq0}$ be a stationary Markov chain with transition kernel
$P$ and stationary distribution $\mu$. Assume that the chain is uniformly
ergodic in the sense of \eqref{eq:uniform-GE}. Let
$h:\mathcal X\to\R^d$ satisfy
\[
    \E_{x \sim \mu}[h(x)]=0,
    \qquad
    \sup_{x \in \mathcal X}\|h(x)\|<\infty,
\]
and let $g:=\sum_{m=0}^{\infty}P^m h$ be the canonical solution of $g-Pg=h$. For $i=1,\ldots,n$, let
\[
    D_i:=g(x_i)-Pg(x_{i-1}),
    \qquad
    Y_i:=M_iD_i,
    \qquad
    S_n:=\sum_{i=1}^nY_i,
\]
where $M_1,\ldots,M_n\in\R^{d\times d}$ are deterministic matrices. Suppose that $\operatorname{Cov}(S_n)=I_d.$ Then there exists a finite constant $C_{d,P,h}>0$, depending only on
$d$, $P$, and $h$, such that
\begin{equation}\label{eq:main-bound}
    \W_2^2\!\left(
        \law(S_n),\mathcal N(0,I_d)
    \right)
    \leq
    C_{d,P,h}
    \left(
        1+
        \sqrt{\sum_{i=1}^n\|M_i\|_{\mathrm{op}}^2}
    \right)
    \sum_{i=1}^n\|M_i\|_{\mathrm{op}}^4,
\end{equation}
where \(\|\cdot\|_{\mathrm{op}}\) denotes the matrix operator norm induced by
the Euclidean norm.
\end{lemma}
The next lemma follows from \cite[Corollary~1 and Section~EC.3]{zhang2026martingale}.
The possible singularity of $\Sigma_h$ causes no additional difficulty.
Indeed, the Poisson decomposition writes $S_n$ as a stationary
martingale sum plus a uniformly bounded boundary term, and identifies
$\Sigma_h$ with the covariance of a single martingale increment.
Consequently, the martingale component vanishes almost surely on
$\ker(\Sigma_h)$, while its component on $\operatorname{Range}(\Sigma_h)$
can be whitened and treated by \cite[Corollary~1]{zhang2026martingale}
with $p=2$. The boundary term contributes only $O(n^{-1/2})$ in
$\W_2$. We omit
the routine details.
\begin{lemma}\label{lem:GA-MC}
Consider the same Markov chain $\{x_t\}_{t\geq0}$ and function
$h:\mathcal X\to\R^d$ as in Lemma~\ref{lemma:GA-MDS}. Define
\[
    S_n:=\sum_{i=1}^n h(x_i),
    \qquad
    \Sigma_h:=
    \lim_{n\to\infty}
    \frac{1}{n}\operatorname{Cov}(S_n).
\]
Then there exists a finite constant $C_{d,P,h}>0$, depending only on
$d$, $P$, and $h$, such that
\[
    \W_2\left(
        \law\left(S_n/\sqrt n\right),
        \mathcal N(0,\Sigma_h)
    \right)
    \leq C_{d,P,h}/\sqrt{n}.
\]
\end{lemma}
Finally, we record a Gaussian approximation bound for sums of independent
random vectors. While the presentation in \cite{bonis2020stein} focuses on
the balanced-increment regime, we require the following nonuniform
formulation. For completeness, we provide a detailed proof based on the
techniques developed in \cite{bonis2020stein}.
\begin{lemma}\label{lem:GA-independent}
Let \(Y_1,\dots,Y_n\) be independent centered random vectors in
\(\mathbb R^d\), with finite fourth moments. Suppose that $\operatorname{Cov}\left(\sum_{i=1}^n Y_i\right)=I_d .$ Then there exists a constant \(C_d<\infty\), depending only on \(d\), such that
\[
\mathcal W_2^2
\Big(
    \law\big(\sum_{i=1}^n Y_i\big),
    \mathcal N(0,I_d)
\Big)
\le
C_d
\sum_{i=1}^n
\mathbb E\left[\|Y_i\|^4\right] .
\]
\end{lemma}

\begin{proof}[Proof of Lemma \ref{lem:GA-independent}]
Let $S:=\sum_{i=1}^n Y_i$ and $A:=\sum_{i=1}^n \mathbb E\|Y_i\|^4 .$ Throughout the proof, \(C\) denotes a universal constant and \(C_d\) denotes
a constant depending only on \(d\); both may change from line to line. Let \(Y_1',\dots,Y_n'\) be independent copies of
\(Y_1,\dots,Y_n\), independent of everything else. Let \(I\) be uniformly
distributed on \(\{1,\dots,n\}\). For \(t>0\), define $ \Delta(t):=e^{2t}-1,$ and
\[
    B_i(t):=\left\{
        \|Y_i\|\vee \|Y_i'\|\le \sqrt{\Delta(t)}
    \right\},
    \qquad
    D_i(t):=(Y_i'-Y_i)\mathbf 1_{B_i(t)} .
\]
Set $ S_t:=S+D_I(t).$ Conditionally on \(I=i\), the \(i\)-th summand \(Y_i\) in \(S\) is replaced by $Y_i'\mathds 1_{B_i(t)}
    +
    Y_i\mathds 1_{B_i(t)^c}.$ Since the event \(B_i(t)\) is symmetric in \((Y_i,Y_i')\), this replacement
has the same distribution as \(Y_i\). Hence \(S_t\) has the same law as \(S\).
Moreover, $\|S_t-S\|\le 2\sqrt{\Delta(t)} .$ Applying \cite[Theorem 2]{bonis2020stein} with
\(\nu=\operatorname{law}(S)\), \(X_0=S\), \(X_t=S_t\), and \(s=1/n\), we obtain
\[
    \mathcal W_2\bigl(\law(S),\mathcal N(0,I_d)\bigr)
    \le
    \int_0^\infty e^{-t}\bigl(\mathbb E H(t)\bigr)^{1/2}\,dt,
\]
where
\[
H(t)
=
\left\|
    \mathbb E\left[nD_I(t)+S\,\middle|\,S\right]
\right\|^2                                                        +
\frac{1}{\Delta(t)}
\left\|
    \mathbb E\left[
        \frac n2 D_I(t)^{\otimes 2}-I_d
        \,\middle|\,S
    \right]
\right\|^2                                                  +
\sum_{k\ge 3}
\frac{n^2}{k\,k!\,\Delta(t)^{k-1}}
\left\|
    \mathbb E\left[
        D_I(t)^{\otimes k}
        \,\middle|\,S
    \right]
\right\|^2 .
\]
Since \(I\) is uniform, this can be rewritten as
\[
H(t)=
\left\|
    \mathbb E\left[
        \sum_{i=1}^n \bigl(D_i(t)+Y_i\bigr)
        \,\middle|\,S
    \right]
\right\|^2                                                      +
\frac{\left\|
    \mathbb E\left[
        \frac12\sum_{i=1}^n D_i(t)^{\otimes 2}-I_d
        \,\middle|\,S
    \right]
\right\|^2  }{\Delta(t)}
                                                     +
\sum_{k\ge 3}
\frac{\left\|
    \mathbb E\left[
        \sum_{i=1}^n D_i(t)^{\otimes k}
        \,\middle|\,S
    \right]
\right\|^2}{k\,k!\,\Delta(t)^{k-1}}
 .
\]
We now bound the three terms in \(\mathbb E H(t)\). First, define $h_i(Y_i)
    :=
    \mathbb E\left[
        (Y_i-Y_i')\mathbf 1_{B_i(t)^c}
        \,\middle|\,Y_i
    \right].$ Since $D_i(t)+Y_i
    =
    Y_i'\mathbf 1_{B_i(t)}
    +
    Y_i\mathbf 1_{B_i(t)^c},$ and \(\mathbb E[Y_i'\mid Y_i]=0\), we have $\mathbb E\left[D_i(t)+Y_i\,\middle|\,Y_i\right]
    =
    h_i(Y_i).$ Thus, by the tower property and Jensen's inequality,
\[
\mathbb E\Big[
\big\|
    \mathbb E[
        \sum_{i=1}^n \bigl(D_i(t)+Y_i\bigr)
        \mid S
    ]
\big\|^2 \Big]                                                     \le
\mathbb E\Big[
\big\|
    \sum_{i=1}^n h_i(Y_i)
\big\|^2\Big] .
\]
The variables \(h_i(Y_i)\) are independent and centered because \(B_i(t)^c\) is symmetric in \((Y_i,Y_i')\), while
\(Y_i-Y_i'\) is antisymmetric. Therefore,
\[
\mathbb E
\Big[\big\|
    \sum_{i=1}^n h_i(Y_i)
\big\|\Big]^2
=
\sum_{i=1}^n
\mathbb E[\|h_i(Y_i)\|^2]                 \le
\sum_{i=1}^n
\mathbb E\left[
    \|Y_i-Y_i'\|^2\mathbf 1_{B_i(t)^c}
\right].
\]
Using $ B_i(t)^c
    \subseteq
    \{\|Y_i\|>\sqrt{\Delta(t)}\}
    \cup
    \{\|Y_i'\|>\sqrt{\Delta(t)}\}$ and $ \|Y_i-Y_i'\|^2
    \le
    2\|Y_i\|^2+2\|Y_i'\|^2,$ we get
\[
    \mathbb E\left[
        \|Y_i-Y_i'\|^2\mathbf 1_{B_i(t)^c}
    \right]
    \le
    \frac{C}{\Delta(t)}\mathbb E\|Y_i\|^4 .
\]
Hence
\[
\mathbb E\Big[\big\|
    \mathbb E\big[
        \sum_{i=1}^n \bigl(D_i(t)+Y_i\bigr)
        \mid S
    \big]
\big\|^2\Big]
\le
\frac{CA}{\Delta(t)} .
\]
For the second-order term, write $\Sigma_i:=\mathbb E\left[Y_i^{\otimes 2}\right].$ Since \(\sum_{i=1}^n\Sigma_i=I_d\), we have
\[
\frac12\sum_{i=1}^nD_i(t)^{\otimes 2}-I_d
=
\frac12
\sum_{i=1}^n
\left[
    (Y_i'-Y_i)^{\otimes 2}
    -
    2\Sigma_i
\right]                                                    -
\frac12
\sum_{i=1}^n
(Y_i'-Y_i)^{\otimes 2}\mathbf 1_{B_i(t)^c}                 =: U(t)-V(t).
\]
The summands defining \(U(t)\) are independent and centered. Moreover,
\[
    \mathbb E\left[\left\|(Y_i'-Y_i)^{\otimes 2}\right\|^2\right ]
    =
    \mathbb E\left[\|Y_i'-Y_i\|^4\right ]
    \le
    C\mathbb E\left[\|Y_i\|^4\right],\qquad\|\Sigma_i\|^2
    \le
    \left(\mathbb E\left[\|Y_i\|^2\right]\right)^2
    \le
    \mathbb E\left[\|Y_i\|^4\right] .
\]
Therefore, $\mathbb E\left[\|U(t)\|^2\right]\le CA .$ Next, for \(V(t)\), the previous tail bound gives
\[
    \sum_{i=1}^n
    \mathbb E\left[
        \|Y_i'-Y_i\|^2\mathbf 1_{B_i(t)^c}
    \right]
    \le
    \frac{CA}{\Delta(t)}.
\]
On the other hand, using independence and centeredness,
\[
    \sum_{i=1}^n \mathbb E\left[\|Y_i'-Y_i\|^2\right]
    =
    2\sum_{i=1}^n\mathbb E\left[\|Y_i\|^2\right]
    =
    2d.
\]
Consequently,
\[
    \|\mathbb E [V(t)]\|
    \le
    C_d\min\left\{
        1,\frac{A}{\Delta(t)}
    \right\}.
\]
Furthermore,
\[
\sum_{i=1}^n
\mathbb E
\left[\left\|
    (Y_i'-Y_i)^{\otimes 2}\mathbf 1_{B_i(t)^c}
\right\|^2\right]
\le
\sum_{i=1}^n
\mathbb E\left[\|Y_i'-Y_i\|^4\right]                           \le
CA .
\]
Since the summands in \(V(t)\) are independent, we obtain
\[
\mathbb E\left[\|V(t)\|^2\right] \le
2\mathbb E\left[\|V(t)-\mathbb EV(t)\|^2\right]
+
2\|\mathbb E\left[V(t)\right]\|^2                                \le
C_dA
+
C_d\min\left\{
    1,\frac{A^2}{\Delta(t)^2}
\right\}.
\]
By Jensen's inequality,
\[
\frac{1}{\Delta(t)}
\mathbb E\Big[
\big\|
    \mathbb E\big[
        \frac12\sum_{i=1}^nD_i(t)^{\otimes 2}-I_d
        \mid S
    \big]
\big\|^2\Big]                                                 \le
\frac{C_dA}{\Delta(t)}
+
C_d\min\left\{
    \frac1{\Delta(t)},
    \frac{A^2}{\Delta(t)^3}
\right\}.
\]
It remains to control the higher-order terms. Let \(k\ge3\). Since
\(\|D_i(t)\|\le 2\sqrt{\Delta(t)}\),
\[
    \sum_{i=1}^n
    \mathbb E\left[\|D_i(t)\|^{2k}\right]
    \le
    C^k \Delta(t)^{k-2} A .
\]
For the mean terms, if \(k\) is odd, then $\mathbb E\left[D_i(t)^{\otimes k}\right]=0$ by symmetry. If \(k\ge4\) is even, then
\[
\left\|
    \sum_{i=1}^n
    \mathbb E\left[D_i(t)^{\otimes k}\right]
\right\|
 \le
\sum_{i=1}^n
\mathbb E\left[\|D_i(t)\|^k\right]                                   \le
C^k \Delta(t)^{(k-4)/2}A .
\]
We also have the crude estimate
\[
    \sum_{i=1}^n
    \mathbb E\left[\|D_i(t)\|^k\right]
    \le
    C_d C^k \Delta(t)^{(k-2)/2}.
\]
Therefore, for every \(k\ge3\),
\[
\begin{aligned}
\mathbb E\Big[
\big\|
    \sum_{i=1}^n D_i(t)^{\otimes k}
\big\|^2\Big]
\le
C^k\Delta(t)^{k-2}A
+
C_d C^k
\min\left\{
    \Delta(t)^{k-2},
    A^2\Delta(t)^{k-4}
\right\}.
\end{aligned}
\]
Another application of Jensen's inequality gives
\[
\begin{aligned}
\sum_{k\ge 3}
\frac{\left\|
    \mathbb E\left[
        \sum_{i=1}^n D_i(t)^{\otimes k}
        \,\middle|\,S
    \right]
\right\|^2}{k\,k!\,\Delta(t)^{k-1}} \le
C_d
\left(
    \frac{A}{\Delta(t)}
    +
    \min\left\{
        \frac1{\Delta(t)},
        \frac{A^2}{\Delta(t)^3}
    \right\}
\right),
\end{aligned}
\]
because \(\sum_{k\ge3}C^k/(k\,k!)<\infty\). Combining the first-, second-, and higher-order estimates yields
\[
    \mathbb E [H(t)]
    \le
    C_d
    \left(
        \frac{A}{\Delta(t)}
        +
        \min\left\{
            \frac1{\Delta(t)},
            \frac{A^2}{\Delta(t)^3}
        \right\}
    \right).
\]
Therefore,
\[
\begin{aligned}
\mathcal W_2\bigl(\law(S),\mathcal N(0,I_d)\bigr)
&\le
C_d
\int_0^\infty e^{-t}
\left[
    \left(\frac{A}{\Delta(t)}\right)^{1/2}
    +
    \min\left\{
        \frac1{\Delta(t)^{1/2}},
        \frac{A}{\Delta(t)^{3/2}}
    \right\}
\right]\,dt .
\end{aligned}
\]
Since $\int_0^\infty e^{-t}\Delta(t)^{-1/2}\d t=1,$ the first integral is bounded by \(C_d\sqrt A\). For the second integral,
use the change of variables \(u=\Delta(t)\). Then $  e^{-t}\d t
    =
    \frac12(1+u)^{-3/2}\d u
    \le
    \frac12\d u,$
and hence
\[
\begin{aligned}
\int_0^\infty e^{-t}
\min\left\{
    \frac1{\Delta(t)^{1/2}},
    \frac{A}{\Delta(t)^{3/2}}
\right\}\d t                                           &\le
\frac12
\int_0^\infty
\min\left\{
    u^{-1/2},
    A u^{-3/2}
\right\}\d u                                                  \\
&=
\frac12
\left(
    \int_0^A u^{-1/2}\d u
    +
    \int_A^\infty A u^{-3/2}\d u
\right)
\le
2\sqrt A .
\end{aligned}
\]
Thus
\[
    \mathcal W_2\bigl(\law(S),\mathcal N(0,I_d)\bigr)
    \le
    C_d\sqrt A .
\]
Squaring the above inequality completes the proof of Lemma \ref{lem:GA-independent}.
\end{proof}

\section{Proof of Proposition~\ref{prop:additive}}
\label{sec:proof-additive}
We couple the original recursion~\eqref{eq:update} and the additive-noise
recursion~\eqref{eq:auxiliary-additive} using the same data sequence, with
\[
    \theta_0^{(\alpha)}=a_0^{(\alpha)}=\theta^*,
    \qquad x_0\sim\mu.
\]
Only the data process is initialized in stationarity; no joint stationary
law for the coupled iterates is assumed. Define
\[
    Y_t^{(\alpha)}
    :=\frac{\theta_t^{(\alpha)}-\theta^*}{\sqrt\alpha},
    \qquad
    A_t^{(\alpha)}
    :=\frac{a_t^{(\alpha)}-\theta^*}{\sqrt\alpha}.
\]
We suppress the superscript $(\alpha)$ when there is no ambiguity.
Throughout this proof, $C<\infty$ denotes a constant independent of
$\alpha$ and $t$, whose value may change from line to line. Throughout
the remainder of the paper, $\alpha_0>0$ denotes a generic stepsize
threshold, while $t_\alpha$ denotes a generic time depending on $\alpha$;
their values may change from one occurrence to another.

Fix $r\in(\gamma,\sqrt\gamma)$ and then choose $\eta>0$ sufficiently
small that
\begin{equation}
\label{eq:additive-contraction-margin}
    \gamma\frac{u_{cm}}{l_{cm}}\leq r.
\end{equation}
Such a choice is possible because $u_{cm}/l_{cm}\to1$ as
$\eta\downarrow0$. Both $r$ and $\eta$ are fixed independently of
$\alpha$ and $t$. Let
\[
    \Delta_t:=Y_t-A_t,
    \qquad m_t:=\E[M_\eta(\Delta_t)].
\]
By Lemma~\ref{lem:fourthmoment}, applied separately to the two recursions
with the initialization above, for every sufficiently small $\alpha$
there is $t_\alpha$ such that
\begin{equation}
\label{eq:additive-uniform-moments}
    \sup_{t\geq t_\alpha}
    \left(\E\|Y_t\|^4+\E\|A_t\|^4+m_t\right)\leq C.
\end{equation}
Writing $h_t:=h(x_t)$, the scaled recursions satisfy
\begin{align*}
    Y_{t+1}
    &=(1-\alpha)Y_t+\sqrt\alpha\bigl(
        \tmT(x_t,\theta^*+\sqrt\alpha Y_t)
        -\tmT(x_t,\theta^*)+h_t\bigr),\\
    A_{t+1}
    &=(1-\alpha)A_t+\sqrt\alpha\bigl(
        \mT(\theta^*+\sqrt\alpha A_t)-\mT(\theta^*)+h_t\bigr).
\end{align*}
Subtracting gives
\begin{equation}
\label{eq:additive-delta-recursion}
    \Delta_{t+1}=(1-\alpha)\Delta_t+R_t,
\end{equation}
where
\[
    R_t:=\sqrt\alpha\Bigl[
        \tmT(x_t,\theta^*+\sqrt\alpha Y_t)-\tmT(x_t,\theta^*)
        -\mT(\theta^*+\sqrt\alpha A_t)+\mT(\theta^*)
    \Bigr].
\]
The smoothness and homogeneity properties in
Lemma~\ref{lem:Moreau-envelope} imply
\begin{equation}
\label{eq:maininduction}
    m_{t+1}
    \leq(1-\alpha)^2m_t+(1-\alpha)T_{1,t}
        +\frac{1}{2\eta}T_{2,t},
    \qquad
    T_{1,t}:=\E\langle\nabla M_\eta(\Delta_t),R_t\rangle,
    \quad T_{2,t}:=\E\|R_t\|^2.
\end{equation}
The following two lemmas provide the needed bounds. In the Markovian
case, the first bound retains a telescoping Poisson term rather than
assuming that its expectation vanishes.

\begin{lemma}
\label{lem:additiveT1}
Under the coupling and initialization above, there exist
$C<\infty$ and $\alpha_0>0$ such that, for every
$\alpha\in(0,\alpha_0)$, there are a time $t_\alpha$ and a real sequence
$\{u_t\}_{t\geq0}$ satisfying, for all $t\geq t_\alpha$,
\begin{equation}
\label{eq:additive-T1-bound}
    T_{1,t}
    \leq2\sqrt\gamma\,\alpha m_t
        +\sqrt\alpha\,(u_t-u_{t+1})+C\alpha^2,
    \qquad
    |u_t|\leq C\sqrt\alpha\,\sqrt{m_t}.
\end{equation}
In the i.i.d.\ case, one may take $u_t=0$. In the Markovian case,
$u_t:=\E[\mathcal U_\alpha(\Delta_t,Y_t,x_t)]$, where
$\mathcal U_\alpha$ is the Poisson solution constructed in
Section~\ref{sec:proof-additiveT1}.
\end{lemma}

\begin{lemma}
\label{lem:additiveT2}
Under the same coupling and initialization, there exist
$C<\infty$ and $\alpha_0>0$ such that, for every
$\alpha\in(0,\alpha_0)$, there is $t_\alpha$ for which
$T_{2,t}\leq C\alpha^2$ for all $t\geq t_\alpha$.
\end{lemma}

Substituting these bounds into \eqref{eq:maininduction} yields
\[
    m_{t+1}
    \leq\bigl((1-\alpha)^2+2\sqrt\gamma\,\alpha(1-\alpha)\bigr)m_t
        +(1-\alpha)\sqrt\alpha\,(u_t-u_{t+1})+C\alpha^2.
\]
Set
\[
    \lambda:=1-\sqrt\gamma>0,
    \qquad q_\alpha:=1-\lambda\alpha,
    \qquad b_\alpha:=(1-\alpha)\sqrt\alpha.
\]
For sufficiently small $\alpha$, the coefficient of $m_t$ is at most
$q_\alpha\in(0,1)$, because it equals
$1-2\lambda\alpha+(1-2\sqrt\gamma)\alpha^2$. Thus
\begin{equation}
\label{eq:additive-poisson-drift}
    m_{t+1}\leq q_\alpha m_t+b_\alpha(u_t-u_{t+1})+C\alpha^2,
    \qquad t\geq t_\alpha.
\end{equation}
Fix an integer $t_0\geq t_\alpha$. Iterating
\eqref{eq:additive-poisson-drift}, for every $t>t_0$,
\begin{align}
    m_t
    &\leq q_\alpha^{t-t_0}m_{t_0}
        +b_\alpha\sum_{k=t_0}^{t-1}
            q_\alpha^{t-1-k}(u_k-u_{k+1})
        +\frac{C\alpha^2}{1-q_\alpha}.
    \label{eq:additive-drift-unrolled}
\end{align}
Summation by parts gives the exact identity
\begin{align*}
    \sum_{k=t_0}^{t-1}q_\alpha^{t-1-k}(u_k-u_{k+1})
    &=q_\alpha^{t-t_0-1}u_{t_0}-u_t
      +(1-q_\alpha)\sum_{k=t_0+1}^{t-1}q_\alpha^{t-1-k}u_k.
\end{align*}
Consequently, using \eqref{eq:additive-T1-bound} and
\eqref{eq:additive-uniform-moments},
\[
    \left|\sum_{k=t_0}^{t-1}
        q_\alpha^{t-1-k}(u_k-u_{k+1})\right|
    \leq2\sup_{k\geq t_0}|u_k|
    \leq C\sqrt\alpha.
\]
Since $b_\alpha\leq\sqrt\alpha$ and
$1-q_\alpha=\lambda\alpha$, \eqref{eq:additive-drift-unrolled} implies
\[
    m_t\leq q_\alpha^{t-t_0}m_{t_0}+C\alpha,
    \qquad
    \limsup_{t\to\infty}m_t\leq C\alpha.
\]
Assumption~\ref{assumption:convergence} gives convergence of the original
marginal law to $\law(Y_\infty^{(\alpha)})$ in $\W_2$ for each fixed
$\alpha$. The corresponding convergence of the additive-noise marginal
to $\law(A_\infty^{(\alpha)})$ follows from the direct contraction
argument in Section~\ref{sec:additive-marginal-convergence} below.
Therefore, the triangle inequality gives
\begin{align*}
    &\W_2\!\left(\law(Y_\infty^{(\alpha)}),
                  \law(A_\infty^{(\alpha)})\right)\\
    &\quad\leq\limsup_{t\to\infty}\Bigl[
        \W_2\!\left(\law(Y_\infty^{(\alpha)}),\law(Y_t)\right)
        +\W_2\!\left(\law(Y_t),\law(A_t)\right)
        +\W_2\!\left(\law(A_t),\law(A_\infty^{(\alpha)})\right)
    \Bigr]\\
    &\quad\leq\limsup_{t\to\infty}
        \bigl(\E\|\Delta_t\|^2\bigr)^{1/2}
    \leq C\sqrt{\limsup_{t\to\infty}m_t}
    \leq C\sqrt\alpha.
\end{align*}
This proves Proposition~\ref{prop:additive}.

\subsection{Proof of Lemma~\ref{lem:additiveT1}}
\label{sec:proof-additiveT1}
Define the centered operator fluctuation
\[
    \varepsilon_\alpha(y,x)
    :=\tmT(x,\theta^*+\sqrt\alpha y)-\tmT(x,\theta^*)
      -\mT(\theta^*+\sqrt\alpha y)+\mT(\theta^*).
\]
Then $\E_\mu[\varepsilon_\alpha(y,x)]=0$ for every fixed $y$, and
\begin{align*}
    T_{1,t}
    &=\sqrt\alpha\,
      \E\left\langle\nabla M_\eta(\Delta_t),
          \mT(\theta^*+\sqrt\alpha Y_t)
          -\mT(\theta^*+\sqrt\alpha A_t)\right\rangle+\sqrt\alpha\,
      \E\left\langle\nabla M_\eta(\Delta_t),
          \varepsilon_\alpha(Y_t,x_t)\right\rangle.
\end{align*}
The mean-operator contribution is bounded in both noise settings by
\begin{align*}
    \sqrt\alpha
      \E\left\langle\nabla M_\eta(\Delta_t),
          \mT(\theta^*+\sqrt\alpha Y_t)
          -\mT(\theta^*+\sqrt\alpha A_t)\right\rangle\leq\frac{\alpha\gamma}{l_{cm}}
        \E\bigl[\|\Delta_t\|_m\|\Delta_t\|_c\bigr]
    \leq\frac{2\alpha\gamma u_{cm}}{l_{cm}}m_t
    \leq2r\alpha m_t,
\end{align*}
where we used Lemma~\ref{lem:Moreau-envelope} and
\eqref{eq:additive-contraction-margin}.

\noindent\textbf{I.i.d.\ noise.}
Here $(Y_t,A_t)$ is a function of $x_0,\ldots,x_{t-1}$ and is independent
of $x_t$. Conditional centering makes the fluctuation contribution zero.
Thus $T_{1,t}\leq2r\alpha m_t\leq2\sqrt\gamma\,\alpha m_t$, proving
\eqref{eq:additive-T1-bound} with $u_t=0$.

\noindent\textbf{Markovian noise.}
For fixed $(\delta,y)\in\R^d\times\R^d$, let
\[
    \mathcal H_\alpha(\delta,y,x)
    :=\langle\nabla M_\eta(\delta),\varepsilon_\alpha(y,x)\rangle.
\]
This function is centered under $\mu$. By
Assumption~\ref{assumption:noise:markovian}, norm equivalence, and the
Lipschitz property of $\nabla M_\eta$,
\[
    \sup_x|\mathcal H_\alpha(\delta,y,x)|
    \leq C\sqrt\alpha\,\|\delta\|_c\|y\|_c.
\]
Define the canonical Poisson solution, with $P$ acting only on the
$x$-coordinate, by
\[
    \mathcal U_\alpha(\delta,y,\cdot)
    :=\sum_{k=0}^{\infty}P^k\mathcal H_\alpha(\delta,y,\cdot).
\]
Uniform ergodicity ensures convergence of this series and gives
\begin{equation}
\label{eq:additive-poisson-equation}
    \mathcal U_\alpha(\delta,y,\cdot)
      -P\mathcal U_\alpha(\delta,y,\cdot)
    =\mathcal H_\alpha(\delta,y,\cdot),
\end{equation}
as well as
\begin{equation}
\label{eq:additive-poisson-size}
    \sup_x|\mathcal U_\alpha(\delta,y,x)|
    \leq C\sqrt\alpha\,\|\delta\|_c\|y\|_c.
\end{equation}
Furthermore, the operator regularity and smoothness of $M_\eta$ imply
\begin{align*}
    |\mathcal H_\alpha(\delta',y',x)
        -\mathcal H_\alpha(\delta,y,x)|\leq C\sqrt\alpha\Bigl[
        (\|y\|_c+\|y'\|_c)\|\delta'-\delta\|_c
        +(\|\delta\|_c+\|\delta'\|_c)\|y'-y\|_c\Bigr].
\end{align*}
The difference on the left is also centered under $\mu$. Summing its
uniform-ergodicity bound yields
\begin{equation}
\label{eq:additive-poisson-regularity}
\begin{aligned}
    |\mathcal U_\alpha(\delta',y',x)
        -\mathcal U_\alpha(\delta,y,x)|\leq C\sqrt\alpha\Bigl[
        (\|y\|_c+\|y'\|_c)\|\delta'-\delta\|_c
        +(\|\delta\|_c+\|\delta'\|_c)\|y'-y\|_c\Bigr].
\end{aligned}
\end{equation}
Let $\mathcal F_t:=\sigma(x_0,\ldots,x_t)$ and set
\[
    \mathcal M_{t+1}
    :=\mathcal U_\alpha(\Delta_t,Y_t,x_{t+1})
       -P\mathcal U_\alpha(\Delta_t,Y_t,x_t).
\]
Because  $\Delta_t$ and $Y_t$ are $\mathcal F_t$-measurable, and the
Markov property gives
$\E[\mathcal M_{t+1}\mid\mathcal F_t]=0$. Equation
\eqref{eq:additive-poisson-equation} therefore gives the decomposition
\begin{equation}
\label{eq:additive-poisson-decomposition}
\begin{aligned}
    \mathcal H_\alpha(\Delta_t,Y_t,x_t)
    &=\mathcal U_\alpha(\Delta_t,Y_t,x_t)
      -\mathcal U_\alpha(\Delta_{t+1},Y_{t+1},x_{t+1})
      +\mathcal R_{t+1}^{\mathrm{PE}}+\mathcal M_{t+1},
\end{aligned}
\end{equation}
where
\[
    \mathcal R_{t+1}^{\mathrm{PE}}
    :=\mathcal U_\alpha(\Delta_{t+1},Y_{t+1},x_{t+1})
       -\mathcal U_\alpha(\Delta_t,Y_t,x_{t+1}).
\]
Thus, writing
$u_t:=\E[\mathcal U_\alpha(\Delta_t,Y_t,x_t)]$, we have the exact identity
\begin{equation}
\label{eq:additive-poisson-expectation}
    \E[\mathcal H_\alpha(\Delta_t,Y_t,x_t)]
    =u_t-u_{t+1}+\E[\mathcal R_{t+1}^{\mathrm{PE}}].
\end{equation}
By \eqref{eq:additive-poisson-size}, Cauchy--Schwarz, and
\eqref{eq:additive-uniform-moments},
\[
    |u_t|
    \leq C\sqrt\alpha\,
          (\E\|\Delta_t\|_c^2)^{1/2}(\E\|Y_t\|_c^2)^{1/2}
    \leq C\sqrt\alpha\,\sqrt{m_t},
    \qquad t\geq t_\alpha.
\]
It remains to bound the Poisson remainder. The coupled difference
recursion and the operator Lipschitz bounds give
\begin{equation}
\label{eq:additive-delta-increment}
    \|\Delta_{t+1}-\Delta_t\|_c
    \leq C\alpha\bigl(\|\Delta_t\|_c+\|Y_t\|_c\bigr).
\end{equation}
Similarly, boundedness of $h$ and the original scaled recursion imply
\begin{equation}
\label{eq:additive-Y-increment}
    \|Y_{t+1}-Y_t\|_c
    \leq C\bigl(\alpha\|Y_t\|_c+\sqrt\alpha\bigr).
\end{equation}
Using these estimates in \eqref{eq:additive-poisson-regularity}, for
$0<\alpha\leq1$, gives
\begin{align*}
    \sqrt\alpha\,\E|\mathcal R_{t+1}^{\mathrm{PE}}|
    &\leq C\alpha^2\Bigl(
        \E\|Y_t\|_c^2+\E[\|\Delta_t\|_c\|Y_t\|_c]\Bigr)
        +C\alpha^{3/2}\E\|\Delta_t\|_c
        +C\alpha^{5/2}\E\|Y_t\|_c\\
    &\leq C\alpha^2+C\alpha^{3/2}\sqrt{m_t},
    \qquad t\geq t_\alpha,
\end{align*}
where the second inequality follows from Cauchy--Schwarz,
\eqref{eq:additive-uniform-moments}, and norm equivalence. The strictly
positive margin $\sqrt\gamma-r$ permits the explicit Young bound
\[
    C\alpha^{3/2}\sqrt{m_t}
    \leq2(\sqrt\gamma-r)\alpha m_t
        +\frac{C^2}{8(\sqrt\gamma-r)}\alpha^2.
\]
Consequently,
\[
    \sqrt\alpha\,\E|\mathcal R_{t+1}^{\mathrm{PE}}|
    \leq2(\sqrt\gamma-r)\alpha m_t+C\alpha^2.
\]
Combining this estimate with \eqref{eq:additive-poisson-expectation}
and the $2r\alpha m_t$ mean-operator bound proves
\[
    T_{1,t}
    \leq\bigl(2r+2(\sqrt\gamma-r)\bigr)\alpha m_t
        +\sqrt\alpha\,(u_t-u_{t+1})+C\alpha^2,
\]
which completes the proof of
Lemma~\ref{lem:additiveT1}.

\subsection{Proof of Lemma~\ref{lem:additiveT2}}
\label{sec:proof-additiveT2}
The definition of $R_t$, the operator regularity, and Jensen's
inequality imply
\begin{align*}
    T_{2,t}
    &\leq2\alpha\,
        \E\|\tmT(x_t,\theta^*+\sqrt\alpha Y_t)-\tmT(x_t,\theta^*)\|^2
       +2\alpha\,
        \E\|\mT(\theta^*+\sqrt\alpha A_t)-\mT(\theta^*)\|^2\\
    &\leq C\alpha^2\bigl(\E\|Y_t\|^2+\E\|A_t\|^2\bigr)
    \leq C\alpha^2,
    \qquad t\geq t_\alpha.
\end{align*}
In the i.i.d.\ case, the first bound on the random-operator increment is
applied conditionally on $x_0,\ldots,x_{t-1}$, using independence of
$x_t$ and the fourth-moment Lipschitz condition. In the Markovian case,
the pointwise Lipschitz condition applies directly. The final inequality
uses \eqref{eq:additive-uniform-moments} and proves the lemma.

\subsection{Marginal Convergence of the Additive-Noise Auxiliary}
\label{sec:additive-marginal-convergence}
For completeness, we justify the auxiliary steady-state law and the
marginal convergence used above without applying
Assumption~\ref{assumption:convergence} to a different recursion.
For $0<\alpha<1$, define
\[
    \Phi_{\alpha,x}(a)
    :=(1-\alpha)a+\alpha\mT(a)+\alpha h(x),
    \qquad \varrho_\alpha:=1-\alpha(1-\gamma)\in(0,1).
\]
The deterministic synchronous contraction is
\begin{equation}
\label{eq:additive-synchronous-contraction}
    \|\Phi_{\alpha,x}(a)-\Phi_{\alpha,x}(a')\|_c
    \leq\varrho_\alpha\|a-a'\|_c,
    \qquad \forall\,a,a'\in\R^d,\ x\in\mathcal X.
\end{equation}
On a two-sided stationary extension $\{x_t\}_{t\in\mathbb Z}$, set
\[
    a_0^{[-n]}
    :=\Phi_{\alpha,x_{-1}}\circ\cdots\circ
        \Phi_{\alpha,x_{-n}}(\theta^*),
    \qquad n\geq1.
\]
Writing $\|Z\|_{L^2(c)}:=(\E\|Z\|_c^2)^{1/2}$, we obtain
\[
    \|a_0^{[-(n+1)]}-a_0^{[-n]}\|_{L^2(c)}
    \leq\alpha\varrho_\alpha^n\|h(x_0)\|_{L^2(c)}.
\]
The right-hand side is summable in $n$. Hence $a_0^{[-n]}$ converges in
$L^2$ to a random variable $a_0^{\mathrm{st}}$. Time shifts of this
construction define a stationary solution
$\{(x_t,a_t^{\mathrm{st}})\}_{t\in\mathbb Z}$ of the additive-noise
recursion, with a finite second moment in its iterate coordinate.
The recursion follows by passing to the limit through the Lipschitz
map $\Phi_{\alpha,x_t}$.

Couple the forward recursion initialized at $a_0=\theta^*$ with this
stationary solution using the same future data. By
\eqref{eq:additive-synchronous-contraction},
\[
    \|a_t-a_t^{\mathrm{st}}\|_{L^2(c)}
    \leq\varrho_\alpha^t
          \|\theta^*-a_0^{\mathrm{st}}\|_{L^2(c)}
    \longrightarrow0.
\]
Thus, with $\law(a_\infty^{(\alpha)}):=\law(a_0^{\mathrm{st}})$,
\[
    \W_2\!\left(\law(A_t),\law(A_\infty^{(\alpha)})\right)
    \longrightarrow0
    \qquad\text{as }t\to\infty
\]
for each fixed $\alpha$. Uniqueness of the invariant law with a finite
iterate second moment follows from the same contraction: couple two
invariant initial laws to have the same data state $x_0\sim\mu$ and the
same future data, drawing their initial iterates from their respective
conditional laws given $x_0$. Their iterate distance converges to zero
in $L^2$, while the data coordinates agree at every time, so their
stationary joint laws coincide. This constructs a stationary
\emph{auxiliary} process only; it does not assume stationarity of the
jointly coupled original and auxiliary iterates.

\section{Proof of Proposition  \ref{prop:Jacobian}}\label{sec:proof-Jacobian}
We study two diffusion-scaled processes,
\(\{A_t^{(\alpha)}\}_{t\ge 0}\) and
\(\{B_t^{(\alpha)}\}_{t\ge 0}\), defined by
\[
A_t^{(\alpha)}
:=
\frac{a_t^{(\alpha)}-\theta^*}{\sqrt{\alpha}},
\qquad
B_t^{(\alpha)}
:=
\frac{b_t^{(\alpha)}-\theta^*}{\sqrt{\alpha}}.
\]
We couple the two recursions using the same noise sequence and initialize
\(a_0^{(\alpha)}=b_0^{(\alpha)}=\theta^*\).
When the dependence on \(\alpha\) is clear, we suppress the
superscript \((\alpha)\) for notational simplicity. From \eqref{eq:auxiliary-additive} and \eqref{eq:auxiliary-Jacobian}, the dynamics of the scaled processes
can be written as
\begin{equation}
\begin{aligned}
A_{t+1}
&=
(1-\alpha)A_t
+
\sqrt{\alpha}\big(
\mT(\sqrt{\alpha}A_t+\theta^*)-\mT(\theta^*)+h_t
\big), \\
B_{t+1}
&=
(1-\alpha)B_t
+
\sqrt{\alpha}\big(
\sqrt{\alpha}J B_t+h_t
\big),
\end{aligned}
\end{equation}
Subtracting the second recursion from the
first gives
\begin{align*}
A_{t+1}-B_{t+1}
&=
(1-\alpha)(A_t-B_t) +
\sqrt{\alpha}\Big(
\mT(\sqrt{\alpha}A_t+\theta^*)
-\mT(\theta^*)
-\sqrt{\alpha}J B_t
\Big)\\
&=\big(I-\alpha I+\alpha J\big)(A_t-B_t) +
\sqrt{\alpha}\Big(
\mT(\sqrt{\alpha}A_t+\theta^*)
-\mT(\theta^*)
-\sqrt{\alpha}J A_t
\Big)
\end{align*}
By Assumption~\ref{assumption:local-quadratic},
\[
    \|\mT(\theta^*+u)-\mT(\theta^*)-Ju\|
    \leq L_R\|u\|^2,
    \qquad \|u\|\leq\epsilon.
\]
Since \(J=\nabla\mT(\theta^*)\), Assumption~\ref{assumption:contraction}
also implies \(\|Ju\|_c\leq\gamma\|u\|_c\) for every \(u\in\R^d\).
Consequently, contraction and norm equivalence give a constant \(C_0<\infty\)
such that
\[
    \|\mT(\theta^*+u)-\mT(\theta^*)-Ju\|
    \leq C_0\|u\|,
    \qquad \forall u\in\R^d.
\]
For \(\|u\|>\epsilon\), the right-hand side is at most
\((C_0/\epsilon)\|u\|^2\). Combining the two bounds, we obtain
\begin{equation}\label{eq:quadratic-remainder}
    \|\mT(\theta^*+u)-\mT(\theta^*)-Ju\|
    \lesssim \|u\|^2,
    \qquad \forall u\in\R^d.
\end{equation}
Let $G_\alpha:=I-\alpha I+\alpha J$. We will use the following lemma, whose proof is deferred to Section~\ref{sec:proof-Jacobian-hurwitz}.
\begin{lemma}\label{lem:Jacobian-hurwitz}
There exist constants
\(\alpha_0>0\), \(\lambda>0\), and a matrix \(P=P^\top\succ0\) such that, for
all \(\alpha\in(0,\alpha_0)\),
\(
G_\alpha^\top P G_\alpha
\preceq
(1-\lambda\alpha)P .
\)
\end{lemma}
Define the \(P\)-norm by $\|z\|_P^2:= z^\top P z$ for $z \in \R^d$. Therefore, by Lemma \ref{lem:Jacobian-hurwitz} and Young's inequality with parameter \(\lambda\alpha/2\),
\begin{align*}
&\E[\|A_{t+1}-B_{t+1}\|_P^2] \\
&\quad = \E[\|G_\alpha(A_t-B_t) +
\sqrt{\alpha}\big(
\mT(\sqrt{\alpha}A_t+\theta^*)
-\mT(\theta^*)
-\sqrt{\alpha}J A_t
\big)\|_P^2]\\
&\leq (1+\lambda\alpha/2)\E[\|G_\alpha(A_t-B_t)\|_P^2] \\
&\qquad + (1+2/(\lambda\alpha))\E[\|
\sqrt{\alpha}\big(
\mT(\sqrt{\alpha}A_t+\theta^*)
-\mT(\theta^*)
-\sqrt{\alpha}J A_t
\big)\|_P^2]\\
&\leq (1+\lambda\alpha/2)(1-\lambda\alpha)\E[\|A_t-B_t\|_P^2] + C\E[\|\mT(\sqrt{\alpha}A_t+\theta^*)
-\mT(\theta^*)
-\sqrt{\alpha}J A_t
\|_P^2]\\
&\leq (1-\lambda\alpha/2)\E[\|A_t-B_t\|_P^2] + C\E[\|\mT(\sqrt{\alpha}A_t+\theta^*)
-\mT(\theta^*)
-\sqrt{\alpha}J A_t
\|_P^2],
\end{align*}
where $C>0$ is a constant independent with $\alpha$ and $t$. By \eqref{eq:quadratic-remainder} and Lemma~\ref{lem:fourthmoment}, for all
sufficiently small \(\alpha>0\) and all \(t\geq t_\alpha\), we have
\begin{align*}
\E[\|\mT(\sqrt{\alpha}A_t+\theta^*)
-\mT(\theta^*)
-\sqrt{\alpha}J A_t
\|_P^2] \lesssim\alpha^2\E[\|A_t\|^4]
\lesssim \alpha^2.
\end{align*}
Therefore, there exists \(\alpha_0>0\) such that, for every
\(\alpha\in(0,\alpha_0)\), there exists \(t_\alpha>0\) for which
\begin{align*}
\E[\|A_{t+1}-B_{t+1}\|_P^2] &\leq (1-\lambda\alpha/2)\E[\|A_t-B_t\|_P^2]  + \mathcal{O}(\alpha^2).
\end{align*}
Therefore, we obtain
$
\limsup_{t \to \infty} \E[\|A_{t}-B_{t}\|_P^2]  \in \mathcal{O}(\alpha).
$
By triangle inequality, we have
\begin{align*}
\W_2\big(\mathcal{L}(A^{(\alpha
)}), \mathcal{L}({B }^{(\alpha
)})\big) &\lesssim \limsup_{t \to \infty} \W_2\big(\mathcal{L}(A^{(\alpha
)}), \mathcal{L}(A_{t})\big) +\W_2\big(\mathcal{L}(A_{t}), \mathcal{L}(B _{t})\big) +\W_2\big(\mathcal{L}(B _{t}), \mathcal{L}({B }^{(\alpha
)})\big)\\
&\lesssim \limsup_{t \to \infty} \sqrt{\E[\|A_{t}-B_{t}\|_P^2]}\in \mathcal{O}(\sqrt{\alpha}),
\end{align*}
which completes the proof of Theorem  \ref{prop:Jacobian}.

\subsection{Proof of Lemma \ref{lem:Jacobian-hurwitz}}\label{sec:proof-Jacobian-hurwitz}
\(H:=J-I_d\) is Hurwitz and since there exists \(P=P^\top\succ0\) such that $H^\top P+PH=-I.$ Therefore,
\begin{align*}
G_\alpha^\top P G_\alpha
=
(I+\alpha H)^\top P(I+\alpha H) =
P+\alpha(H^\top P+PH)+\alpha^2 H^\top P H =
P-\alpha I+\alpha^2 H^\top P H.
\end{align*}
Since \(P\succ 0\), we have $I \succeq \frac{1}{\lambda_{\max}(P)}P.$ Moreover, there exists a constant \(C_H<\infty\) such that $H^\top P H \preceq C_H P.$ Hence,
\[
G_\alpha^\top P G_\alpha
\preceq
P-\frac{\alpha}{\lambda_{\max}(P)}P+\alpha^2 C_H P.
\]
Choose $\alpha_0
\le
\frac{1}{2C_H\lambda_{\max}(P)}$ and set $\lambda:=\frac{1}{2\lambda_{\max}(P)}.$ Then, for all \(\alpha\in(0,\alpha_0)\),
\[
G_\alpha^\top P G_\alpha
\preceq
(1-\lambda\alpha)P,
\]
thereby completing the proof of Lemma \ref{lem:Jacobian-hurwitz}.
\section{Proof of Proposition \ref{prop:geometric-gaussian}}
\label{sec:proof-geometric-gaussian}

We focus on the Markovian setting. The i.i.d.\ case follows directly from
Lemma~\ref{lem:GA-independent} and does not require the Poisson-equation
decomposition below. Throughout the proof, \(C<\infty\) denotes a constant
independent of \(\alpha\) and \(t\), whose value may change from line to line.

For convenience, write $A:=I_d-J$ so that \(Q_\alpha=I_d-\alpha A\). Let \(g\) be the solution to the Poisson equation
\(g-Pg=h\) appearing in Lemma~\ref{lemma:GA-MDS}, and define
\[
    D_t:=g(x_t)-Pg(x_{t-1}),
    \qquad
    R_t:=Pg(x_t).
\]
Then \(\{D_t\}_{t\in\mathbb Z}\) is a stationary martingale-difference
sequence and
\begin{equation}
\label{eq:poisson-decomposition-h}
    h(x_t)
    =
    D_t+R_{t-1}-R_t.
\end{equation}
Indeed,
\[
    \sum_{t=1}^n h(x_t)
    =
    \sum_{t=1}^n D_t+R_0-R_n.
\]
Since \(R_0-R_n\) is bounded in \(L^2\), it is negligible under the
\(n^{-1/2}\) scaling. Consequently, the long-run covariance of
\(\{h(x_t)\}\) is $\Sigma
    =
    \E[D_0D_0^\top].$ Substituting \eqref{eq:poisson-decomposition-h} into
\eqref{eq:geometric-sum-representation}, we obtain $B_\infty^{(\alpha)}
    =
    M_\alpha+E_\alpha,$ where
\[
    M_\alpha
    :=
    \sqrt{\alpha}
    \sum_{k=0}^{\infty}
    Q_\alpha^kD_{-k-1},\qquad E_\alpha
    :=
    \sqrt{\alpha}
    \sum_{k=0}^{\infty}
    Q_\alpha^k
    \bigl(R_{-k-2}-R_{-k-1}\bigr).
\]
A summation by parts gives
\begin{align*}
    E_\alpha
    =
    \sqrt{\alpha}
    \left[
        -R_{-1}
        +
        \sum_{k=0}^{\infty}
        Q_\alpha^k(I_d-Q_\alpha)R_{-k-2}
    \right]=
    \sqrt{\alpha}
    \left[
        -R_{-1}
        +
        \alpha
        \sum_{k=0}^{\infty}
        Q_\alpha^k A R_{-k-2}
    \right].
\end{align*}
Hence, by stationarity, Minkowski's inequality, and
\eqref{eq:Qalpha-geometric-decay},
\begin{align*}
    \|E_\alpha\|_{L^2}
    \leq
    \sqrt{\alpha}\,\|R_0\|_{L^2}
    \left(
        1+
        \alpha\|A\|_{\mathrm{op}}
        \sum_{k=0}^{\infty}
        \|Q_\alpha^k\|_{\mathrm{op}}
    \right)\leq
    C\sqrt{\alpha}.
\end{align*}
Using the coupling
\(B_\infty^{(\alpha)}=M_\alpha+E_\alpha\), we therefore have
\begin{equation}
\label{eq:Salpha-Malpha-W2}
    \W_2\left(
        \law(B_\infty^{(\alpha)}),
        \law(M_\alpha)
    \right)
    \leq
    C\sqrt{\alpha}.
\end{equation}
Let \(V\) be the unique positive semidefinite solution to
\begin{equation}
\label{eq:continuous-Lyapunov-V}
    AV+VA^\top=\Sigma,
\end{equation}
or equivalently, $V
    =
    \int_0^\infty
    e^{-sA}\Sigma e^{-sA^\top}\,\mathrm ds.$ The covariance of \(M_\alpha\) is
\[
    V_\alpha
    :=
    \operatorname{Cov}(M_\alpha)
    =
    \alpha
    \sum_{k=0}^{\infty}
    Q_\alpha^k\Sigma(Q_\alpha^k)^\top.
\]
Let $\mathsf H:=\operatorname{Range}(V)$ and $r:=\dim(\mathsf H).$ For any \(u\in\ker(V)\),
\[
    0
    =
    u^\top Vu
    =
    \int_0^\infty
    \left\|
        \Sigma^{1/2}e^{-sA^\top}u
    \right\|^2
    \,\mathrm ds.
\]
Then, we have
\[
    \Sigma^{1/2}e^{-sA^\top}u=0,
    \qquad s\geq0.
\]
Setting \(s=0\) gives \(\Sigma u=0\), and therefore
\[
    \operatorname{Range}(\Sigma)\subseteq\operatorname{Range}(V)=\mathsf H.
\]
Moreover, differentiating the preceding identity in \(s\) shows that
\(A^\top u\in\ker(V)\). Thus \(\ker(V)\) is invariant under \(A^\top\),
or equivalently, \(\mathsf H\) is invariant under \(A\), and hence also
under \(Q_\alpha=I_d-\alpha A\). Since
\[
    \E[(u^\top D_t)^2]
    =
    u^\top\Sigma u
    =
    0,
    \qquad
    u\in\ker(\Sigma),
\]
we have
\[
    D_t\in\operatorname{Range}(\Sigma)\subseteq\mathsf H
    \qquad\text{a.s.}
\]
Consequently,
\[
    M_\alpha\in\mathsf H
    \qquad\text{a.s.}
\]
If \(r=0\), then \(\Sigma=0\), so \(D_t=0\) almost surely and hence
\(M_\alpha=0\). Since \(V=0\), \eqref{eq:Salpha-Malpha-W2} immediately
yields
\[
    \W_2\left(
        \law(B_\infty^{(\alpha)}),
        \mathcal N(0,V)
    \right)
    \leq
    C\sqrt{\alpha}.
\]
We therefore assume \(r\geq1\). Let \(U\in\R^{d\times r}\) have orthonormal columns spanning
\(\mathsf H\), and define
\[
    \overline V:=U^\top VU,
    \qquad
    \overline V_\alpha:=U^\top V_\alpha U.
\]
By construction, \(\overline V\succ0\). The covariance \(V_\alpha\) satisfies $V_\alpha-Q_\alpha V_\alpha Q_\alpha^\top
    =
    \alpha\Sigma.$ On the other hand, \eqref{eq:continuous-Lyapunov-V} gives
\begin{align*}
    V-Q_\alpha VQ_\alpha^\top
    =
    V-(I_d-\alpha A)V(I_d-\alpha A)^\top=
    \alpha\Sigma-\alpha^2AVA^\top.
\end{align*}
Subtracting the two identities and iterating yields
\begin{equation}
\label{eq:Valpha-minus-V}
    V_\alpha-V
    =
    \alpha^2
    \sum_{k=0}^{\infty}
    Q_\alpha^kAVA^\top(Q_\alpha^k)^\top
    \succeq0.
\end{equation}
In particular, $\overline V_\alpha
    \succeq
    \overline V
    \succ0,$ so that $\|\overline V_\alpha^{-1}\|_{\mathrm{op}}
    \leq
    \lambda_{\min}(\overline V)^{-1}.$ Moreover,
\[
    \|\overline V_\alpha\|_{\mathrm{op}}
    \leq
    \|V_\alpha\|_{\mathrm{op}}
    \leq
    \alpha\|\Sigma\|_{\mathrm{op}}
    \sum_{k=0}^{\infty}
    \|Q_\alpha^k\|_{\mathrm{op}}^2
    \leq C.
\]
Therefore,
\begin{equation}
\label{eq:Valpha-uniform-bounds}
    \sup_{\alpha}
    \left(
        \|\overline V_\alpha\|_{\mathrm{op}}
        +
        \|\overline V_\alpha^{-1}\|_{\mathrm{op}}
    \right)
    <\infty,
\end{equation}
For \(n\geq1\), define $M_{\alpha,n}
    :=
    \sqrt{\alpha}
    \sum_{k=0}^{n-1}
    Q_\alpha^kD_{-k-1},$ with covariance
\[
    V_{\alpha,n}
    :=
    \alpha
    \sum_{k=0}^{n-1}
    Q_\alpha^k\Sigma(Q_\alpha^k)^\top,
    \qquad
    \overline V_{\alpha,n}
    :=
    U^\top V_{\alpha,n}U.
\]
By \eqref{eq:Qalpha-geometric-decay},
\[
    \|V_\alpha-V_{\alpha,n}\|_{\mathrm{op}}
    \leq
    C\alpha
    \sum_{k=n}^{\infty}e^{-2c\alpha k}
    \leq
    Ce^{-2c\alpha n}.
\]
Hence, there exists \(C_0<\infty\) such that, for all sufficiently small
\(\alpha\) and all \(n\geq C_0/\alpha\),
\begin{equation}
\label{eq:finite-covariance-conditioning}
    \overline V_{\alpha,n}
    \succeq
    \frac12\overline V.
\end{equation}
Since $\overline V_{\alpha,n}
    \preceq
    \overline V_\alpha$ and \eqref{eq:Valpha-uniform-bounds} holds, it follows that
\[
    \|\overline V_{\alpha,n}^{1/2}\|_{\mathrm{op}}
    +
    \|\overline V_{\alpha,n}^{-1/2}\|_{\mathrm{op}}
    \leq C,
\]
where \(C\) is independent of \(\alpha\) and \(n\) whenever
\(n\geq C_0/\alpha\). Let
\[
    \overline D_i:=U^\top D_i,
    \qquad
    \overline Q_\alpha:=U^\top Q_\alpha U.
\]
Since \(\mathsf H\) is invariant under \(Q_\alpha\), stationarity gives
\[
    U^\top M_{\alpha,n}
    \overset{\mathrm d}{=}
    \sqrt{\alpha}
    \sum_{i=1}^n
    \overline Q_\alpha^{\,n-i}\overline D_i.
\]
Moreover, \(\{\overline D_i\}\) is precisely the martingale-difference
sequence associated with the projected Poisson equation
\[
    U^\top g-P(U^\top g)=U^\top h.
\]
Fix \(n\geq C_0/\alpha\), and define
\[
    Y_i
    :=
    \overline V_{\alpha,n}^{-1/2}
    \sqrt{\alpha}\,
    \overline Q_\alpha^{\,n-i}\overline D_i,
    \qquad
    i=1,\ldots,n.
\]
Then $\operatorname{Cov}\left(\sum_{i=1}^nY_i\right)
    =
    I_r.$ Using Lemma~\ref{lemma:GA-MDS}, \eqref{eq:Qalpha-geometric-decay} and
\eqref{eq:finite-covariance-conditioning}, we have
\begin{align*}
    \W_2^2\left(
        \law\left(
            \overline V_{\alpha,n}^{-1/2}
            U^\top M_{\alpha,n}
        \right),
        \mathcal N(0,I_r)
    \right)
    \leq C\alpha.
\end{align*}
Applying the linear map
\(z\mapsto U\overline V_{\alpha,n}^{1/2}z\) gives
\begin{equation}
\label{eq:finite-martingale-GA}
    \W_2\left(
        \law(M_{\alpha,n}),
        \mathcal N(0,V_{\alpha,n})
    \right)
    \leq
    C\sqrt{\alpha}.
\end{equation}
As \(n\to\infty\), $M_{\alpha,n}$ converges to $M_\alpha$ in $L^2$ and $V_{\alpha,n}\to  V_\alpha.$ Hence
\[
    \mathcal N(0,V_{\alpha,n})
    \longrightarrow
    \mathcal N(0,V_\alpha)
    \qquad\text{in }\W_2.
\]
Letting \(n\to\infty\) in \eqref{eq:finite-martingale-GA} yields
\begin{equation}
\label{eq:infinite-martingale-GA}
    \W_2\left(
        \law(M_\alpha),
        \mathcal N(0,V_\alpha)
    \right)
    \leq
    C\sqrt{\alpha}.
\end{equation}
By \eqref{eq:Valpha-minus-V} and
\eqref{eq:Qalpha-geometric-decay} gives
\begin{align}
    \W_2^2\left(
        \mathcal N(0,V_\alpha),
        \mathcal N(0,V)
    \right)
    &\leq
    C\alpha^2
    \sum_{k=0}^{\infty}e^{-2c\alpha k}
    \leq
    C\alpha.
\label{eq:Gaussian-covariance-W2}
\end{align}
Combining
\eqref{eq:Salpha-Malpha-W2},
\eqref{eq:infinite-martingale-GA}, and
\eqref{eq:Gaussian-covariance-W2} with the triangle inequality gives
\begin{align*}
    \W_2\left(
        \law(B_\infty^{(\alpha)}),
        \mathcal N(0,V)
    \right)\leq
    C\sqrt{\alpha}.
\end{align*}
This completes the proof of Proposition~\ref{prop:geometric-gaussian}.
\section{Proof of Proposition \ref{prop:uniform}}\label{sec:proof-uniform}
We prove a strengthened version of the proposition that also specifies
the coupling's measurability and conditional independence properties at
block boundaries. For each fixed \(n\in\mathbb N\), the coupling can be
realized sequentially using mutually independent auxiliary random
elements \(\{\zeta_m\}_{m\geq0}\), with the entire family independent of
the original Markov chain. Define
\(\mathcal F_0^{(n)}:=\{\varnothing,\Omega\}\) and
\begin{equation}
\label{eq:block-boundary-filtration}
    \mathcal F_m^{(n)}
    :=
    \sigma\bigl(
        x_0,\ldots,x_{mn-1},
        \zeta_0,\ldots,\zeta_{m-1}
    \bigr),
    \qquad m\geq1.
\end{equation}
We construct the coupling so that each Gaussian block
\[
    \mathbf w_m
    :=
    \bigl(
        w_{mn}^{\top},\ldots,w_{(m+1)n-1}^{\top}
    \bigr)^{\top}
\]
is \(\mathcal F_{m+1}^{(n)}\)-measurable and satisfies, almost surely,
\begin{equation}
\label{eq:block-boundary-Gaussian}
    \mathcal L\bigl(
        \mathbf w_m\mid\mathcal F_m^{(n)}
    \bigr)
    =
    \mathcal N(0,I_n\otimes\Sigma_h),
    \qquad m\geq0.
\end{equation}
Moreover, independence of the auxiliary randomization from the original
chain ensures that, for every bounded measurable
\(f:\mathcal X\to\mathbb R\),
\begin{equation}
\label{eq:block-boundary-Markov}
    \E\bigl[
        f(x_{mn+r})\mid\mathcal F_m^{(n)}
    \bigr]
    =
    P^{r+1}f(x_{mn-1})
    \quad\text{a.s.},
    \qquad m\geq1,\quad r\geq0.
\end{equation}
Fix \(n\in\mathbb N\). Throughout this proof, constants are independent of
\(n\) and the block index \(m\). For \(x\in\mathcal X\), let
\(K_n(x,\cdot)\) denote the law of
\[
    S^{(n)}(x):=\frac1{\sqrt n}\sum_{j=0}^{n-1}h(X_j^x),
\]
where \(\{X_j^x\}_{j\geq0}\) is a copy of the Markov chain started from
\(X_0^x=x\). For \(m\geq1\), define
\(\mathcal H_m:=\sigma(x_0,\ldots,x_{mn-1})\). By the Markov property,
\[
    \mathcal L\bigl(S_m^{(n)}\mid\mathcal H_m\bigr)
    =\int K_n(x,\cdot)\,P(x_{mn-1},\mathrm dx).
\]
Since \(x_0\sim\mu\), all block sums have the same marginal law
\[
    \nu_n:=\mathcal L(S_0^{(n)})
    =\mathcal L(S_m^{(n)})
    =\int K_n(x,\cdot)\,\mu(\mathrm dx).
\]
We use the following mixture inequality with \(p=2\); its general
Wasserstein--\(p\) formulation is recorded for completeness, and its proof
is deferred to Section~\ref{sec:proof-mixture-Wp}.

\begin{lemma}\label{lem:mixture-Wp}
Let $(\mathsf S,d)$ be a Polish metric space, and let
$\mathcal P(\mathsf S)$ denote the set of Borel probability measures on
$\mathsf S$. Let $(\mathsf E,\mathcal E)$ and $(\mathsf F,\mathcal F)$ be
standard Borel spaces, let $\alpha\in\mathcal P(\mathsf E)$ and
$\beta\in\mathcal P(\mathsf F)$, and let
\[
x\mapsto \mu_x\in\mathcal P(\mathsf S),
\qquad
y\mapsto \nu_y\in\mathcal P(\mathsf S)
\]
be Borel probability kernels. Define the corresponding mixtures
\[
\bar\mu \;:=\; \int_{\mathsf E}\mu_x\,\alpha(\mathrm dx),
\qquad
\bar\nu \;:=\; \int_{\mathsf F}\nu_y\,\beta(\mathrm dy)
\;\in\; \mathcal P(\mathsf S).
\]
Then, for any $p\in[1,\infty)$,
\begin{equation}\label{eq:mixture-Wp}
\mathcal W_p(\bar\mu,\bar\nu)^p
\;\le\;
\inf_{\pi\in\Pi(\alpha,\beta)}
\int_{\mathsf E\times\mathsf F}
\mathcal W_p(\mu_x,\nu_y)^p\,\pi(\mathrm dx,\mathrm dy),
\end{equation}
where $\Pi(\alpha,\beta)$ denotes the set of couplings of $\alpha$ and $\beta$.
\end{lemma}
Applying Lemma~\ref{lem:mixture-Wp} with the independent coupling of
\(P(x_{mn-1},\cdot)\) and \(\mu\) gives
\begin{equation}\label{eq:cond-Wp-final}
    \W_2^2\bigl(\mathcal L(S_m^{(n)}\mid\mathcal H_m),\nu_n\bigr)
    \leq
    \int\W_2^2\bigl(K_n(x,\cdot),K_n(y,\cdot)\bigr)
        P(x_{mn-1},\mathrm dx)\,\mu(\mathrm dy).
\end{equation}
To bound the integrand, we use the following coupling result for uniformly
ergodic Markov chains.

\begin{lemma}[Lemma~7 in \cite{zhang2026wasserstein} with $V\equiv1$]
\label{lem:forward-meeting}
Under \eqref{eq:uniform-GE}, there exist constants
$C<\infty$ and $\rho_{\mathrm{meet}}\in(0,1)$ such that, for every
$x,y\in\mathcal X$, one can construct a coupling
$(X_t^x,X_t^y)_{t\geq0}$ of two copies of the Markov chain initialized at
$X_0^x=x$ and $X_0^y=y$ such that the following holds: the meeting time $T^+(x,y):=\inf\{t\ge0: X_t^x=X_t^y\}$ satisfies%
\[
  \P\bigl(T^+(x,y)>k\bigr)
  \lesssim
C\rho_{\mathrm{meet}}^k,
  \qquad k\ge0.
\]
Here, $C<\infty$ and $\rho_{\mathrm{meet}} \in (0,1)$ are independent of $x$, $y$, and $k$.
Consequently, for every $p\geq1$, there exists a finite constant
$C_p<\infty$, independent of $x$ and $y$, such that $\E\!\left[T^+(x,y)^p\right]
    \leq C_p.$
\end{lemma}
Couple the chains started from \(x\) and \(y\) as in
Lemma~\ref{lem:forward-meeting}, keeping them equal after their meeting
time \(T^+(x,y)\). Boundedness of \(h\) yields
\begin{align}
    \W_2^2\bigl(K_n(x,\cdot),K_n(y,\cdot)\bigr)
    &\leq\E\bigl[\|S^{(n)}(x)-S^{(n)}(y)\|^2\bigr]\nonumber\\
    &\leq\frac{4\sup_{z\in\mathcal X}\|h(z)\|^2}{n}
        \E\bigl[(T^+(x,y)\wedge n)^2\bigr]
    \leq\frac{C}{n}.
    \label{eq:cond-Wp-final1}
\end{align}
Together with \eqref{eq:cond-Wp-final}, this gives an almost-sure bound
\[
    \W_2^2\bigl(\mathcal L(S_m^{(n)}\mid\mathcal H_m),\nu_n\bigr)
    \leq C/n,
    \qquad m\geq1.
\]

\noindent\textbf{Sequential decoupling and Gaussian coupling.}
Enlarge the probability space by adjoining mutually independent primitive
randomizers, all independent of the entire original chain. For each block
\(m\), let \(U_m,V_m\) be independent uniform random variables on
\((0,1)\), and let
\(\eta_{m,2},\ldots,\eta_{m,n}\) be independent
\(\mathcal N(0,\Sigma_h)\) random vectors. Put
\[
    \zeta_m:=(U_m,V_m,\eta_{m,2},\ldots,\eta_{m,n})
\]
and use the block-boundary sigma-fields
\(\mathcal F_m^{(n)}\) defined in
\eqref{eq:block-boundary-filtration}. For \(n=1\), the Gaussian bridge
variables are absent. Independence of the primitive randomizers from the
chain implies
\begin{equation}
\label{eq:block-conditional-law}
    \mathcal L(S_m^{(n)}\mid\mathcal F_m^{(n)})
    =\mathcal L(S_m^{(n)}\mid\mathcal H_m),
    \qquad m\geq1.
\end{equation}

Write \(S_m^\sharp\) for the decoupled block
\(S_m'^{(n)}\) introduced in Section~\ref{sec:ind-Gaussian-universality}.
Take \(S_0^\sharp=S_0^{(n)}\). For \(m\geq1\), apply
Lemma~\ref{lem:wp-decoupling} with \(p=2\),
\(Y=S_m^{(n)}\), and \(\mathcal G=\mathcal F_m^{(n)}\).
The measurable optimal-coupling and disintegration construction in its proof
can be realized with the fresh randomizer \(U_m\). In particular, choose
\(S_m^\sharp\) as a measurable function of
\(\mathcal L(S_m^{(n)}\mid\mathcal H_m)\), \(S_m^{(n)}\), and \(U_m\),
so that
\[
    \mathcal L(S_m^\sharp\mid\mathcal F_m^{(n)})=\nu_n,
    \qquad
    \E\bigl[\|S_m^{(n)}-S_m^\sharp\|^2\bigr]
    =\E\!\left[
        \W_2^2\bigl(\mathcal L(S_m^{(n)}\mid\mathcal F_m^{(n)}),\nu_n\bigr)
      \right].
\]
Thus \(S_m^\sharp\) is \(\mathcal F_{m+1}^{(n)}\)-measurable and independent
of \(\mathcal F_m^{(n)}\). Equations~\eqref{eq:cond-Wp-final}--
\eqref{eq:block-conditional-law} give
\begin{equation}\label{eq:block-decoupled}
    \sup_{m\geq0}
    \left(\E\bigl[\|S_m^{(n)}-S_m^\sharp\|^2\bigr]\right)^{1/2}
    \leq C/\sqrt n.
\end{equation}
In particular, the decoupled blocks are i.i.d.\ with law \(\nu_n\).

Fix an optimal coupling between \(\nu_n\) and
\(\mathcal N(0,\Sigma_h)\). Disintegrate it with respect to its first
marginal and use \(V_m\) to generate \(G_m\) conditionally on
\(S_m^\sharp\) through this fixed kernel. Then
\[
    \mathcal L(G_m\mid\mathcal F_m^{(n)})=\mathcal N(0,\Sigma_h),
    \qquad
    \E\bigl[\|S_m^\sharp-G_m\|^2\bigr]
    =\W_2^2\bigl(\nu_n,\mathcal N(0,\Sigma_h)\bigr).
\]
By stationarity and Lemma~\ref{lem:GA-MC},
\begin{equation}\label{eq:block-decoupled1}
    \sup_{m\geq0}
    \left(\E\bigl[\|S_m^\sharp-G_m\|^2\bigr]\right)^{1/2}
    \leq C/\sqrt n.
\end{equation}
Combining \eqref{eq:block-decoupled} and \eqref{eq:block-decoupled1} gives
\[
    \sup_{m\geq0}
    \left(\E\bigl[\|S_m^{(n)}-G_m\|^2\bigr]\right)^{1/2}
    \leq C/\sqrt n.
\]

\noindent\textbf{Gaussian bridges and block-boundary information.}
By construction, \(G_m\) depends only on the original chain through the
end of block \(m\), \(U_m\), and \(V_m\). In particular, the entire
family of Gaussian bridge variables is independent of \(\{G_m\}_{m\geq0}\).
Define
\[
    \mathbf u_m:=(G_m^\top,\eta_{m,2}^\top,\ldots,
                  \eta_{m,n}^\top)^\top.
\]
Conditionally on \(\mathcal F_m^{(n)}\), this vector has law
\(\mathcal N(0,I_n\otimes\Sigma_h)\). Choose an orthogonal matrix
\(H\in\mathbb R^{n\times n}\) whose first row is
\(n^{-1/2}(1,\ldots,1)\), and set
\begin{equation}\label{eq:w-def}
    \mathbf w_m
    :=(w_{mn}^\top,\ldots,w_{(m+1)n-1}^\top)^\top
    =(H^\top\otimes I_d)\mathbf u_m.
\end{equation}
Orthogonality implies
\[
    \mathcal L(\mathbf w_m\mid\mathcal F_m^{(n)})
    =\mathcal N(0,I_n\otimes\Sigma_h),
    \qquad
    \frac1{\sqrt n}\sum_{k=mn}^{(m+1)n-1}w_k=G_m.
\]
The block \(\mathbf w_m\) is \(\mathcal F_{m+1}^{(n)}\)-measurable. Since
all preceding Gaussian blocks are \(\mathcal F_m^{(n)}\)-measurable,
its deterministic conditional law proves that the blocks are mutually
independent; the coordinates within each block are also independent.
Thus \(\{w_k\}_{k\geq0}\) is i.i.d.\ with common law
\(\mathcal N(0,\Sigma_h)\), and
\(Z_m^{(n)}=G_m\). This proves \eqref{eq:uniform-block-W2} and
\eqref{eq:block-boundary-Gaussian}.

Finally, because all primitive randomizers are independent of the original
chain, the Markov property gives, for every bounded measurable \(f\),
\[
    \E[f(x_{mn+r})\mid\mathcal F_m^{(n)}]
    =\E[f(x_{mn+r})\mid\mathcal H_m]
    =P^{r+1}f(x_{mn-1}),
    \qquad m\geq1,\quad r\geq0.
\]
This proves \eqref{eq:block-boundary-Markov} and completes the proof of
Proposition~\ref{prop:uniform}. No adaptedness of the individual Gaussian
noises to the original within-block Markov filtration is asserted or needed.

\subsection{Proof of Lemma \ref{lem:wp-decoupling}}\label{sec:proof-wp-decoupling}
Let
\[
    \mu_\omega
    :=
    \mathcal L(Y\mid\mathcal G)(\omega),
    \qquad
    \mu:=\mathcal L(Y).
\]
Since $\mathcal M$ is Polish, a regular conditional distribution
$\mu_\omega$ exists. Moreover, since $\mu\in\mathcal P_p(\mathcal M)$,
for any fixed $x_0\in\mathcal M$,
\[
    \E\left[
        \int_{\mathcal M}d(x,x_0)^p\,\mu_\omega(\mathrm dx)
    \right]
    =
    \E[d(Y,x_0)^p]
    <
    \infty.
\]
Hence $\mu_\omega\in\mathcal P_p(\mathcal M)$ almost surely. After
modifying $\mu_\omega$ on a $\mathcal G$-null set if necessary, we may
regard $\omega\mapsto\mu_\omega$ as a $\mathcal G$-measurable map from $\Omega$ into
$\mathcal P_p(\mathcal M)$.

We next choose the optimal couplings measurably. Since
$\mathcal P_p(\mathcal M)$ is Polish, we may apply the measurable-selection
result for the parametric Kantorovich problem in
\cite[Theorem~4.2]{bogachev2020kantorovich}, with parameter
$\nu\in\mathcal P_p(\mathcal M)$, marginals $(\nu,\mu)$, and cost $c(x,y):=d(x,y)^p.$ It follows that there exists a Borel measurable map $\nu
    \longmapsto
    \Psi(\nu)\in\mathcal P(\mathcal M\times\mathcal M)$ such that $\Psi(\nu)\in\Pi(\nu,\mu)$ and
\[
    \int_{\mathcal M\times\mathcal M}
    d(x,y)^p\,\Psi(\nu)(\mathrm dx,\mathrm dy)
    =
    \W_{p,d}^p(\nu,\mu).
\]
Define $\pi_\omega:=\Psi(\mu_\omega).$ Then $\omega\mapsto\pi_\omega$ is $\mathcal G$-measurable,
$\pi_\omega\in\Pi(\mu_\omega,\mu)$, and
\begin{equation}\label{eq:wp-decoupling-optimal}
    \int_{\mathcal M\times\mathcal M}
    d(x,y)^p\,\pi_\omega(\mathrm dx,\mathrm dy)
    =
    \W_{p,d}^p(\mu_\omega,\mu).
\end{equation}
By measurable disintegration for standard Borel spaces
(cf.\ \cite[Corollary~4.4]{bogachev2020kantorovich}), there exists a
$(\mathcal G\otimes\mathcal B(\mathcal M))$-measurable Markov kernel
$K$ such that
\begin{equation}\label{eq:wp-decoupling-disintegration}
    \pi_\omega(\mathrm dx,\mathrm dy)
    =
    \mu_\omega(\mathrm dx)\,
    K(\omega,x,\mathrm dy).
\end{equation}
We now use $K$ to construct the desired extension. Let $\widetilde\Omega:=\Omega\times\mathcal M,$ $\widetilde{\mathcal F}
    :=
    \mathcal F\otimes\mathcal B(\mathcal M),$ and define a probability measure $\widetilde{\mathbb P}$ by
\[
    \widetilde{\mathbb P}(\mathrm d\omega,\mathrm dy)
    :=
    \mathbb P(\mathrm d\omega)\,
    K(\omega,Y(\omega),\mathrm dy).
\]
Then
\[
    \widetilde{\mathbb P}(A\times\mathcal M)
    =
    \mathbb P(A),
    \qquad A\in\mathcal F,
\]
so this indeed defines an extension of the original probability space.
We identify the original random variables and $\sigma$-fields with their
canonical lifts to $\widetilde\Omega$, and define $ Y^*(\omega,y):=y.$ We claim that, conditionally on $\mathcal G$,
\begin{equation}\label{eq:conditional-optimal-coupling}
    \mathcal L\bigl((Y,Y^*)\mid\mathcal G\bigr)(\omega)
    =
    \pi_\omega.
\end{equation}
Indeed, for any bounded Borel function
$\varphi:\mathcal M\times\mathcal M\to\mathbb R$,
the definition of $\widetilde{\mathbb P}$ and the tower property give
\begin{align*}
    \widetilde{\E}\!\left[
        \varphi(Y,Y^*)
        \mid\mathcal G
    \right](\omega)
    =
    \int_{\mathcal M}
    \int_{\mathcal M}
    \varphi(x,y)\,
    K(\omega,x,\mathrm dy)\,
    \mu_\omega(\mathrm dx) =
    \int_{\mathcal M\times\mathcal M}
    \varphi(x,y)\,
    \pi_\omega(\mathrm dx,\mathrm dy),
\end{align*}
where the last equality follows from
\eqref{eq:wp-decoupling-disintegration}. This proves
\eqref{eq:conditional-optimal-coupling}. Since the second marginal of $\pi_\omega$ is $\mu$, we immediately obtain
\[
    \mathcal L(Y^*\mid\mathcal G)
    =
    \mu
    \qquad\text{a.s.}
\]
As the right-hand side is deterministic, $Y^*$ is independent of
$\mathcal G$. Moreover, $\mathcal L(Y^*)
    =
    \mu
    =
    \mathcal L(Y).$ Finally, applying \eqref{eq:conditional-optimal-coupling} to the
nonnegative cost $(x,y)\mapsto d(x,y)^p$ and using
\eqref{eq:wp-decoupling-optimal}, we obtain
\[
    \widetilde{\E}\!\left[
        d(Y,Y^*)^p
        \mid\mathcal G
    \right](\omega)
    =
    \int_{\mathcal M\times\mathcal M}
    d(x,y)^p\,
    \pi_\omega(\mathrm dx,\mathrm dy) =
    \W_{p,d}^p(\mu_\omega,\mu).
\]
Taking expectations yields
\[
    \widetilde{\E}\!\left[d(Y,Y^*)^p\right]
    =
    \E\left[
        \W_{p,d}^p\bigl(
            \mathcal L(Y\mid\mathcal G),
            \mathcal L(Y)
        \bigr)
    \right],
\]
which completes the proof of Lemma \ref{lem:wp-decoupling}.
\subsection{Proof of Lemma \ref{lem:mixture-Wp}}\label{sec:proof-mixture-Wp}
Fix any $\pi\in\Pi(\alpha,\beta)$.  For each $(x,y)\in\mathsf E\times\mathsf F$, let
$\eta_{x,y}\in\Pi(\mu_x,\nu_y)$ be an optimal coupling satisfying
\[
    \int_{\mathsf S\times\mathsf S}
    d(z,z')^p\,
    \eta_{x,y}(\mathrm dz,\mathrm dz')
    =
    \mathcal W_p^p(\mu_x,\nu_y).
\]
Now define a probability measure $\Lambda$ on
$\mathsf S\times\mathsf S$ by
\[
    \Lambda(B)
    :=
    \int_{\mathsf E\times\mathsf F}
    \eta_{x,y}(B)\,
    \pi(\mathrm dx,\mathrm dy),
    \qquad
    B\in\mathcal B(\mathsf S\times\mathsf S).
\]
We first verify that $\Lambda$ couples $\bar\mu$ and $\bar\nu$.
For any bounded Borel function $\varphi:\mathsf S\to\R$, Fubini's theorem
gives
\begin{align*}
    \int_{\mathsf S\times\mathsf S}
    \varphi(z)\,
    \Lambda(\mathrm dz,\mathrm dz')
    &=
    \int_{\mathsf E\times\mathsf F}
    \left(
        \int_{\mathsf S}
        \varphi(z)\,\mu_x(\mathrm dz)
    \right)
    \pi(\mathrm dx,\mathrm dy) \\
    &=
    \int_{\mathsf E}
    \left(
        \int_{\mathsf S}
        \varphi(z)\,\mu_x(\mathrm dz)
    \right)
    \alpha(\mathrm dx) =
    \int_{\mathsf S}\varphi(z)\,\bar\mu(\mathrm dz),
\end{align*}
where the second equality uses that the first marginal of $\pi$ is
$\alpha$. Hence the first marginal of $\Lambda$ is $\bar\mu$.
Similarly, since the second marginal of $\pi$ is $\beta$, the second
marginal of $\Lambda$ is $\bar\nu$. Therefore, $ \Lambda\in\Pi(\bar\mu,\bar\nu).$ The transportation cost of this coupling is
\begin{align*}
    \int_{\mathsf S\times\mathsf S}
    d(z,z')^p\,
    \Lambda(\mathrm dz,\mathrm dz')
    &=
    \int_{\mathsf E\times\mathsf F}
    \left[
        \int_{\mathsf S\times\mathsf S}
        d(z,z')^p\,
        \eta_{x,y}(\mathrm dz,\mathrm dz')
    \right]
    \pi(\mathrm dx,\mathrm dy) \\
    &=
    \int_{\mathsf E\times\mathsf F}
    \mathcal W_p^p(\mu_x,\nu_y)\,
    \pi(\mathrm dx,\mathrm dy).
\end{align*}
Since $\mathcal W_p^p(\bar\mu,\bar\nu)$ is the minimum transportation cost
over all couplings of $\bar\mu$ and $\bar\nu$, we conclude that
\[
    \mathcal W_p^p(\bar\mu,\bar\nu)
    \leq
    \int_{\mathsf E\times\mathsf F}
    \mathcal W_p^p(\mu_x,\nu_y)\,
    \pi(\mathrm dx,\mathrm dy).
\]
Because $\pi\in\Pi(\alpha,\beta)$ was arbitrary, taking the infimum over
$\pi$ yields \eqref{eq:mixture-Wp}.
\section{Proof of Proposition \ref{prop:ind-Gaussian}}\label{sec:proof-universal-noise}
We study two diffusion-scaled processes,
\(\{A_t^{(\alpha)}\}_{t\ge 0}\) and
\(\{Z_t^{(\alpha)}\}_{t\ge 0}\), defined by
\[
A_t^{(\alpha)}
:=
\frac{a_t^{(\alpha)}-\theta^*}{\sqrt{\alpha}},
\qquad
Z_t^{(\alpha)}
:=
\frac{\beta_t^{(\alpha)}-\theta^*}{\sqrt{\alpha}}.
\]
When the dependence on \(\alpha\) is clear, we suppress the
superscript \((\alpha)\) for notational simplicity. From \eqref{eq:auxiliary-additive} and \eqref{eq:auxiliary-Gaussian}, the dynamics of the scaled processes
can be written as
\begin{equation}\label{eq: additive different noise}
\begin{aligned}
A_{t+1} &= (1-\alpha)A_t + \sqrt{\alpha}\big(\mT(\sqrt{\alpha} A_t + \theta^*) - \mT(\theta^*) + h(x_t)\big),\\
Z_{t+1}  &= (1-\alpha)Z_t  + \sqrt{\alpha}\big(\mT(\sqrt{\alpha} Z_t  + \theta^*) - \mT(\theta^*) + w_t\big),
\end{aligned}
\end{equation}
We take \(x_0\sim\mu\), set \(A_0=Z_0=0\), and choose
\(n=\lfloor\alpha^{-1/2}\rfloor\). In the Markovian setting, use the
coupling of Proposition~\ref{prop:uniform}, including the sequential
realization in \eqref{eq:block-boundary-filtration}--\eqref{eq:block-boundary-Markov}. In the i.i.d.\ setting, couple each
original block independently to a Gaussian block sum and then use the
same Gaussian bridge construction. Lemma~\ref{lem:GA-independent},
applied on \(\operatorname{Range}(\Sigma_h)\) when \(\Sigma_h\) is
singular, gives the same bound \eqref{eq:uniform-block-W2}. The components
of the centered i.i.d.\ noise on \(\ker(\Sigma_h)\) vanish almost surely.
Thus, in both settings, the Gaussian block is independent of
\(\mathcal F_m^{(n)}\) and has the conditional law
\eqref{eq:block-boundary-Gaussian}.

Let \(\Delta_k:=A_k-Z_k\), and write
\(\|U\|_{L^2}:=(\E\|U\|^2)^{1/2}\) for a random vector \(U\).
By Lemma~\ref{lem:fourthmoment}, for every sufficiently small \(\alpha\)
there is a time \(k_\alpha\) such that
\begin{equation}
\label{eq:universal-moment-bound}
    \sup_{k\geq k_\alpha}
    \bigl(\|A_k\|_{L^2}+\|Z_k\|_{L^2}+\|\Delta_k\|_{L^2}\bigr)
    \leq C,
\end{equation}
where \(C\) is independent of \(\alpha\), \(n\), and \(k\).
All constants below may depend on the fixed problem parameters, but not
on \(\alpha\), \(n\), or the time index. The estimate \eqref{eq:uniform-block-W2} implies
\begin{equation}
\label{eq:unscaled-block-error}
    \left\|\sum_{k=mn}^{(m+1)n-1}\bigl(h(x_k)-w_k\bigr)\right\|_{L^2}
    =\sqrt n\,\|S_m^{(n)}-Z_m^{(n)}\|_{L^2}
    \leq C.
\end{equation}
For \(s=mn\), subtracting the recursions in
\eqref{eq: additive different noise} and telescoping gives
\begin{align*}
    \Delta_{s+n}-\Delta_s
    &=-\alpha\sum_{k=s}^{s+n-1}\Delta_k+\sqrt\alpha\sum_{k=s}^{s+n-1}
        \bigl[\mT(\theta^*+\sqrt\alpha A_k)
              -\mT(\theta^*+\sqrt\alpha Z_k)\bigr]+\sqrt\alpha\sum_{k=s}^{s+n-1}\bigl(h(x_k)-w_k\bigr).
\end{align*}
By contraction, equivalence of norms, Minkowski's inequality, and
\eqref{eq:universal-moment-bound}--\eqref{eq:unscaled-block-error},
\begin{equation}
\label{eq:Delta-block-increment}
    \|\Delta_{s+n}-\Delta_s\|_{L^2}
    \leq C\alpha\sum_{k=s}^{s+n-1}\|\Delta_k\|_{L^2}+C\sqrt\alpha
    \leq C(\alpha n+\sqrt\alpha)
    \leq C\sqrt\alpha,
    \qquad s=mn\geq k_\alpha.
\end{equation}
Unrolling \eqref{eq: additive different noise} over \(n\) steps gives
\begin{equation}\label{eq:different-noise}
\begin{aligned}
&A_{n t+n}   
= (1-\alpha )^n A_{n t} + \sqrt{\alpha}{n}\big(\mT(\sqrt{\alpha} A_{{n}t}   + \theta^*) - \mT(\theta^*)\big)\\
&+\sqrt{\alpha} \sum_{j=1}^n \big(\mT(\sqrt{\alpha} A_{n t+n-j} + \theta^*)  - \mT(\sqrt{\alpha} A_{n t} + \theta^*)  \big)\\
&+\sqrt{\alpha} \sum_{j=1}^n \big((1-\alpha )^{j-1}-1\big)\big(\mT(\sqrt{\alpha} A_{n t+n-j} + \theta^*) - \mT( \theta^*) \big)
+\sqrt{\alpha} \sum_{j=1}^n (1-\alpha )^{j-1}h(x_{n t+n-j}),\\
&Z_{{n} t+{n}}   = (1-\alpha )^{n} Z_{{n} t}   + \sqrt{\alpha}{n}\big(\mT(\sqrt{\alpha} Z_{{n}t}   + \theta^*) - \mT(\theta^*)\big)\\
&+\sqrt{\alpha} \sum_{j=1}^{n} \big(\mT(\sqrt{\alpha} Z_{{n}t+{n}-j}  + \theta^*) - \mT(\sqrt{\alpha} Z_{{n}t}   + \theta^*) \big)\\
&+\sqrt{\alpha} \sum_{j=1}^{n} \big((1-\alpha )^{j-1}-1\big)\big(\mT(\sqrt{\alpha} Z_{{n}t+{n}-j}   + \theta^*) - \mT(\theta^*) \big) 
+\sqrt{\alpha} \sum_{j=1}^{n} (1-\alpha )^{j-1}w_{{n}t+{n}-j}.
\end{aligned}
\end{equation}
Subtracting the two recursions in \eqref{eq:different-noise} yields
\begin{equation}
\label{eq:universal-block-difference}
    \Delta_{nt+n}=(1-\alpha)^n\Delta_{nt}+R,
\end{equation}
where
\begin{align*}
R={}&\sqrt\alpha n
    \bigl[\mT(\theta^*+\sqrt\alpha A_{nt})
          -\mT(\theta^*+\sqrt\alpha Z_{nt})\bigr]\\
&+\sqrt\alpha\sum_{j=1}^n
    \bigl[\mT(\theta^*+\sqrt\alpha A_{nt+n-j})
          -\mT(\theta^*+\sqrt\alpha A_{nt})\bigr]\\
&-\sqrt\alpha\sum_{j=1}^n
    \bigl[\mT(\theta^*+\sqrt\alpha Z_{nt+n-j})
          -\mT(\theta^*+\sqrt\alpha Z_{nt})\bigr]\\
&+\sqrt\alpha\sum_{j=1}^n\bigl((1-\alpha)^{j-1}-1\bigr)
    \bigl[\mT(\theta^*+\sqrt\alpha A_{nt+n-j})-\mT(\theta^*)\bigr]\\
&-\sqrt\alpha\sum_{j=1}^n\bigl((1-\alpha)^{j-1}-1\bigr)
    \bigl[\mT(\theta^*+\sqrt\alpha Z_{nt+n-j})-\mT(\theta^*)\bigr]\\
&+\sqrt\alpha\sum_{j=1}^n(1-\alpha)^{j-1}
    \bigl(h(x_{nt+n-j})-w_{nt+n-j}\bigr).
\end{align*}
Choose \(\gamma<r<\sqrt\gamma\) and then fix \(\eta>0\) sufficiently
small that
\begin{equation}
\label{eq:universal-contraction-margin}
    \gamma\frac{u_{cm}}{l_{cm}}\leq r.
\end{equation}
This is possible because \(u_{cm}/l_{cm}\to1\) as \(\eta\downarrow0\).
Both \(r\) and \(\eta\) are fixed independently of \(\alpha\) and \(n\).
Let \(V_t:=\E[M_\eta(\Delta_{nt})]\), where \(M_\eta\) is the Moreau
envelope in \eqref{eq:Moreau-envelope}. Applying
Lemma~\ref{lem:Moreau-envelope} to
\eqref{eq:universal-block-difference} gives
\begin{equation}\label{eq:maininduction1}
    V_{t+1}
    \leq (1-\alpha)^{2n}V_t
    +(1-\alpha)^n
        \underbrace{\E\langle\nabla M_\eta(\Delta_{nt}),R\rangle}_{T_1}
    +\frac1{2\eta}\underbrace{\E\|R\|^2}_{T_2}.
\end{equation}
The following lemmas, proved in Sections~\ref{sec:problemT1} and
\ref{sec:problemT2}, bound these terms under the same coupling fixed above.

\begin{lemma}
\label{lem:SSCT1}
Under the coupling and initialization specified above, there exist
\(C<\infty\) and \(\alpha_0>0\) such that, for every
\(\alpha\in(0,\alpha_0)\), there is \(t_\alpha\) for which
\[
    T_1\leq2\sqrt\gamma\,\alpha n V_t+C\alpha,
    \qquad t\geq t_\alpha.
\]
\end{lemma}

\begin{lemma}
\label{lem:SSCT2}
Under the same coupling and initialization, there exist \(C<\infty\) and
\(\alpha_0>0\) such that, for every \(\alpha\in(0,\alpha_0)\), there is
\(t_\alpha\) for which
\[
    T_2\leq C\alpha,
    \qquad t\geq t_\alpha.
\]
\end{lemma}

Substituting the two bounds into \eqref{eq:maininduction1}, for all
sufficiently small \(\alpha\) and sufficiently large \(t\), gives
\begin{align*}
    V_{t+1}
    &\leq\bigl((1-\alpha)^{2n}
          +2\sqrt\gamma\,\alpha n(1-\alpha)^n\bigr)V_t+C\alpha\\
    &\leq\bigl(1-(1-\sqrt\gamma)\alpha n\bigr)V_t+C\alpha.
\end{align*}
Indeed, since \(\alpha n\to0\), the coefficient in the first line equals
\(1-2(1-\sqrt\gamma)\alpha n+\mathcal O((\alpha n)^2)\).
Iterating this inequality at fixed \(\alpha\) yields
\[
    \limsup_{t\to\infty}V_t
    \leq\frac{C\alpha}{(1-\sqrt\gamma)\alpha n}
    \leq\frac{C}{n}
    \leq C\sqrt\alpha.
\]
Finally, the triangle inequality and convergence of the marginal laws give
\begin{align*}
    \W_2\bigl(\law(A_\infty^{(\alpha)}),\law(Z_\infty^{(\alpha)})\bigr)
    &\leq\limsup_{t\to\infty}\Bigl[
        \W_2\bigl(\law(A_\infty^{(\alpha)}),\law(A_{nt})\bigr)+
        \W_2\bigl(\law(A_{nt}),\law(Z_{nt})\bigr)\\
    &\hspace{3em}
        +\W_2\bigl(\law(Z_{nt}),\law(Z_\infty^{(\alpha)})\bigr)
      \Bigr]\\
    &\leq C\sqrt{\limsup_{t\to\infty}V_t}
    \leq C\alpha^{1/4}.
\end{align*}
This completes the proof of Proposition~\ref{prop:ind-Gaussian}.

\subsection{Proof of Lemma \ref{lem:SSCT1}}\label{sec:problemT1}
The block-boundary construction ensures that \(\Delta_{nt}\) is
\(\mathcal F_t^{(n)}\)-measurable and that the entire Gaussian block
\((w_{nt},\ldots,w_{nt+n-1})\) is centered and independent of
\(\mathcal F_t^{(n)}\). Its contribution to
\(\E\langle\nabla M_\eta(\Delta_{nt}),R\rangle\) is therefore zero.
By property (4) in Lemma~\ref{lem:Moreau-envelope},
\(T_1\leq\sum_{i=1}^6T_{1i}\), where
\begin{align*}
T_{11}:={}&\sqrt\alpha n\,
    \E\!\left[\|\Delta_{nt}\|_m
        \bigl\|\mT(\theta^*+\sqrt\alpha A_{nt})
               -\mT(\theta^*+\sqrt\alpha Z_{nt})\bigr\|_m\right],\\
T_{12}:={}&\sqrt\alpha\,
    \E\!\left[\|\Delta_{nt}\|_m
        \left\|\sum_{j=1}^n
            \bigl[\mT(\theta^*+\sqrt\alpha A_{nt+n-j})
                  -\mT(\theta^*+\sqrt\alpha A_{nt})\bigr]
        \right\|_m\right],\\
T_{13}:={}&\sqrt\alpha\,
    \E\!\left[\|\Delta_{nt}\|_m
        \left\|\sum_{j=1}^n
            \bigl[\mT(\theta^*+\sqrt\alpha Z_{nt+n-j})
                  -\mT(\theta^*+\sqrt\alpha Z_{nt})\bigr]
        \right\|_m\right],\\
T_{14}:={}&\sqrt\alpha\,
    \E\!\left[\|\Delta_{nt}\|_m
        \left\|\sum_{j=1}^n\bigl((1-\alpha)^{j-1}-1\bigr)
            \bigl[\mT(\theta^*+\sqrt\alpha A_{nt+n-j})-\mT(\theta^*)\bigr]
        \right\|_m\right],\\
T_{15}:={}&\sqrt\alpha\,
    \E\!\left[\|\Delta_{nt}\|_m
        \left\|\sum_{j=1}^n\bigl((1-\alpha)^{j-1}-1\bigr)
            \bigl[\mT(\theta^*+\sqrt\alpha Z_{nt+n-j})-\mT(\theta^*)\bigr]
        \right\|_m\right],\\
T_{16}:={}&\sqrt\alpha\left|
    \E\left\langle\nabla M_\eta(\Delta_{nt}),
        \sum_{j=1}^n(1-\alpha)^{j-1}h(x_{nt+n-j})
    \right\rangle\right|.
\end{align*}
We bound the six terms separately, taking \(t\) sufficiently large that
\eqref{eq:universal-moment-bound} holds throughout the preceding block.

\noindent\textbf{The \(T_{11}\) term.}
By contraction, the norm equivalences in
Lemma~\ref{lem:Moreau-envelope}, and
\eqref{eq:universal-contraction-margin},
\begin{align*}
    T_{11}
    \leq\frac{\alpha n\gamma}{l_{cm}}
        \E\bigl[\|\Delta_{nt}\|_m\|\Delta_{nt}\|_c\bigr]\leq\frac{\alpha n\gamma u_{cm}}{l_{cm}}
        \E\|\Delta_{nt}\|_m^2
    =\frac{2\alpha n\gamma u_{cm}}{l_{cm}}V_t
    \leq2r\alpha nV_t.
\end{align*}

\noindent\textbf{The \(T_{12}\) and \(T_{13}\) terms.}
We use the following individual-increment estimate, whose proof is deferred
to Section~\ref{sec:difference-in-Y}.
\begin{lemma}\label{lem:difference-in-Y}
Under Assumptions~\ref{assumption:contraction} and \ref{assumption:noise},
for the auxiliary recursions and initialization specified above, there is
\(\alpha_0>0\) such that, for every \(\alpha\in(0,\alpha_0)\), there is
\(t_\alpha\) for which
\[
    \E\|A_{t+m}-A_t\|^2\lesssim\alpha m+\alpha^2m^2,
    \qquad
    \E\|Z_{t+m}-Z_t\|^2\lesssim\alpha m+\alpha^2m^2,
    \qquad m\geq1,\quad t\geq t_\alpha.
\]
\end{lemma}
Since \(\alpha n\leq1\) for sufficiently small \(\alpha\), contraction,
Cauchy--Schwarz, and Lemma~\ref{lem:difference-in-Y} give
\begin{align*}
    T_{12}
    \leq C\alpha\sum_{j=1}^n
        \|\Delta_{nt}\|_{L^2}\|A_{nt+n-j}-A_{nt}\|_{L^2}\leq C\alpha^{3/2}\sum_{j=1}^n\sqrt{n-j}\,\sqrt{V_t}
    &\leq C\alpha^{3/2}n^{3/2}\sqrt{V_t}\\
    &\leq(\sqrt\gamma-r)\alpha nV_t+C\alpha^2n^2,
\end{align*}
where the last step uses Young's inequality with the fixed positive
margin \(\sqrt\gamma-r\). The same argument gives
\[
    T_{13}\leq(\sqrt\gamma-r)\alpha nV_t+C\alpha^2n^2.
\]

\noindent\textbf{The \(T_{14}\) and \(T_{15}\) terms.}
By contraction and \eqref{eq:universal-moment-bound},
\begin{align*}
    T_{14}
    \leq C\alpha\sum_{j=1}^n
        \bigl(1-(1-\alpha)^{j-1}\bigr)
        \E\bigl[\|\Delta_{nt}\|\,\|A_{nt+n-j}\|\bigr]\leq C\alpha\sum_{j=1}^n\alpha(j-1)
    \leq C\alpha^2n^2.
\end{align*}
Similarly, \(T_{15}\leq C\alpha^2n^2\).

\noindent\textbf{The \(T_{16}\) term.}
Let
\[
    H_t:=\sum_{j=1}^n(1-\alpha)^{j-1}h(x_{nt+n-j}).
\]
In the i.i.d.\ setting, the original block is independent of
\(\mathcal F_t^{(n)}\) and centered, so
\(\E[H_t\mid\mathcal F_t^{(n)}]=0\) and \(T_{16}=0\).
We therefore consider the Markovian setting. Boundedness and centering of
\(h\), together with uniform ergodicity, imply
\[
    \sup_{x\in\mathcal X}\|P^\ell h(x)\|
    \leq C\rho_{\mathrm{mix}}^\ell,
    \qquad \ell\geq1.
\]
By \eqref{eq:block-boundary-Markov}, for \(t\geq2\),
\begin{align}
    \left\|\E[H_t\mid\mathcal F_t^{(n)}]\right\|
    &=\left\|\sum_{j=1}^n(1-\alpha)^{j-1}
        P^{n-j+1}h(x_{nt-1})\right\|
    \leq C,\nonumber\\
    \left\|\E[H_t\mid\mathcal F_{t-1}^{(n)}]\right\|
    &=\left\|\sum_{j=1}^n(1-\alpha)^{j-1}
        P^{2n-j+1}h(x_{n(t-1)-1})\right\|
    \leq C\rho_{\mathrm{mix}}^n.
    \label{eq:block-noise-conditional-means}
\end{align}
Both \(\Delta_{nt}\) and \(\Delta_{n(t-1)}\) are
\(\mathcal F_t^{(n)}\)-measurable, and
\(\Delta_{n(t-1)}\) is \(\mathcal F_{t-1}^{(n)}\)-measurable.
The tower property therefore gives the identity
\begin{align*}
    \E\langle\nabla M_\eta(\Delta_{nt}),H_t\rangle
    &=\E\left\langle\nabla M_\eta(\Delta_{n(t-1)}),
        \E[H_t\mid\mathcal F_{t-1}^{(n)}]\right\rangle\\
    &\quad +\E\left\langle
        \nabla M_\eta(\Delta_{nt})-\nabla M_\eta(\Delta_{n(t-1)}),
        \E[H_t\mid\mathcal F_t^{(n)}]\right\rangle.
\end{align*}
Since \(\nabla M_\eta(0)=0\) and \(\nabla M_\eta\) is
\(1/\eta\)-Lipschitz, equations~\eqref{eq:universal-moment-bound},
\eqref{eq:Delta-block-increment}, and
\eqref{eq:block-noise-conditional-means} imply
\begin{align*}
    T_{16}
    \leq C\sqrt\alpha\,\rho_{\mathrm{mix}}^n
        \E\|\Delta_{n(t-1)}\|
       +C\sqrt\alpha\,
        \E\|\Delta_{nt}-\Delta_{n(t-1)}\|\leq C\bigl(\sqrt\alpha\,\rho_{\mathrm{mix}}^n
                 +\alpha^{3/2}n+\alpha\bigr)
    \leq C\alpha.
\end{align*}
The last inequality uses \(n=\lfloor\alpha^{-1/2}\rfloor\) and holds
for all sufficiently small \(\alpha\). The full-block lag is admissible
because the coupled-difference increment in
\eqref{eq:Delta-block-increment} is \(\mathcal O(\sqrt\alpha)\).
No conditioning at an interior point of a coupled block is used. Combining the bounds for \(T_{11},\ldots,T_{16}\), and using
\(\alpha^2n^2\leq\alpha\), gives
\[
    T_1\leq
    \bigl(2r+2(\sqrt\gamma-r)\bigr)\alpha nV_t+C\alpha
    =2\sqrt\gamma\,\alpha nV_t+C\alpha.
\]
This proves Lemma~\ref{lem:SSCT1}.

\subsection{Proof of Lemma \ref{lem:SSCT2}}\label{sec:problemT2}
By \eqref{eq:universal-block-difference},
\[
    R=(\Delta_{nt+n}-\Delta_{nt})
      +\bigl(1-(1-\alpha)^n\bigr)\Delta_{nt}.
\]
Using \(1-(1-\alpha)^n\leq\alpha n\),
\eqref{eq:universal-moment-bound}, and
\eqref{eq:Delta-block-increment}, we obtain, for all sufficiently large \(t\),
\begin{align*}
    \|R\|_{L^2}
    \leq\|\Delta_{nt+n}-\Delta_{nt}\|_{L^2}
          +\alpha n\|\Delta_{nt}\|_{L^2}\leq C(\sqrt\alpha+\alpha n)
    \leq C\sqrt\alpha.
\end{align*}
Consequently, \(T_2=\E\|R\|^2\leq C\alpha\), proving
Lemma~\ref{lem:SSCT2}.

\subsection{Proof of Lemma \ref{lem:difference-in-Y}}\label{sec:difference-in-Y}
By \eqref{eq: additive different noise},
\[
    A_{t+1}-A_t
    =
    -\alpha A_t
    +
    \sqrt{\alpha}
    \bigl(
        \mT(\sqrt{\alpha}A_t+\theta^*)
        -
        \mT(\theta^*)
    \bigr)
    +
    \sqrt{\alpha}h(x_t).
\]
Therefore, telescoping over $m$ steps yields
\begin{align}
A_{t+m}-A_t
&=
\sum_{k=t}^{t+m-1}
\left[
    -\alpha A_k
    +
    \sqrt{\alpha}
    \bigl(
        \mT(\sqrt{\alpha}A_k+\theta^*)
        -
        \mT(\theta^*)
    \bigr)
\right]
+
\sqrt{\alpha}
\sum_{k=t}^{t+m-1}h(x_k).\nonumber
\end{align}
Then, by Assumption \ref{assumption:contraction} and Lemma \ref{lem:fourthmoment},
\begin{align*}
\E[\|A_{t+m}-A_t\|^2]
&\lesssim
\E\left[\left\|
    \sum_{k=t}^{t+m-1}
        -\alpha A_k
        +
        \sqrt{\alpha}
        \bigl(
            \mT(\sqrt{\alpha}A_k+\theta^*)
            -
            \mT(\theta^*)
        \bigr)
\right\|^2\right]+
\alpha
\E\left[\left\|
    \sum_{k=t}^{t+m-1}h(x_k)
\right\|^2\right]\\
&\lesssim
\alpha^2
m\sum_{k=t}^{t+m-1}\E[\|A_k\|^2]
+
\alpha
\E\left[\left\|
    \sum_{k=t}^{t+m-1}h(x_k)
\right\|^2\right]\lesssim
\alpha^2m^2
+
\alpha
\E\left[\left\|
    \sum_{k=t}^{t+m-1}h(x_k)
\right\|^2\right].
\end{align*}
Notice that, by Assumption~\ref{assumption:noise}, 
\begin{align*}
\E\left[\left\|
    \sum_{k=t}^{t+m-1}h(x_k)
\right\|^2\right]
=
\sum_{k,\ell=t}^{t+m-1}
\E\bigl[
    \langle h(x_k),h(x_\ell)\rangle
\bigr]\lesssim
m+
\sum_{r=1}^{m-1}(m-r)\rho_{\operatorname{mix}}^r
\lesssim m.
\end{align*}
Then, we have
\[
    \E[\|A_{t+m}-A_t\|^2]
    \lesssim
    \alpha^2m^2+\alpha m.
\]
Similarly, we have
\[
    \E[\|Z_{t+m}-Z_t\|^2]
    \lesssim
    \alpha^2m^2+\alpha m,
\]
which completes the proof of Lemma \ref{lem:difference-in-Y}.

\section{Details Omitted from Section~\ref{sec:linearSA}}
\begin{lemma}
\label{lem:LSA}
Suppose that \(\overline{\mathsf A}\) is Hurwitz, and $\kappa
    >
    \max_{\lambda\in\sigma(\overline{\mathsf A})}
    \frac{|\lambda|^2}{-2\operatorname{Re}(\lambda)}.$ Define $\mT_\kappa(\theta)
    :=
    \left(I_d+\frac{\overline{\mathsf A}}{\kappa}\right)\theta
    +\frac{\overline s}{\kappa}.$ Then \(\mT_\kappa\) is contractive under some norm.
\end{lemma}

\begin{proof}[Proof of Lemma~\ref{lem:LSA}]
Recall that the Jacobian of \(\mT_\kappa\) is
\[
    J_{\mathrm{LSA}}
    :=\nabla\mT_\kappa(\theta^*)
    =I_d+\frac{\overline{\mathsf A}}{\kappa}.
\]
Since $\overline{\mathsf A}$ is Hurwitz,
every $\lambda\in\sigma(\overline{\mathsf A})$ satisfies
$\operatorname{Re}(\lambda)<0$. By the choice of $\kappa$,
\[
    \left|1+\frac{\lambda}{\kappa}\right|^2
    =1+\frac{2\operatorname{Re}(\lambda)}{\kappa}
      +\frac{|\lambda|^2}{\kappa^2}
    <1.
\]
Hence $\rho(J_{\mathrm{LSA}})<1$, and the series $P:=\sum_{j=0}^{\infty}(J_{\mathrm{LSA}}^j)^\top J_{\mathrm{LSA}}^j$ converges to a symmetric positive-definite matrix satisfying
\[
    J_{\mathrm{LSA}}^\top P J_{\mathrm{LSA}}-P=-I_d,
    \qquad P\succeq I_d.
\]
Define $\|v\|_c:=\sqrt{v^\top Pv}$. For every $v\in\R^d$,
\[
    \|J_{\mathrm{LSA}} v\|_c^2
    =\|v\|_c^2-\|v\|^2
    \leq
    \left(1-\frac{1}{\lambda_{\max}(P)}\right)\|v\|_c^2.
\]
Since $P\succeq I_d$, the number
\[
    \gamma_\kappa
    :=\sqrt{1-\frac{1}{\lambda_{\max}(P)}}
\]
is well-defined and belongs to $[0,1)$. Taking square roots therefore
gives
\[
    \|\mT_\kappa(\theta)-\mT_\kappa(\theta')\|_c
    =\|J_{\mathrm{LSA}}(\theta-\theta')\|_c
    \leq\gamma_\kappa\|\theta-\theta'\|_c,
    \qquad \forall\,\theta,\theta'\in\R^d.
\]
\end{proof}

\section{Details Omitted from Section~\ref{sec:Q-learning}}\label{sec:proof-Q-learning}
Recall that the mean operator of asynchronous \(Q\)-learning is given
coordinatewise by
\begin{align}
    [\mT(q)](s,a)
    &=
    \bigl(1-\nu(s,a)\bigr)q(s,a)
    +
    \nu(s,a)
    \left(
        r(s,a)
        +
        \rho
        \sum_{s'\in\mS}
        \mathsf P(s'\mid s,a)
        \max_{a'\in\mA}q(s',a')
    \right).
    \label{eq:q-mean-operator-proof}
\end{align}
We verify the nonsmoothness condition in
Corollary~\ref{co:bias}.

For \(s\in\mS\), since \(\mS\) and \(\mA\) are finite, there exists
\(\Delta_*>0\) such that every nonoptimal action satisfies
\[
    V^*(s)-q^*(s,a)\geq\Delta_*,
    \qquad
    a\notin\mA^*(s),
\]
with the convention \(\Delta_*=\infty\) if every action is optimal.
Consequently, for any fixed \(u\in\R^{d_Q}\), when \(w>0\) is sufficiently small,
\[
    \max_{a\in\mA}
    \left\{
        q^*(s,a)+wu(s,a)
    \right\}
    =
    V^*(s)
    +
    w\max_{a\in\mA^*(s)}u(s,a).
\]
Using \eqref{eq:q-mean-operator-proof}, we therefore obtain
\[
    \lim_{w\downarrow0}
    \frac{\mT(q^*+wu)-\mT(q^*)}{w}
    =
    H(u),
\]
where
\begin{equation}
\label{eq:q-directional-derivative}
    [H(u)](s,a)
    =
    \bigl(1-\nu(s,a)\bigr)u(s,a) 
    +
    \rho\nu(s,a)
    \sum_{s'\in\mS}
    \mathsf P(s'\mid s,a)
    \max_{a'\in\mA^*(s')}u(s',a').
\end{equation}
Thus, \(\mT\) is locally one-sided directionally differentiable at \(q^*\),
and \(H\) is the positively homogeneous extension of the the one-sided
directional derivative map at \(q^*\). It remains to show that some coordinate of \(H\) has a nonsingleton
subdifferential at the origin. By assumption, there exists
\(\bar s\in\mS\) such that $|\mA^*(\bar s)|>1.$ Moreover, Assumption~\ref{assumption:MC} implies that the state chain induced
by \(\pi_b\) is irreducible. Hence its stationary distribution
\(\rho_{\pi_b}\) has full support, and
\[
    0
    <
    \rho_{\pi_b}(\bar s)
    =
    \sum_{s\in\mS}
    \rho_{\pi_b}(s)
    \sum_{a\in\mA}
    \pi_b(a\mid s)
    \mathsf P(\bar s\mid s,a).
\]
Therefore, there exists
\((\hat s,\hat a)\in\mS\times\mA\) such that $\mathsf P(\bar s\mid\hat s,\hat a)>0.$ Also, Assumption~\ref{assumption:MC} gives $\nu(\hat s,\hat a)>0.$ Since each mapping $u
    \mapsto
    \max_{a'\in\mA^*(s')}u(s',a')$ is a finite-valued convex function, the subdifferential sum rule
\cite{rockafellar2015convex}, together with
\eqref{eq:q-directional-derivative}, yields
\begin{align}
    \partial [H_Q]_{(\hat s,\hat a)}(0)
    &=
    \bigl(1-\nu(\hat s,\hat a)\bigr)e_{\hat s,\hat a}
    +
    \rho\nu(\hat s,\hat a)
    \sum_{s'\in\mS}
    \mathsf P(s'\mid\hat s,\hat a)
    \operatorname{conv}
    \left\{
        e_{s',a}:a\in\mA^*(s')
    \right\}.
    \label{eq:q-subdiff}
\end{align}
Because $\rho\nu(\hat s,\hat a)
    \mathsf P(\bar s\mid\hat s,\hat a)>0,$ the set in \eqref{eq:q-subdiff} contains at least two elements. Hence $\partial [H]_{(\hat s,\hat a)}(0)$ is nonsingleton.

\section{Details of the Numerical Experiments}

\subsection{A Motivating Numerical Example}\label{sec:experiment-intro}
To illustrate the long-run behavior of constant-stepsize Q-learning and the
effect of tail averaging, we consider a two-state, two-action discounted MDP
with discount factor $\rho=0.9$. The reward matrix and the transition
matrices corresponding to the two actions are given by
\[
R=
\begin{pmatrix}
1 & 0.4\\
0.7 & 0.2
\end{pmatrix},
\qquad
P_0=
\begin{pmatrix}
0.85 & 0.15\\
0.75 & 0.25
\end{pmatrix},
\qquad
P_1=
\begin{pmatrix}
0.20 & 0.80\\
0.10 & 0.90
\end{pmatrix}.
\]
We use a fixed behavior policy that selects action $0$ with probability $0.9$
and action $1$ with probability $0.1$ at each state. Starting from
$q_0=0$, we run Q-learning for $300{,}000$ iterations with constant stepsizes
$\alpha\in\{0.2,0.4,0.8\}$. The first three panels of
Figure~\ref{fig:q-learning-long-run-tail-average} display the last $5{,}000$
iterates, projected onto the coordinates
$(q(0,0),q(0,1),q(1,0))$, with the star indicating the optimal action-value
function $q^*$ computed by value iteration. The final panel illustrates the
effect of tail averaging: at each time $t$, we define
\[
    \bar q_t^{(\alpha)}
    :=
    \frac{1}{t-\lfloor t/2\rfloor}
    \sum_{k=\lfloor t/2\rfloor+1}^{t} q_k^{(\alpha)},
\]
and plot $\bigl\|\bar q_t^{(\alpha)}-q^*\bigr\|_\infty$ against $t$ on a log--log scale for
the three stepsizes.
\subsection{Main Numerical Experiments}\label{app:numerical-details}
\noindent\textbf{MDP and Behavior Policy.}
We consider discounted MDPs with
\[
    |\mS|=4,
    \qquad
    |\mA|=2,
    \qquad
    \rho=0.9.
\]
The behavior policy is uniform over the two actions:
\[
    \pi_b(a\mid s)=\frac12,
    \qquad
    \forall (s,a)\in\mS\times\mA.
\]
The transition kernel is action-independent and given by $\mathsf P(\cdot\mid s,a)=P_0(s,\cdot),$
where
\[
P_0
=
\begin{pmatrix}
0.92 & 0.06 & 0.01 & 0.01\\
0.01 & 0.92 & 0.06 & 0.01\\
0.01 & 0.01 & 0.92 & 0.06\\
0.06 & 0.01 & 0.01 & 0.92
\end{pmatrix}.
\]
The reward is deterministic and has the form $r(s,a)=r_{\mS}(s)+r_{\mA}(a),$ with
\[
    r_{\mS}=(0.05,\,0.15,\,0,\,0.20).
\]
For the no-tie MDP, we set
\[
    r_{\mA}=(0.75,\,0.10),
\]
so that one action is uniquely optimal at every state. For the tied MDP, we
set
\[
    r_{\mA}=(0.75,\,0.75),
\]
so that the two actions are tied and optimal at every state. 

\noindent\textbf{Q-Learning Implementation.}
We run asynchronous constant-stepsize Q-learning with
\[
    \alpha\in\{0.05,\,0.10,\,0.20,\,0.40\}.
\]
Each run is initialized at
\[
    Q_0(s,a)=-5,
    \qquad
    \forall (s,a)\in\mS\times\mA.
\]
Each trajectory
is simulated for $T=10^7$ iterations.

\end{document}